\documentclass{article}

\PassOptionsToPackage{numbers, compress}{natbib}
\PassOptionsToPackage{hyperfootnotes=false}{hyperref}

\usepackage{mathrsfs}

\usepackage[main, final]{neurips_2026}

\usepackage[utf8]{inputenc} 
\usepackage[T1]{fontenc}    
\usepackage{hyperref}       
\usepackage{url}            
\usepackage{booktabs}       
\usepackage{multirow}
\usepackage{amsfonts}       
\usepackage{nicefrac}       
\usepackage{microtype}      
\usepackage{xcolor}         

\usepackage[toc,page]{appendix}
\usepackage{titletoc}
\usepackage{etoc}

\usepackage{algorithm}
\usepackage{algpseudocode}

\usepackage{amsmath,amssymb,mathtools}
\usepackage{cleveref}
\usepackage{amsthm}

\theoremstyle{plain}
\newtheorem{theorem}{Theorem}
\newtheorem{lemma}{Lemma}

\theoremstyle{definition}
\newtheorem{assumption}{Assumption}
\newtheorem{definition}{Definition}

\theoremstyle{remark}
\newtheorem{remark}{Remark}

\crefname{theorem}{theorem}{theorems}
\Crefname{theorem}{Theorem}{Theorems}

\crefname{lemma}{lemma}{lemmas}
\Crefname{lemma}{Lemma}{Lemmas}

\crefname{proposition}{proposition}{propositions}
\Crefname{proposition}{Proposition}{Propositions}

\crefname{corollary}{corollary}{corollaries}
\Crefname{corollary}{Corollary}{Corollaries}

\crefname{assumption}{assumption}{assumptions}
\Crefname{assumption}{Assumption}{Assumptions}

\crefname{definition}{definition}{definitions}
\Crefname{definition}{Definition}{Definitions}

\crefname{remark}{remark}{remarks}
\Crefname{remark}{Remark}{Remarks}

\title{Learning Fractional-Order Dynamics from a Single Trajectory}

\author{
Xiaole Zhang\textsuperscript{1$*$},
Ziyi Zhang\textsuperscript{2}, 
Zehao Zhao\textsuperscript{1}, 
Stephen Tu\textsuperscript{1},\\
\textbf{Guannan Qu\textsuperscript{2}, 
Yorie Nakahira\textsuperscript{2}, 
Paul Bogdan\textsuperscript{1}}\\
  \small
  $^{1}$Ming Hsieh Department of Electrical and Computer Engineering,
  University of Southern California\\
  \small
  $^{2}$Department of Electrical and Computer Engineering,
  Carnegie Mellon University \\
}

\begin{document}

\maketitle
\begingroup
\renewcommand{\thefootnote}{$*$}
\footnotetext{Corresponding to \texttt{xiaolezh@usc.edu}.}
\endgroup

\begin{abstract}
Many real-world processes exhibit long-range dependence, where the current state depends on a slowly decaying trace of past states rather than on the most recent state alone. This paper studies system identification for discrete-time fractional-order linear time-invariant systems from a single observed trajectory of length $t$, a setting that captures such non-Markovian dynamics through the Grünwald--Letnikov difference operator. Unlike Markovian systems, fractional-order systems couple estimation across the entire history, making both statistical analysis and practical identification more challenging. We propose \emph{Fractional-Order Ordinary-Least-Squares Grid-Search (FO-GS)}, a simple two-stage estimator that exploits the diagonal structure of the fractional-difference operator to decouple the identification problem row-wise. Under the stability assumption, we establish high-probability, non-asymptotic error bounds for estimating both the fractional order and the system matrix in the heterogeneous setting, with both estimation errors scaling as \(\mathcal{O}(t^{-1/2})\). Through experiments, we show that \emph{FO-GS} outperforms existing baselines in recovering both the fractional order and the underlying system dynamics.
\end{abstract}
\section{Introduction}

Many complex natural and technological systems exhibit long-range dependence, with temporal correlations decaying as a power law rather than exponentially (see \Cref{fig:motivation_FOLTI_v1}). Long-range dependence is ubiquitous across diverse domains, including brain~\citep{article,lundstrom2008fractional}, muscle and blood glucose activity~\citep{6659004,7459389}, heart-rate variability~\citep{doi:10.1073/pnas.012579499,ivanov1999multifractality}, gene expression~\citep{10.3389/fphys.2018.01446}, climate and hydrology~\citep{doi:10.1061/TACEAT.0006518,https://doi.org/10.1029/WR005i002p00321}, network traffic~\citep{Leland1993OnTS,willinger2003long}, finance~\citep{https://doi.org/10.1111/1468-0262.00418,DING199383}, and even machine-learning systems (e.g., large language models~\citep{alabdulmohsin2025a,alabdulmohsin2024fractal}). 

Classical Markovian models are inherently ill-suited to capture long-range dependence, as their dynamics depend only on the current state and therefore lack a mechanism to encode persistent historical influence and long-range correlations. In contrast, fractional-order systems provide a natural alternative: the Grünwald--Letnikov difference operator replaces the one-step recursion from an integer-order system with a weighted sum over the full history~\citep{hilfer2000applications,ionescu2017role,monje2010fractional,oldham1974fractional}. While this nonlocal mathematical formulation makes fractional-order models particularly well suited for describing history-dependent dynamics, it also introduces substantial challenges for system identification. The main challenge in identifying fractional-order systems is their intrinsic non-Markovian nature: the current state depends nontrivially on a long history of past states, and the unknown fractional order governs this dependence through the coefficients of the Grünwald--Letnikov difference operator. Consequently, jointly estimating the fractional order and the system matrix leads to a nonlinear inference problem with long-range dependence. Recent works have studied learning and sample-complexity questions for discrete-time fractional-order systems~\cite{chatterjee2022learning,yaghooti2023inferring,11107451,zhang2025endtoend}, but statistical guarantees for learning stochastic fractional-order systems from a single trajectory remain underexplored.

\begin{figure}[t]
 \centering
 \includegraphics[width=\textwidth]{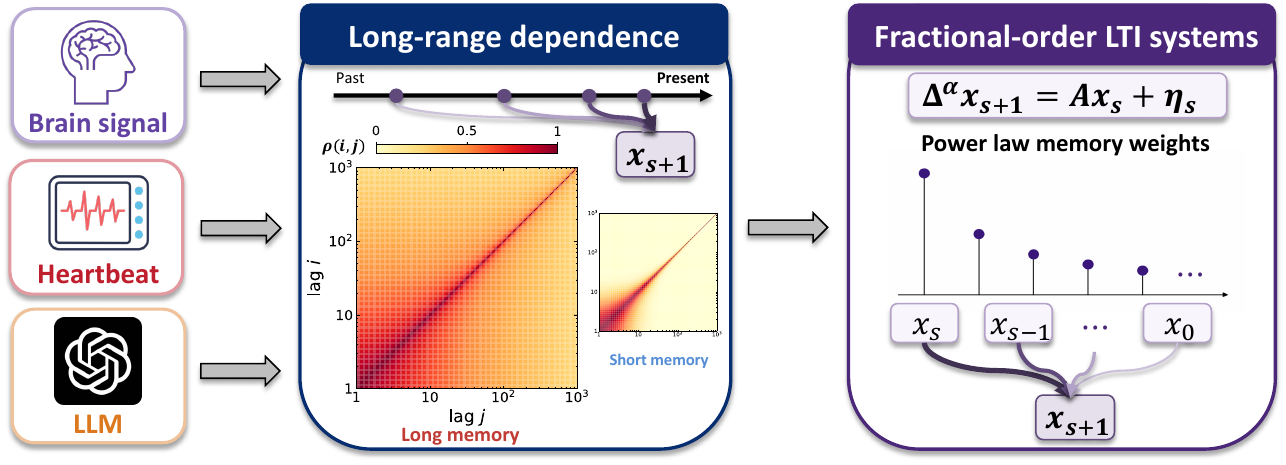}
 \caption{Complex adaptive systems exhibit non-Markovian dynamics, mathematically characterized by long-range dependence. From biological to modern machine learning systems (left panel), the autocorrelation function decays as a power law rather than exponentially (middle panel). Fractional-order operators with their intrinsic power law memory kernel offer compact mathematical strategies to capture this observed long-range dependence dynamics (right panel).} \label{fig:motivation_FOLTI_v1}
\end{figure}
To fill this knowledge gap, we study the fractional-order linear time-invariant (FOLTI) system
\[
\Delta^{\boldsymbol{\alpha}} x_{s+1} = A x_s + \eta_s,
\]
where $\Delta^{\boldsymbol{\alpha}}$ denotes the Grünwald--Letnikov difference operator \eqref{eq:GL_difference_operator}, $\boldsymbol{\alpha} \in (0,1)^n$ is the fractional-order vector, $x_s \in \mathbb{R}^n$ is the $n$-dimensional state at time $s$, and $\eta_s \sim \mathcal{N}(0,\sigma^2 I_n)$ is the additive Gaussian noise with variance $\sigma^2$. We aim to estimate the fractional order $\boldsymbol{\alpha}$ and system matrix $A \in \mathbb{R}^{n \times n}$ given a single observed trajectory \(x_0,x_1,\dots,x_t\). We consider both the heterogeneous setting, in which the coordinates may have different fractional order coefficients in the fractional derivatives, and the homogeneous setting, in which all coordinates share a common fractional order. The central question is whether one can obtain a computationally simple estimator together with non-asymptotic guarantees in this genuinely non-Markovian regime.

This problem has two intertwined challenges: First, the fractional-order difference operator couples each observation to the entire past trajectory, so standard Markovian identification arguments do not apply directly. Second, the dependence on the unknown fractional order \(\boldsymbol{\alpha}\) is nonlinear, making joint optimization over \(\boldsymbol{\alpha}\) and \(A\) difficult even in LTI dynamics. In the single-trajectory setting, these difficulties are compounded by strong temporal dependence and the absence of independent rollouts.

\textbf{Our contributions.} Unlike traditional Ordinary-Least-Squares (OLS)-based algorithms for system identification, we propose the \emph{Fractional-Order OLS Grid-Search (FO-GS)} algorithm, a two-stage identification scheme for a stochastic discrete time FOLTI system from a single trajectory. Using the diagonal structure of the Grünwald--Letnikov difference operator, we identify each row of $A$ together with each coordinate of $\boldsymbol{\alpha}$ separately; along each dimension, \emph{FO-GS} establishes a grid of all potential $\alpha_i$ with a fixed step size $\epsilon_i$ and then estimates $a_i$, the $i$-th row of system matrix $A$, conditioned on each candidate $\alpha_i$. Then, each $(\alpha_i,a_i)$-pair is evaluated to minimize a profiled loss to select the best candidates. In particular, we show that under the stability assumption, \emph{FO-GS} provably recovers the ground truth parameters $\alpha$ and $A$ with high probability, with both estimation errors scaling as \(\mathcal{O}(t^{-1/2})\).

To the best of our knowledge, \emph{FO-GS} offers the first non-asymptotic statistical guarantee for separately identifying $\boldsymbol{\alpha}$ and $A$ for FOLTI systems on a single trajectory. We validate \emph{FO-GS} on both synthetic and real-world datasets. Taken together, \emph{FO-GS} provides a principled framework for learning fractional-order dynamical systems from limited sequential data, thereby substantially broadening the scope of identifiable system classes beyond the standard Markovian setting.

\section{Related Work}
Our work is rooted in the newly developed branch of fractional-order system identification~\cite{gupta2019learning,11107451}, which 
borrows inspiration from online learning and identification for ordinary LTI systems. 

\textbf{Learning on a single trajectory for ordinary LTI systems.} The identification of the ordinary LTI systems has a long history of research~\citep{Oymak18,Sarkar18,pmlr-v75-simchowitz18a,Sun20,zhang2025learning,zhang2025stabilizing,Zheng201}. Compared with those works, our algorithm adapts many methodologies and proof techniques and extends the scope to a broader non-Markovian system dynamics by coupling traditional OLS-based identification methods with a grid search scheme on each coordinate of $\boldsymbol{\alpha}$ and row of $A$, taking advantage of the diagonal structure of the Grünwald--Letnikov difference operator. If $\boldsymbol{\alpha} = \boldsymbol{1}$ and the fractional-order system simplifies to an ordinary LTI system, \emph{FO-GS} offers a theoretical guarantee comparable to the state-of-the-art guarantee in ordinary LTI systems. 

\textbf{Identification of FOLTI systems.}
While system identification for ordinary LTI systems is a relatively well understood field, far less attention has been devoted to fractional-order systems~\citep{1971712334804565506,monje2010fractional,podlubny1998fractional,7459389}. When multiple trajectories can be sampled, some works have been developed in fractional-order system identification. \citet{yaghooti2023inferring} propose a two-stage identification framework that first estimates the fractional-order parameters from trajectory data generated under a prescribed data-collection procedure and, conditioned on these estimates, reformulates the discrete-time control-affine nonlinear fractional dynamics as a regression problem to infer the unknown system dynamics. \citet{11107451,zhang2025endtoend} generalize this framework to stochastic settings. \citet{chatterjee2022learning} study single-trajectory identification for fractional-order systems through system truncation and establish a sample-complexity result for the augmented system matrix. By contrast, we study the original single-trajectory identification problem directly, without requiring either data generation or system truncation, and establish the first non-asymptotic guarantees for both the fractional-order parameter $\boldsymbol{\alpha}$ and the system matrix $A$.
\section{Preliminaries and Problem Formulation}
\label{sec:prelim}
\subsection{Grünwald–Letnikov Difference Operator}
The Grünwald–Letnikov difference operator allows to discretize the fractional-order derivative and represent it as a finite difference of the form as follows:
\begin{equation}
\label{eq:GL_difference_operator}
\Delta^{\boldsymbol{\alpha}} x_s := \sum_{j=0}^{s} \Psi(\boldsymbol{\alpha}, j) x_{s-j},
\end{equation}
where \( x_s \in \mathbb{R}^{n} \), \( \boldsymbol{\alpha} = [\alpha_1, \alpha_2, \ldots, \alpha_{n}]^\top \in (0,1)^{n} \) represents the fractional order, and \( \Psi(\boldsymbol{\alpha}, j) \in \mathbb{R}^{n \times n} \) is a diagonal matrix defined as
\(
\Psi(\boldsymbol{\alpha}, j) := \operatorname{diag}(\psi(\alpha_1, j), \psi(\alpha_2, j), \ldots, \psi(\alpha_{n}, j))
\)
with
\(
\psi(\alpha_i, j) := \frac{\Gamma(j - \alpha_i)}{\Gamma(-\alpha_i) \Gamma(j + 1)}\) for \( i = 1, 2, \ldots, n.
\) Here \(\Gamma(\cdot)\) denotes the gamma function.

\subsection{FOLTI System Identification}
The state-space representation of the discrete-time FOLTI system reads:
\begin{equation}
\label{eq:folti_no_input}
\Delta^{\boldsymbol{\alpha}} x_{s+1} = A x_s + \eta_s, 
\end{equation}
where \(x_s \in \mathbb{R}^n \) is the state vector, $\eta_s \sim \mathcal{N}(0,\sigma^2 I_n)$ for some $\sigma >0$ is independent and identically distributed, and $A \in \mathbb{R}^{n \times n}$ is a constant real matrix. Using the Grünwald–Letnikov difference operator \eqref{eq:GL_difference_operator}, we can write the system \eqref{eq:folti_no_input} as follows:
\begin{equation}
\label{eqn:sys_dyna_wo_input}
x_{s+1} = A x_s  - \sum_{j=1}^{s+1} \Psi(\boldsymbol{\alpha}, j) x_{s+1-j} + \eta_s.
\end{equation}
The solution to the discrete-time FOLTI system \eqref{eq:folti_no_input} is given by~\citep{guermah2012discrete}:
\begin{equation*}
x_s = G_s x_0  + \sum_{j=0}^{s-1} G_{s-1-j} \eta_j,
\end{equation*}
where the matrices \(G_s\) are defined recursively by
\begin{align}
G_s &=
\begin{cases}
I, & s=0, \\
\sum_{j=0}^{s-1} A_j G_{s-1-j}, & s\ge 1,
\end{cases}
\qquad
A_j =
\begin{cases}
A+\operatorname{diag}(\alpha_1,\ldots,\alpha_n), & j=0, \\
-\Psi(\boldsymbol{\alpha},j+1), & j\ge 1.
\end{cases}
\label{con:Aj_eq}
\end{align}
\textbf{Problem statement:} Given a single observed trajectory $x_0, x_1,\dots, x_t$, our goal is to identify the fractional order \(\boldsymbol{\alpha}\) and the system matrix \(A\), and to establish statistical guarantees for the resulting estimators.

\section{Main Results}
\label{sec:main_results}
In this section, we introduce the algorithm for identifying the system parameters \(\boldsymbol{\alpha}\) and $A$ in Section \ref{sec:FO-GS}, and present the sample complexity results for the resulting estimators in Section \ref{sec:theory}.

\subsection{FO-GS}
\label{sec:FO-GS}
\begin{algorithm}[t]
\caption{Fractional-Order OLS Grid-Search (\textit{FO-GS})}
\label{alg:vector_fractional_id}
\begin{algorithmic}[1]
\Require Trajectory $\{x_s\}_{s=0}^{t}$, row search interval $[\underline{\alpha}_i,\bar{\alpha}_i]$, grid size $\epsilon_i$.
\Ensure Estimates $\hat{\boldsymbol{\alpha}} = [\hat{\alpha}_1,\dots,\hat{\alpha}_n]^\top$ and $\hat{A}$.

\State Build grid $\mathcal{A}_{\epsilon,i} \subset [\underline{\alpha}_i,\bar{\alpha}_i]$.

\State Form the data matrix $X_t$ with \eqref{def:Delta_X}.

\For{$i=1,\dots,n$}
    \For{each $\alpha \in \mathcal{A}_{\epsilon,i}$}
        \State Compute the fractional-difference row
        \(
        \Delta^\alpha X_t^{(i)}
        \) with \eqref{def:delta_X_row}.

        \State Compute the least-squares row estimator
        \(
        \hat{a}_i(\alpha)
        \) with \eqref{eq:ols_solution_row}.

        \State Compute the profiled loss \(\mathcal{L}^{(i)}(\alpha)\) with \eqref{eq:prof_loss_row}.
    \EndFor

    \State Select
    \(
    \hat{\alpha}_i
    \gets
    \arg\min_{\alpha\in\mathcal{A}_{\epsilon,i}}\mathcal{L}^{(i)}(\alpha).
    \)

    \State Set
    \(
    \hat{a}_i \gets \hat{a}_i(\hat{\alpha}_i).
    \)
\EndFor

\State Form
\(
\hat{\boldsymbol{\alpha}}
\gets
[\hat{\alpha}_1,\dots,\hat{\alpha}_n]^\top,\) and
\(\hat{A}
\gets
[\hat{a}_1^\top,\dots,\hat{a}_n^\top]^\top.
\)

\Return $(\hat{\boldsymbol{\alpha}},\hat{A})$.
\end{algorithmic}
\end{algorithm}
We introduce the algorithm for identifying the system parameters \((\boldsymbol{\alpha}, A)\) of the FOLTI system \eqref{eq:folti_no_input} from a single observed trajectory. The main idea is to isolate each coordinate of $\boldsymbol{\alpha}$ under the diagonal structure of the Grünwald--Letnikov difference operator and, for each candidate, solve an OLS problem to estimate \(A\). To handle the general fractional-order setting, \Cref{alg:vector_fractional_id} performs a grid search over each component of \(\boldsymbol{\alpha}\), and then solves a row-wise OLS problem. Specifically, for each \(i \in \{1,\dots,n\}\), define
\begin{equation}
\label{def:delta_X_row}
\Delta^{\alpha_i} X_t^{(i)}
:=
\bigl[
\Delta^{\alpha_i} x_1^{(i)},\dots,\Delta^{\alpha_i} x_t^{(i)}
\bigr]
\in \mathbb{R}^{1\times t},
\end{equation}
where
\(
\Delta^{\alpha_i} x_{s+1}^{(i)}
=
\sum_{j=0}^{s+1}\psi(\alpha_i,j)x_{s+1-j}^{(i)}.
\)
Denote \(A = [a_1^\top,\dots,a_n^\top]^\top\), where \(a_i \in \mathbb{R}^{1\times n}\) represents the \(i\)-th row of \(A\). We further define the loss function $\mathcal{L}$ as follows:
\begin{equation}
\label{eq:loss_row_decomp}
\mathcal{L}(\boldsymbol{\alpha},A)
=
\sum_{s=0}^{t-1} \bigl\|\Delta^{\boldsymbol{\alpha}} x_{s+1}-Ax_s\bigr\|_2^2
=
\sum_{i=1}^n \sum_{s=0}^{t-1}
\bigl(\Delta^{\alpha_i}x_{s+1}^{(i)}-a_i x_s\bigr)^2.
\end{equation}
For each row $i$, let \(\mathcal{A}_i = [\underline{\alpha}_i, \bar{\alpha}_i] \subset (0,1]\) be a compact search interval. Fixing a step size \(\epsilon_i > 0\), we construct a uniform grid
\(
\mathcal{A}_{\epsilon,i}
=
\bigl\{
\alpha_{i,k}:=\underline{\alpha}_i + k\epsilon_i
\,|\,
k \in \{1,\dots,M_i\}
\bigr\}
\subset \mathcal{A}_i
\). For each candidate \(\alpha_{i,k}\), we solve the row-wise OLS problem and obtain:
\begin{equation}
\label{eq:ols_solution_row}
\hat{a}_i(\alpha_{i,k}) := \arg\min_{a_i}
\mathcal{L}^{(i)}(\alpha_{i,k},a_i)
=(\Delta^{\alpha_{i,k}}X_t^{(i)})X_t^\top( X_tX_t^\top)^{-1}.
\end{equation}
We then minimize the corresponding profiled loss and obtain the estimated $\alpha_i$:
\begin{equation}
\label{eq:prof_loss_row}
\hat{\alpha}_i = \arg\min_{\alpha \in \mathcal{A}_{\epsilon,i}}\mathcal{L}^{(i)}(\alpha)
=
\arg\min_{\alpha \in \mathcal{A}_{\epsilon,i}}\sum_{s=0}^{t-1}
\bigl(
\Delta^{\alpha}x_{s+1}^{(i)}
-
\hat{a}_i(\alpha)x_s
\bigr)^2.
\end{equation}
Finally, by stacking the row estimators \(\hat{a}_i(\alpha_i)\), we obtain the system matrix estimator
\begin{equation*}
\hat{A}(\boldsymbol{\alpha}) = (\Delta^{\boldsymbol{\alpha}}X_t) X_t^\top (X_t X_t^\top)^{-1},
\end{equation*}
where
\begin{equation}
\label{def:Delta_X}
\Delta^{\boldsymbol{\alpha}} X_t
=
[\Delta^{\boldsymbol{\alpha}} x_1, \ldots, \Delta^{\boldsymbol{\alpha}} x_t],
\;
X_t
=
[x_0, \ldots, x_{t-1}].
\end{equation}
The commensurate setting is a direct specialization of the above procedure. When all coordinates share a common fractional order, i.e., \(\boldsymbol{\alpha} = \alpha \boldsymbol{1}\), the row-wise searches collapse to a single one-dimensional search over \(\alpha\). 
\subsection{Theoretical Guarantees}
\label{sec:theory}
In this section, we present complexity guarantees for the proposed algorithm. We begin by introducing the assumptions required for our main results. For the rest of the paper, we use $\boldsymbol{\alpha}_\star = [\alpha_{1, \star},... \alpha_{n, \star}]^\top$ to denote the true fractional-order vector and \(A_\star\) to denote the true system matrix. We assume $x_0 =0$ for simplicity. We first introduce a stability assumption that is standard in the LTI system identification literature. 
\begin{assumption}
\label{ass:stable_A}
The system parameters \((A_\star,\boldsymbol{\alpha}_\star)\) are stable in the sense that for all \(|z|\le 1\),
\begin{equation*}
\det\!\Bigl(
\operatorname{diag}\bigl((1-z)^{\alpha_{1,\star}},\dots,(1-z)^{\alpha_{n,\star}}\bigr)
- zA_\star
\Bigr)\neq 0.
\end{equation*}
\end{assumption}
\Cref{ass:stable_A} can be interpreted as the fractional-order version of the stability assumption common in control literature~\citep{jedra2020finitetimeidentificationstablelinear,oymak2019nonasymptoticidentificationltisystems,petravs2021stability,rivero2013stability,sarkar2021finite}. It is slightly stronger than minimal stability, since imposing the condition at \(z=1\) implies that \(A_\star\) is invertible. This property is used in our analysis to bound the matrices \(G_s\) in \eqref{con:Aj_eq}. It ensures that the state $x_t$ does not blow-up with time. If the system is unstable, then any error at the early time period would be exponentially amplified by the unstable system dynamics in a phenomenon known as \emph{stochastic coupling}~\citep{zhang2025learning}. We leave this as a future direction of this paper.

We are now ready to introduce the main theorems. First, we introduce the error bound on $\boldsymbol{\alpha}_\star$:

\begin{theorem}
\label{thm:vector_alpha_error_bound}
Suppose \Cref{ass:stable_A} holds and the population separation gap $\gamma$ in \eqref{eq: def_gamma} is positive. Let
\(
\epsilon_{\max}:=\max_{1\le i\le n}\epsilon_i.
\)
Fix $\delta\in(0,1/2)$. If
\(
t
\gtrsim
\frac{1}{\min_i\underline{\alpha}_{i}^4}
\left(
n+\log\frac{\sum_{i=1}^n M_i}{\delta}
\right),
\)
and the excitation \eqref{eq:excitation conditions} and localization \eqref{eq:global_localization_condition} conditions are satisfied, then with probability at least $1-\delta$,
\begin{equation}
\left\|
\widehat{\boldsymbol{\alpha}}
-
\boldsymbol{\alpha}_\star
\right\|_\infty^2
\lesssim
\operatorname{poly}\!\left(n,\frac{1}{\delta}\right)
\left[
\epsilon_{\max}^2
+
\frac{1}{t}
\sum_{i=1}^n
\log\frac{M_i n}{\delta}
+
\frac{1}{t}
\log\frac{n}{\delta}
\right],
\end{equation}
where $\lesssim$ hides system-dependent constants independent of
$t$, $n$, $\delta$, $\epsilon_{\max}$, and $M_i$. Consequently, if
\(
\epsilon_{\max}
=
\mathcal{O}\left(t^{-1/2}\right),
\)
then
\(
\left\|
\hat{\boldsymbol{\alpha}}
-
\boldsymbol{\alpha}_\star
\right\|_\infty
=
\mathcal{O}\left(t^{-1/2}\right).
\)
\end{theorem}

We defer the proof of \Cref{thm:vector_alpha_error_bound} to \Cref{sec:proof_thm_1}. The condition for \(t\) requires the trajectory to be long enough for both the global and local lower isometry bounds to hold uniformly over the finite search grid. In addition to the excitation and localization conditions, the trajectory length scales as $t \gtrsim \frac{1}{\min_i\underline{\alpha}_{i}^4} \left( n + \log \frac{\sum_{i=1}^n M_i}{\delta} \right),$ thus smaller fractional orders require longer trajectories, reflecting the stronger long-memory dependence in this regime. \Cref{thm:vector_alpha_error_bound} makes explicit the tradeoff between statistical error and grid discretization error. Up to logarithmic factors in the grid size \(M_i\), the statistical term decays as \(t^{-1/2}\), whereas the discretization term decays as \(\epsilon_{\max}\). Accordingly, choosing \(\epsilon_{\max}=\mathcal{O}(t^{-1/2})\) makes the two contributions comparable. For a uniform grid over a bounded interval, this corresponds to \(M_i=\mathcal{O}(t^{1/2})\), which is sufficient to match the statistical precision. To the best of our knowledge, this is the first high-probability single trajectory error bound for estimation of the fractional order \(\boldsymbol{\alpha}_\star\) in the FOLTI setting. Our result is complementary to prior work based on truncated, bisection-like identification schemes~\citep{chatterjee2022learning}, and to more recent analyses developed under different data-generation frameworks~\citep{yaghooti2023inferring,11107451,zhang2025endtoend}. We then discuss the error complexity of estimating $A_\star$ in the following theorem:
\begin{theorem}
\label{thm:A_error_bound_vec}
Under the same condition as in \Cref{thm:vector_alpha_error_bound}, fix \(\delta \in (0, \frac{1}{2})\) and consider the system \eqref{eqn:sys_dyna_wo_input}. Let \(\Gamma_s = \sum_{m=0}^{s-1}G_m G_m^\top\) and \(
\Xi_t(\delta,k)
:=
n\log\frac{9n}{\delta}
+
\log\det\!\left(
\Gamma_t\Gamma_k^{-1}
\right).
\) Then there exist universal constants \(c,C>0\) such that, for any integer \(k\) satisfying
\(
\frac{t}{k}
\ge
c \Xi_t(\delta,k),
\)
the following holds with probability at least $1-\delta$:
\begin{align*}
\bigl\|\hat A(\hat{\boldsymbol{\alpha}})-A_\star\bigr\|_\mathrm{op}
\le
C \Biggl(
\sqrt{\frac{
\Xi_t(\delta,k)
}{t\,\lambda_{\min}(\Gamma_k)}}
+
S_1 \|\hat{\boldsymbol{\alpha}} - \boldsymbol{\alpha}_\star \|_\infty
\sqrt{
\frac{n\tilde{C}_G^2}
{\delta\lambda_{\min}(\Gamma_{{\lfloor k/2\rfloor}})}
}
\Biggr),
\end{align*}
where $\tilde{C}_G$ is a constant depending on $(A_\star, \boldsymbol{\alpha}_\star)$, and \(S_1\) is a constant depending on the grid search interval.
\end{theorem}
We defer the proof of \Cref{thm:A_error_bound_vec} to \Cref{sec:proof_thm_2}. \Cref{thm:A_error_bound_vec} shows that the estimation error for \(A_\star\) decomposes into two parts. The first term is a standard oracle least-squares error, which is the error that would arise even if the true fractional-order vector were known. The second term quantifies the propagation of the fractional-order estimation error into the estimation of \(A_\star\). Consequently, when \(\hat{\boldsymbol{\alpha}}\) is obtained from \Cref{thm:vector_alpha_error_bound}, the bound for \(A_\star\) inherits the same statistical-discretization tradeoff as the bound for \(\boldsymbol{\alpha}_\star\). In particular, if \(\epsilon_{\max}=\mathcal{O}(t^{-1/2})\), then the propagated term is of order \(t^{-1/2}\), while the oracle term is also of order \(t^{-1/2}\). Therefore, the overall estimation error for $A_\star$ achieves the $t^{-1/2}$ rate. \citet{chatterjee2022learning} analyzes a truncated system identification scheme, but does not explicitly quantify the estimation error for the original system matrix \(A_\star\) or how fractional-order estimation error propagates to the estimation of \(A_\star\). We also note that in the case when \(\boldsymbol{\alpha}_\star=\mathbf{1}\) is known a priori, there is no need to estimate the fractional-order parameter. So \Cref{thm:A_error_bound_vec} is equivalent to the current state-of-the-art bound for estimating $A_\star$ for ordinary LTI systems~\citep{pmlr-v75-simchowitz18a}.

\section{Proof Outline}
\label{sec:proof_outline}
The proof in this paper is split into two steps, bounding the estimation error for $\boldsymbol{\alpha}_\star$ and then for $A_\star$. In this section, we provide an outline of the proof for each, and defer the details to the appendix. 

\subsection{Proof of \texorpdfstring{\Cref{thm:vector_alpha_error_bound}}{thmvaeb}: bounding the error of \texorpdfstring{$\boldsymbol{\alpha}_\star$}{alpha}}
\label{sec:bound_alpha}

The estimation of \(\boldsymbol{\alpha}_\star\) can be analyzed in two steps: controlling the in-sample error and establishing a lower isometry bound. The first step gives an absorbable upper bound on the in-sample prediction error, while the second localizes the estimator and converts the same error into a quadratic lower bound on the fractional-order estimation error.

\textbf{Controlling the in-sample error.}
For each row $i$, define
\(
b_s^{(i)}(\alpha_i)
:=
\left(
\Delta^{\alpha_i}
-
\Delta^{\alpha_{i,\star}}
\right)x_{s+1}^{(i)}
\)
and $y_s^{(i)}(\alpha_i) := b_s^{(i)}(\alpha_i)+2\eta_s^{(i)}$, and let $\alpha_i^\circ$ denote the grid point closest to
$\alpha_{i,\star}$. We study the unnormalized in-sample error
\begin{equation*}
\mathcal E_t
:=
\sum_{s=0}^{t-1}
\left\|
\left(
\Delta^{\hat{\boldsymbol{\alpha}}}
-
\Delta^{\boldsymbol{\alpha}_\star}
\right)x_{s+1}
-
(\hat A-A_\star)x_s
\right\|_2^2 .
\end{equation*}
By the optimality of the row-wise estimator $L^{(i)}(\hat\alpha_i,\hat a_i)
\le
L^{(i)}
\!\left(
\alpha_i^\circ,\hat a_i(\alpha_i^\circ)
\right)
\le
L^{(i)}(\alpha_i^\circ,a_{i,\star})$ and a quadratic
maximization over the system matrix perturbation, we obtain
\begin{align}
\mathcal E_t
\le
\sum_{i=1}^n
\max_{\alpha_i\in\mathcal A_{\epsilon,i}}
&\left\{
\underbrace{-4\sum_{s=0}^{t-1}
\eta_s^{(i)}b_s^{(i)}(\alpha_i)
-
\sum_{s=0}^{t-1}
|b_s^{(i)}(\alpha_i)|^2}_{U_{t,i}(\alpha_i)} 
+
\underbrace{\left\|
\left(
\sum_{s=0}^{t-1}
y_s^{(i)}(\alpha_i) x_s^\top
\right)
(X_tX_t^\top)^{-\frac12}
\right\|_2^2}_{V_{t,i}(\alpha_i)}
\right\} \notag \\
&+
\underbrace{\sum_{i=1}^n
\left[
4\sum_{s=0}^{t-1}
\eta_s^{(i)}b_s^{(i)}(\alpha_i^\circ)
+
2\sum_{s=0}^{t-1}
|b_s^{(i)}(\alpha_i^\circ)|^2
\right]}_{\Gamma_t^{\mathrm{grid}}}.
\label{eq:insample_offset_outline}
\end{align}

\textbf{\textit{Exact grid case.}}
If the true parameter lies exactly on the search grid, i.e.,
$\alpha_{i,\star}\in\mathcal A_{\epsilon,i}$ for every $i$, then we may
take $\alpha_i^\circ=\alpha_{i,\star}$. In this case
\(
    b_s^{(i)}(\alpha_i^\circ)=0
\)
for every $i$ and $s$, and hence
\(
    \Gamma_t^{\mathrm{grid}}=0.
\)
Thus, the continuous empirical risk minimizer offset inequality is recovered as a special
case.

\textbf{\textit{Row-wise complexity decomposition.}}
Since the fractional-order operator is diagonal across state
coordinates and the estimator searches for each $\alpha_i$ separately,
the offset complexity decomposes as
\begin{equation*}
    \sum_{i=1}^n
    \max_{\alpha_i\in\mathcal A_{\epsilon,i}}
    \left\{
        U_{t,i}(\alpha_i)
        +
        V_{t,i}(\alpha_i)
    \right\},
\end{equation*}
rather than requiring a supremum over the full Cartesian grid
$\mathcal A_{\epsilon,1}\times\cdots\times
\mathcal A_{\epsilon,n}$.
This row-wise decomposition is important for obtaining the sharp
complexity dependence of the grid-search estimator.

The offset martingale complexity argument~\citep{pmlr-v178-ziemann22a} preserves the negative quadratic term in $U_{t,i}$, yielding a term proportional to $\left\|
\hat{\boldsymbol{\alpha}}-\boldsymbol{\alpha}_\star
\right\|_\infty^2$ with a tunable coefficient, which can later be absorbed by the lower isometry bound. The additional
term $\Gamma_t^{\mathrm{grid}}$ accounts for the finite grid
approximation and vanishes when the true fractional orders lie on the
search grid. Bounding $U_{t,i}$, $V_{t,i}$, and $\Gamma_t^{\mathrm{grid}}$ separately then yields the following high probability control of $\mathcal{E}_t$.

\begin{lemma}
\label{lem:control_in_sample}
For any fixed $\tau\in(0,1]$ and $\rho>0$, if $t/k\ge c\Xi_t(\delta,k)$, then with probability at least $1-\delta$,
\[
\begin{aligned}
\mathcal E_t\le{}&(\tau+\rho)S_1^2\left\|
\hat{\boldsymbol{\alpha}}-\boldsymbol{\alpha}_\star
\right\|_\infty^2\operatorname{tr}(X_tX_t^\top)
+3S_1^2\epsilon_{\max}^2\operatorname{tr}(X_tX_t^\top)\\
&+\frac{4(1+\rho^{-1})C^2\operatorname{tr}(X_tX_t^\top)}{t\lambda_{\min}(\Gamma_k)}\Xi_t(\delta,k)
+\frac{8\sigma^2}{\tau}\sum_i\log\frac{3M_in}{\delta}
+8\sigma^2\log\frac3\delta .
\end{aligned}
\]
The coefficient $\tau+\rho$ is tunable and can be absorbed by lower isometry.
\end{lemma}
We defer the proof of \Cref{lem:control_in_sample} to \Cref{sec:offset_row_search}.

\textbf{Lower isometry.} For each row $i$, define the unprofiled and profiled noiseless errors $\mathcal Q_{t,i}(\alpha_i,v) :=\sum_{s=0}^{t-1}
\bigl|b_s^{(i)}(\alpha_i)-v x_s\bigr|^2$ and $\underline{\mathcal Q}_{t,i}(\alpha_i) :=\inf_{v\in\mathbb R^{1\times n}}
\mathcal Q_{t,i}(\alpha_i,v),$ and the row-wise population risk
\(
R_t^{(i)}(\alpha_i)
:=
\frac{1}{t}
\inf_{a_i\in\mathbb{R}^{1\times n}}
\sum_{s=0}^{t-1}
\mathbb E
\left[
\left(
\Delta^{\alpha_i}x_{s+1}^{(i)}-a_i x_s
\right)^2
\right].
\)
A global profiled lower isometry bound shows that, uniformly over grid points outside the local neighborhood of $\alpha_{i, \star},$
\[
\underline{\mathcal Q}_{t,i}(\alpha_i)
\ge
\frac t2
\left[
R_t^{(i)}(\alpha_i)
-
R_t^{(i)}(\alpha_{i,\star})
\right].
\]
Together with the upper bound on $\mathcal E_t$ in
\Cref{lem:control_in_sample}, this excludes grid points outside the
local neighborhood and localizes $\hat\alpha_i$ to a set
$\mathcal G_i$ in \eqref{eq:def_mathG_I} around $\alpha_{i,\star}$, where the first-order
expansion of $b_s^{(i)}(\alpha_i)$ has controlled remainder and the
corresponding population derivative curvature $\mu_{t,i}$ is
nondegenerate. Within $\mathcal G_i$, this yields the quadratic
lower isometry bound
\(
\underline{\mathcal Q}_{t,i}(\alpha_i)
\ge
\frac{t\mu_{t,i}}{8}
|\alpha_i-\alpha_{i,\star}|^2.
\)
The following lemma formalizes this local lower isometry property.
\begin{lemma}
\label{lem:Lower_isometry}
Under \Cref{ass:stable_A}, if
\(
t
\gtrsim
\frac{1}{\min_i\underline{\alpha}_{i}^4}
\left(
n
+
\log\frac{\sum_{i=1}^n M_{i}}{\delta}
\right),
\)
then with probability at least $1-\delta$, simultaneously for all
$i\in[n]$ and all $\alpha_i\in\mathcal G_i$,
\(
\underline{\mathcal Q}_{t,i}(\alpha_i)
\ge
\frac{t\mu_{t,i}}{8}
|\alpha_i-\alpha_{i,\star}|^2.
\)
Consequently, on any event for which
$\hat\alpha_i\in\mathcal G_i$ for every $i\in[n]$,
\(
\mathcal Q_{t,i}
\left(
\hat\alpha_i,
\hat a_i-a_{i,\star}
\right)
\ge
\frac{t\mu_{t,i}}{8}
|\hat\alpha_i-\alpha_{i,\star}|^2.
\)
\end{lemma}
We defer the proof of \Cref{lem:Lower_isometry} to \Cref{sec:Lower Isometry for the Row-Wise Grid-Search Estimator}. Combining the lower isometry bound in \Cref{lem:Lower_isometry} with the in-sample upper bound in \Cref{lem:control_in_sample} and choosing the tunable coefficient sufficiently small gives
\[
\left\|
\hat{\boldsymbol\alpha}
-\boldsymbol\alpha_\star
\right\|_\infty^2
\lesssim
\epsilon_{\max}^2+\mathcal O(t^{-1}).
\]
Hence, choosing $\epsilon_{\max}=O(t^{-1/2})$ yields the $t^{-1/2}$ rate in \Cref{thm:vector_alpha_error_bound}.

\subsection{Proof of \texorpdfstring{\Cref{thm:A_error_bound_vec}}{aerrorboundvec}: bounding the error of \texorpdfstring{$A_\star$}{Astar} }
\label{sec:Proof_thm2}

We decompose the identification error of $A$ given \(\boldsymbol{\alpha}\). For any \(\boldsymbol{\alpha}\), we have
\begin{align*}
\hat{A}(\boldsymbol{\alpha}) - A_* 
&= ( B_t(\boldsymbol{\alpha}) + W_t)X_t^\top (X_tX_t^\top)^{-1},
\end{align*}
where $W_t := [\eta_0,\dots,\eta_{t-1}] \in \mathbb{R}^{n \times t}$,
\(
B_t(\boldsymbol{\alpha}) := [\,b_0(\boldsymbol{\alpha}),\ldots,b_{t-1}(\boldsymbol{\alpha})\,] \in \mathbb{R}^{n \times t}, \)
and \(
b_{s}(\boldsymbol{\alpha})
=
\sum_{j=0}^{s}
\bigl(\Psi(\boldsymbol{\alpha},j)-\Psi(\boldsymbol{\alpha}_\star,j)\bigr)\,x_{s-j} =
[b_s^{(1)}(\boldsymbol{\alpha}),\ldots,b_s^{(n)}(\boldsymbol{\alpha})]^\top.
\)
Thus
\begin{equation}
\label{eqn:A_error}
\bigl\|\hat A(\boldsymbol{\alpha})-A_\star\bigr\|_\mathrm{op}
\;\le\;
\underbrace{\bigl\|W_t X_t^\top (X_t X_t^\top)^{-1}\bigr\|_\mathrm{op}}_{\text{noise term}}
\;+\;
\underbrace{\bigl\|B_t(\boldsymbol{\alpha}) X_t^\top (X_t X_t^\top)^{-1}\bigr\|_\mathrm{op}}_{\text{bias term}}.
\end{equation}
We therefore decompose the identification error in \eqref{eqn:A_error} into noise and bias terms, which we bound separately in \Cref{lem:bound_for_noise_term,lem:bound_for_bias_term}.
To bound the noise term, we adapt the technique from \citet{Sarkar18, pmlr-v75-simchowitz18a} and get the following lemma:
\begin{lemma}
\label{lem:bound_for_noise_term}
Fix \(\delta \in (0, \frac{1}{2})\) and consider the system \eqref{eqn:sys_dyna_wo_input} under \Cref{ass:stable_A}. Then there exist universal constants \(c,C>0\) such that
\begin{equation*}
\mathbb{P}\!\left[
\bigl\|W_t X_t^\top (X_t X_t^\top)^{-1}\bigr\|_\mathrm{op}
>
\frac{C}{\sqrt{t\,\lambda_{\min}(\Gamma_k)}}
\sqrt{
n\log\frac{n}{\delta}
+
\log\det(\Gamma_t\Gamma_k^{-1})
}
\right]
\le \delta,
\end{equation*}
for any \(k\) such that
\(
\frac{t}{k}
\ge
c\!\left(
n\log(n/\delta)
+
\log\det(\Gamma_t\Gamma_k^{-1})
\right)
\)
holds.
    
\end{lemma}
We defer the proof of \Cref{lem:bound_for_noise_term} to \Cref{pf:bound_for_noise}. We now offer the bound for the bias term in \eqref{eqn:A_error}:

\begin{lemma}
\label{lem:bound_for_bias_term}
With probability at least \(1- 2\delta\), the following holds
\begin{equation*}
\bigl\|B_t(\boldsymbol{\alpha}) X_t^\top (X_t X_t^\top)^{-1}\bigr\|_\mathrm{op} 
\le
S_1 \|\boldsymbol{\alpha}-\boldsymbol{\alpha}_\star\|_\infty   \sqrt{\frac{320n\tilde{C}_G^2}{9\delta {p^2} \lambda_{\min}(\Gamma_{{\lfloor k/2\rfloor}})}},
\end{equation*}
where $\alpha_{\min} = \min_{1 \le i \le n} \alpha_{i, \star}$, and $p = \frac{3}{20}$.
\end{lemma}
We defer the proof of \Cref{lem:bound_for_bias_term} to \Cref{sec:bound_for_bias}.

\section{Experiments}
\label{sec:expreiment}
We evaluate the proposed method \emph{FO-GS} through two sets of experiments. First, we use synthetic data to validate the theoretical guarantee and compare \emph{FO-GS} with existing fractional-order identification algorithms~\citep{chatterjee2022learning,flandrin2002wavelet}. \emph{FO-WT} estimates the fractional order \(\boldsymbol{\alpha}_\star\) using a wavelet-based technique~\citep{flandrin2002wavelet} and then applies OLS to estimate the system matrix \(A_\star\). \emph{FO-BS} estimates \(\boldsymbol{\alpha}_\star\) via binary search and then identifies \(A_\star\) by applying OLS to an augmented state representation obtained through system truncation~\citep{chatterjee2022learning}. Second, as fractional-order systems have been used in analyzing electroencephalogram (EEG) data, we further demonstrate that \emph{FO-GS} excels in minimizing the one-step normalized mean squared error (NMSE) on both training and testing datasets in comparison to both existing methods (\emph{FO-BS} and \emph{FO-WT}). Both the synthetic and real-world experiments demonstrate that \emph{FO-GS} outperforms the baselines. We provide additional experimental details in \Cref{sec:experiment_detail}.

\subsection{Performance Evaluation on Synthetic Data}
In this section, we compare \emph{FO-GS} with two existing baselines on synthetic data. The trajectories are generated according to the fractional-order dynamics in \eqref{eq:folti_no_input}. The ground truth parameters \(\boldsymbol{\alpha}_\star\) and \(A_\star\) are randomly sampled, and the reported mean squared error (MSE) is averaged over five randomly generated system instances, with 20 independent rollouts for each instance.

\textbf{Varying trajectory horizons.} We evaluate the MSE of the fractional order \(\boldsymbol{\alpha}_\star\) and the system matrix \(A_\star\) as functions of the trajectory horizon \(t\) by fixing the noise scale $\sigma$. As shown in \Cref{fig:alpha_vec_A_mse_comp_baseline}, \emph{FO-GS} outperforms baselines in estimating both \(\boldsymbol{\alpha}_\star\) and \(A_\star\) across all trajectory lengths. The improvement is particularly pronounced for shorter horizons, where accurate identification is most challenging. A plausible reason is that our method directly fits the original fractional-order model and exploits its structural decomposition, whereas the baselines rely either on a wavelet-based proxy for estimating the fractional order \(\boldsymbol{\alpha}_\star\) or on a truncated lifted state approximation for identifying the system matrix \(A_\star\). These additional approximation steps can introduce non-negligible finite-sample error, especially when the available trajectory is short. We also observe that the estimation error generally decreases as \(t\) increases, which is consistent with the theoretical predictions in \Cref{thm:vector_alpha_error_bound,thm:A_error_bound_vec}.

\textbf{Varying noise scales.} We examines how the MSE of the fractional order \(\boldsymbol{\alpha}_\star\) and the system matrix \(A_\star\) varies with the noise scale \(\sigma\) given the same horizon $t$. \Cref{fig:alpha_vec_A_mse_comp_baseline} shows that \emph{FO-GS} consistently outperforms the two baselines across all noise levels. As expected, the MSE of all methods increases as the noise level grows, but the \emph{FO-GS} remains the most robust, likely because it estimates \(\boldsymbol{\alpha}_\star\) and \(A_\star\) directly from the original fractional-order model, whereas the baselines incur additional approximation error through wavelet-based estimation or system truncation.

\textbf{Varying grid sizes.} We study how the number of grids \(M\) affects the performance of \emph{FO-GS} (\Cref{fig:alpha_vec_A_mse_comp_baseline}). Consistent with \Cref{thm:vector_alpha_error_bound,thm:A_error_bound_vec}, the MSE for both the fractional order \(\boldsymbol{\alpha}\) and the system matrix \(A\) decreases as the number of grids increases. Notably, \emph{FO-GS} outperforms both \emph{FO-BS} and \emph{FO-WT} without requiring a large number of grids: it surpasses the baselines in estimating \(\boldsymbol{\alpha}_\star\) with roughly ten grids and in estimating \(A_\star\) with roughly five grids. This demonstrates that \emph{FO-GS} is not only accurate but also computationally efficient.

\textbf{Small $\boldsymbol{\alpha}_\star$ regime.}
We further examine the small fractional-order regime by uniformly sampling
\(\boldsymbol{\alpha}_\star\) from the interval \([0.01,0.2]\) for systems with \(n=10\) and \(n=20\), and extending the trajectory horizon up to \(t=25{,}600\). As shown in \Cref{tab:small_alpha_short_horizon}, over the shorter horizons \(t\in\{100,200,300,400\}\), the empirical convergence is slower than the predicted \(t^{-1/2}\) rate, indicating stronger finite-sample effects when the fractional orders are small. As the trajectory length increases, however, the fitted rates become progressively faster. Specifically, \Cref{tab:small_alpha_long_horizon} shows that over the full extended horizon, the rate scales approximately as \(t^{-0.36}\) for \(n=10\) and \(t^{-0.38}\) for \(n=20\), while fitting only the larger horizon regime \(t\ge 3200\) yields rates of approximately \(t^{-0.43}\) and \(t^{-0.42}\), respectively. These results show a clear trend toward the \(t^{-1/2}\) rate predicted by \Cref{thm:vector_alpha_error_bound} as the trajectory becomes longer. We emphasize that the sample-size requirement in \Cref{thm:vector_alpha_error_bound}, including its explicit \(\alpha_{\min}^{-4}\) dependence, is a sufficient condition for entering the fast rate regime rather than a necessary or optimal threshold. Thus, the observed finite-sample behavior may be better than what is implied by the conservative sufficient condition.

\begin{figure}[t]
 \centering
 \includegraphics[width=0.98\textwidth]{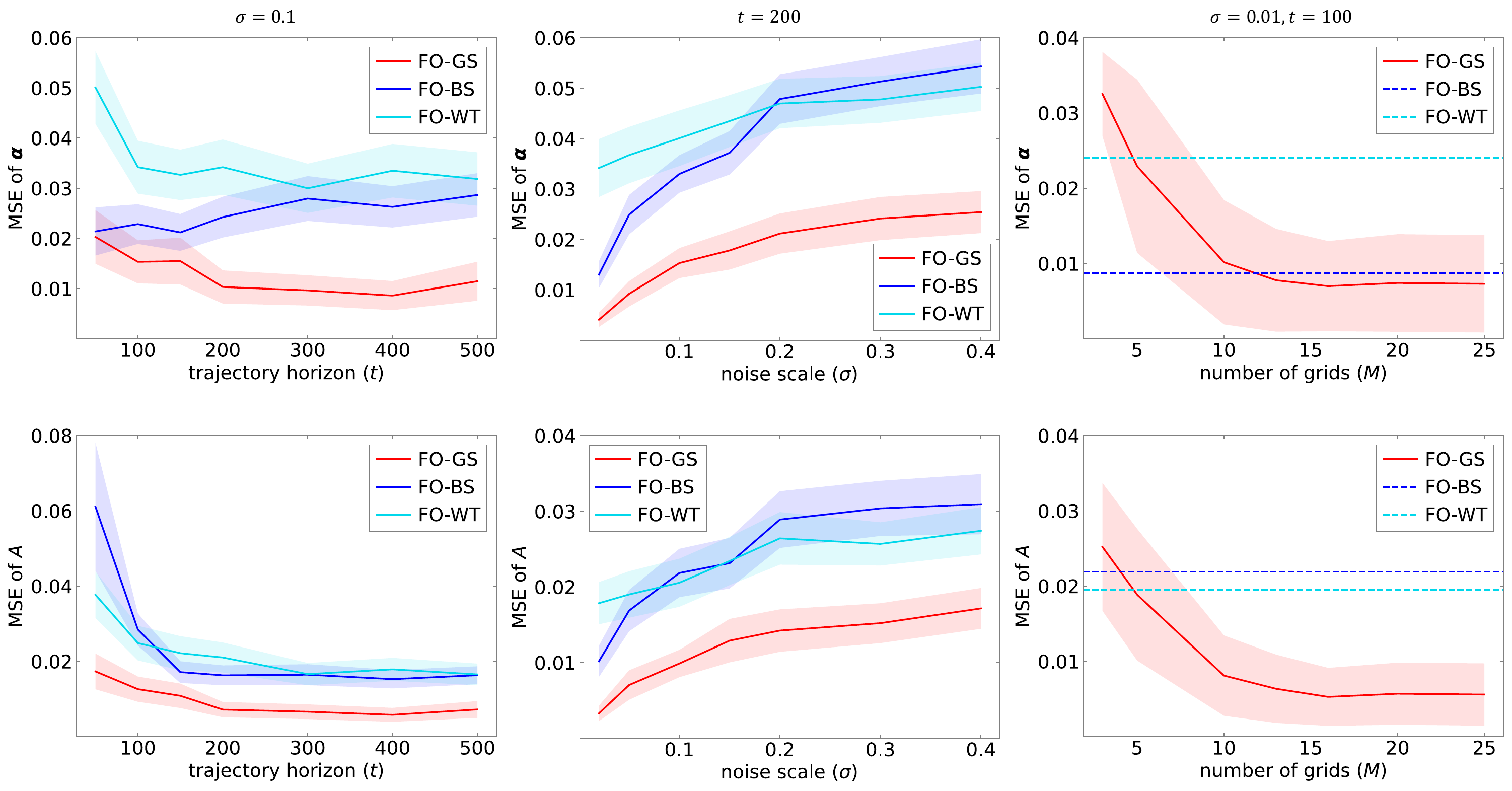}
 \caption{Comparison of the MSE on synthetic FOLTI system identification. Shaded regions indicate $95\%$ confidence intervals (CI).}
 \label{fig:alpha_vec_A_mse_comp_baseline}
\end{figure}

\begin{figure}[t]
 \centering
 \includegraphics[width=0.98\textwidth]{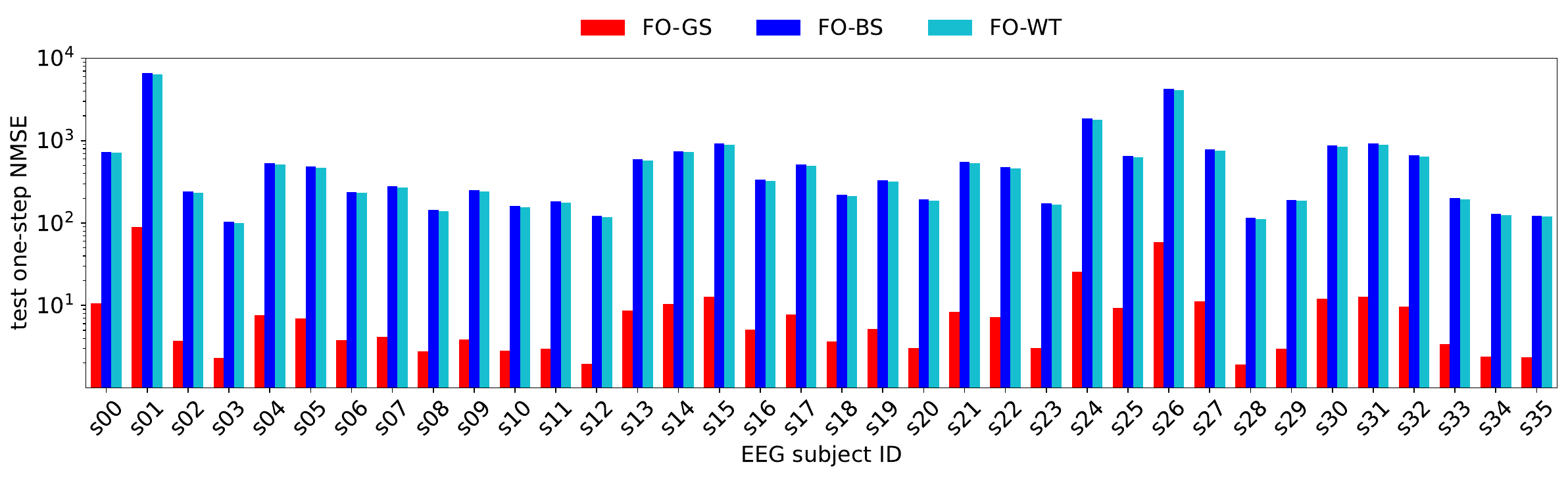}
 \caption{Subject-level average test one-step NMSE on the EEG mental-arithmetic dataset. Each point corresponds to one subject and is obtained by averaging the window-level test one-step NMSE over all non-overlapping windows from that subject.}
 \label{fig:EEG_experiment_comp}
\end{figure}
\begin{figure}[t]
 \centering
 \includegraphics[width=0.98\textwidth]{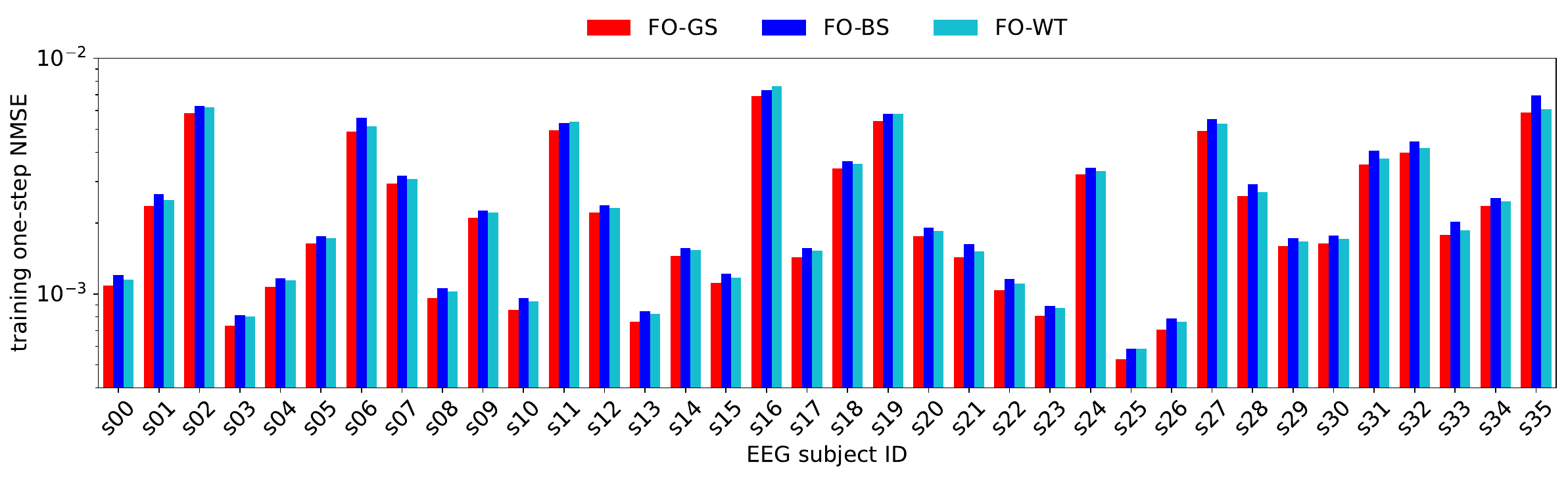}
 \caption{Subject-level average training one-step NMSE on the EEG mental-arithmetic dataset. Each point corresponds to one subject and is obtained by averaging the window-level training one-step NMSE over all non-overlapping windows from that subject.}
 \label{fig:EEG_train_experiment_comp}
\end{figure}

\begin{table}[t]
\centering
\caption{Short horizon convergence rates from log--log fits.
We report a 95\% bootstrap confidence interval for each MSE exponent.}
\label{tab:small_alpha_short_horizon}
\small
\renewcommand{\arraystretch}{1}
\setlength{\tabcolsep}{9pt}

\begin{tabular}{lccc}
\toprule
$\boldsymbol{\alpha}_\star$ regime
& MSE exponent (95\% CI)
& $R^2$
& RMSE rate \\
\midrule
Small $\boldsymbol{\alpha}_\star$ ($n=10$)
& $-0.505^{+0.096}_{-0.091}$
& $0.998$
& $t^{-0.253}$ \\

Small $\boldsymbol{\alpha}_\star$ ($n=20$)
& $-0.571^{+0.101}_{-0.093}$
& $0.999$
& $t^{-0.286}$ \\

Large $\boldsymbol{\alpha}_\star$ ($n=2$)
& $-1.129^{+0.348}_{-0.362}$
& $0.973$
& $t^{-0.564}$ \\
\bottomrule
\end{tabular}
\end{table}

\begin{table}[!htbp]
\centering
\caption{Long horizon convergence rates from log--log fits.
Results are reported for the small $\boldsymbol{\alpha}_\star$
regime. ``All horizons'' refers to fits using all time points.}
\label{tab:small_alpha_long_horizon}
\small
\renewcommand{\arraystretch}{1}
\setlength{\tabcolsep}{9pt}

\begin{tabular}{clcc}
\toprule
Dimension
& Fitting window
& MSE exponent (95\% CI)
& RMSE rate \\
\midrule
\multirow[c]{2}{*}{$n=10$}
& All horizons
& $-0.723^{+0.023}_{-0.024}$
& $t^{-0.362}$ \\
& $t \geq 3200$
& $-0.868^{+0.089}_{-0.085}$
& $t^{-0.434}$ \\

\multirow[c]{2}{*}{$n=20$}
& All horizons
& $-0.758^{+0.013}_{-0.013}$
& $t^{-0.379}$ \\
& $t \geq 3200$
& $-0.848^{+0.056}_{-0.058}$
& $t^{-0.424}$ \\
\bottomrule
\end{tabular}
\end{table}

\subsection{Performance Evaluation on Real-World Data}

We evaluate \emph{FO-GS} on an EEG mental-arithmetic dataset~\citep{data4010014} comprising artifact-free recordings from \(36\) subjects, sampled at \(500\)~Hz with a Neurocom \(23\)-channel system and \(19\) electrodes placed obeying the International \(10/20\) scheme. We use the first minute of the serial-subtraction task, treating each subject’s recording as a \(19\)-dimensional time series. The data are segmented into non-overlapping windows of length \(W=150\), with \(70\%\) of samples used for training and \(30\%\) for testing. We compare \emph{FO-GS} against \emph{FO-BS} and \emph{FO-WT} using one-step NMSE. \Cref{fig:EEG_experiment_comp} reports subject-level average test NMSE. Training errors are similar across methods as shown in \Cref{fig:EEG_train_experiment_comp}, but test performance differs markedly: \emph{FO-GS} consistently achieves the lowest NMSE for all subjects, while \emph{FO-BS} and \emph{FO-WT} incur higher errors. This improvement stems from the fractional-order identification in \emph{FO-GS}, where accurate estimation of \(\boldsymbol{\alpha}\) yields better history weighting and prediction. In contrast, \emph{FO-BS} uses a fixed finite-memory approximation and \emph{FO-WT} estimates the order via a separate wavelet-based step. Overall, \emph{FO-GS} delivers superior predictive performance.

\textbf{Learned-order interpretation.} We also examine the learned fractional orders on the EEG dataset and find clear evidence of non-integer, channel-dependent memory. For \textit{FO-GS}, the median window-averaged order is 0.855; 74.6\% of windows have a cross-channel order range greater than 0.5, and 72.8\% contain at least one channel with $\alpha < 0.1$. These results suggest substantial heterogeneity in long-term memory across EEG channels.

\section{Conclusion}
\label{sec:conclusion}
We study the 
identification of FOLTI systems from a single observed trajectory and propose \emph{FO-GS}, a simple two-stage estimator that exploits the diagonal structure of the Grünwald--Letnikov difference operator to decouple the estimation of the fractional order \(\boldsymbol{\alpha}_\star\) and the system matrix \(A_\star\) row-wise. Under the stability assumption, we show that \emph{FO-GS} admits high-probability non-asymptotic error guarantees for recovering both $\boldsymbol{\alpha}_\star$ and $A_\star$. 
\emph{FO-GS} outperforms existing baselines on both synthetic and EEG data. These results indicate that direct single-trajectory identification of FOLTI systems is both statistically analyzable and practically effective despite the non-Markovian system dynamics.



\textbf{Limitations and future work.} The convergence rate for $\boldsymbol{\alpha}_\star$ may be suboptimal. Our finite-sample condition is sufficient for achieving the fast $t^{-1/2}$ rate, but grid search introduces an additional $\log t$ factor in the bound. Our EEG evaluation includes only two FOLTI-based baselines, and appropriately tuned classical long-memory \textit{ARFIMA} models~\citep{https://doi.org/10.1111/j.1467-9892.1980.tb00297.x} can outperform \textit{FO-GS}. Future work should relax the stability assumption and sharpen the $\boldsymbol{\alpha}_\star$ dependence of the convergence bound.

\clearpage

\begin{ack}
Stephen Tu acknowledges support from the National Science Foundation (NSF) under NSF CPS 2434460. Guannan Qu was supported by NSF Grants 2339112, 2512805, AFOSR Grant FA9550261B281, Jane Street, Pennsylvania Infrastructure Technology Alliance, CMU Manufacturing Futures Institute. Yorie Nakahira was supported by NSF Grant No. ECCS-2442948 (CAREER) and Japan Science and Technology Agency (JST): Grant No. JPMJPR2516 (PRESTO). Paul Bogdan was supported by the National Science Foundation under NSF Award 2243104 under the Center for Complex Particle Systems (COMPASS), CAREER Award CPS-1453860, NSF Mid-Career Advancement Award BCS-2527046, U.S. Army Research Office (ARO) under Grant No. W911NF-23-1-0111, Defense Advanced Research Projects Agency (DARPA) Young Faculty Award and DARPA Director Fellowship Award under Grant Number N66001-17-1-4044, DARPA MATHBAC grant HR0011263E153, Intel Faculty Awards, and a Northrop Grumman grant. P.B. is also grateful to the National Institutes of Health (NIH) for Grants R01 AG079957 and RF1 AG082201. The views, opinions, and/or findings in this article are those of the authors and should not be interpreted as official views or policies of the Department of Defense, the National Institutes of Health, or the National Science Foundation.
\end{ack}

\bibliographystyle{plainnat}
\bibliography{example_paper}

@book{monje2010fractional,
  title={Fractional-order systems and controls: fundamentals and applications},
  author={Monje, Concepci{\'o}n A and Chen, YangQuan and Vinagre, Blas M and Xue, Dingyu and Feliu-Batlle, Vicente},
  year={2010},
  publisher={Springer Science \& Business Media}
}

@inproceedings{gupta2019learning,
  title={Learning latent fractional dynamics with unknown unknowns},
  author={Gupta, Gaurav and Pequito, S{\'e}rgio and Bogdan, Paul},
  booktitle={2019 American Control Conference (ACC)},
  pages={217--222},
  year={2019},
  organization={IEEE}
}

@book{hilfer2000applications,
  title={Applications of fractional calculus in physics},
  author={Hilfer, Rudolf},
  year={2000},
  publisher={World scientific}
}

@article{ionescu2017role,
  title={The role of fractional calculus in modeling biological phenomena: A review},
  author={Ionescu, C and Lopes, A and Copot, Dana and Machado, JA Tenreiro and Bates, Jason HT},
  journal={Communications in Nonlinear Science and Numerical Simulation},
  volume={51},
  pages={141--159},
  year={2017},
  publisher={Elsevier}
}

@inproceedings{yaghooti2023inferring,
  title={Inferring dynamics of discrete-time, fractional-order control-affine nonlinear systems},
  author={Yaghooti, Bahram and Sinopoli, Bruno},
  booktitle={2023 American Control Conference (ACC)},
  pages={935--940},
  year={2023},
  organization={IEEE}
}

@inproceedings{chatterjee2022learning,
  title={On learning discrete-time fractional-order dynamical systems},
  author={Chatterjee, Sarthak and Pequito, S{\'e}rgio},
  booktitle={2022 American Control Conference (ACC)},
  pages={4335--4340},
  year={2022},
  organization={IEEE}
}

@book{oldham1974fractional,
  title={The fractional calculus theory and applications of differentiation and integration to arbitrary order},
  author={Oldham, Keith and Spanier, Jerome},
  year={1974},
  publisher={Elsevier}
}

@article{guermah2012discrete,
  title={Discrete-time fractional-order systems: Modeling and stability issues},
  author={Guermah, Sa{\"\i}d and Djennoune, Sa{\"\i}d and Bettayeb, Ma{\^a}mar},
  journal={Advances in Discrete Time Systems},
  pages={183--212},
  year={2012},
  publisher={InTech Rijeka}
}

@INPROCEEDINGS{11107451,
  author={Zhang, Xiaole and Gupta, Vijay and Bogdan, Paul},
  booktitle={2025 American Control Conference (ACC)}, 
  title={A Sampling Complexity-aware Framework for Discrete-time Fractional-Order Dynamical System Identification}, 
  year={2025},
  volume={},
  number={},
  pages={5093-5098},
  doi={10.23919/ACC63710.2025.11107451}}

@article{lundstrom2008fractional,
  title={Fractional differentiation by neocortical pyramidal neurons},
  author={Lundstrom, Brian N and Higgs, Matthew H and Spain, William J and Fairhall, Adrienne L},
  journal={Nature neuroscience},
  volume={11},
  number={11},
  pages={1335--1342},
  year={2008},
  publisher={Nature Publishing Group US New York}
}

@article{ivanov1999multifractality,
  title={Multifractality in human heartbeat dynamics},
  author={Ivanov, Plamen Ch and Amaral, Luis A Nunes and Goldberger, Ary L and Havlin, Shlomo and Rosenblum, Michael G and Struzik, Zbigniew R and Stanley, H Eugene},
  journal={Nature},
  volume={399},
  number={6735},
  pages={461--465},
  year={1999},
  publisher={Nature Publishing Group UK London}
}

@inproceedings{
zhang2025endtoend,
title={End-to-End Learning Framework for Solving Non-Markovian Optimal Control},
author={Xiaole Zhang and Peiyu Zhang and Xiongye Xiao and Shixuan Li and Vasileios Tzoumas and Vijay Gupta and Paul Bogdan},
booktitle={Forty-second International Conference on Machine Learning},
year={2025},
url={https://openreview.net/forum?id=1k4dKH1XOz}
}

@inproceedings{Sarkar18,
  title={Near optimal finite time identification of arbitrary linear dynamical systems},
  author={Tuhin Sarkar and Alexander Rakhlin},
  booktitle={International Conference on Machine Learning},
  year={2018}
}

@article{willinger2003long,
  title={Long-range dependence and data network traffic},
  author={Willinger, Walter and Paxson, Vern and Riedi, Rolf H and Taqqu, Murad S},
  journal={Theory and applications of long-range dependence},
  pages={373--407},
  year={2003},
  publisher={Birkh{\"a}user Boston}
}

@inproceedings{
zhang2025stabilizing,
title={Stabilizing {LTI} Systems under Partial Observability: Sample Complexity and Fundamental Limits},
author={Ziyi Zhang and Yorie Nakahira and Guannan Qu},
booktitle={The Thirty-ninth Annual Conference on Neural Information Processing Systems},
year={2025},
url={https://openreview.net/forum?id=KwHsZJatB8}
}

@inproceedings{
zhang2025learning,
title={Learning to Stabilize Unknown {LTI} Systems on a Single Trajectory under Stochastic Noise},
author={Ziyi Zhang and Yorie Nakahira and Guannan Qu},
booktitle={The 41st Conference on Uncertainty in Artificial Intelligence},
year={2025},
url={https://openreview.net/forum?id=duNaSFJ1sF}
}

@inproceedings{
alabdulmohsin2025a,
title={A Tale of Two Structures: Do {LLM}s Capture the Fractal Complexity of Language?},
author={Ibrahim Alabdulmohsin and Andreas Peter Steiner},
booktitle={Forty-second International Conference on Machine Learning},
year={2025},
url={https://openreview.net/forum?id=p2smPMRQae}
}

@book{podlubny1998fractional,
  title={Fractional differential equations: an introduction to fractional derivatives, fractional differential equations, to methods of their solution and some of their applications},
  author={Podlubny, Igor},
  volume={198},
  year={1998},
  publisher={elsevier}
}

@article{Oymak18,
  title={Non-asymptotic Identification of LTI Systems from a Single Trajectory},
  author={Samet Oymak and Necmiye Ozay},
  journal={2019 American Control Conference (ACC)},
  year={2018},
  pages={5655-5661}
}

@InProceedings{Sun20,
  title = 	 {Finite Sample System Identification: Optimal Rates and the Role of Regularization},
  author =       {Sun, Yue and Oymak, Samet and Fazel, Maryam},
  booktitle = 	 {Proceedings of the 2nd Conference on Learning for Dynamics and Control},
  pages = 	 {16--25},
  year = 	 {2020},
  editor = 	 {Bayen, Alexandre M. and Jadbabaie, Ali and Pappas, George and Parrilo, Pablo A. and Recht, Benjamin and Tomlin, Claire and Zeilinger, Melanie},
  volume = 	 {120},
  series = 	 {Proceedings of Machine Learning Research},
  month = 	 {10--11 Jun},
  publisher =    {PMLR},
}

@article{Zheng201,
author = {Zheng, Yang and Li, Na},
year = {2020},
month = {12},
pages = {1-1},
title = {Non-Asymptotic Identification of Linear Dynamical Systems Using Multiple Trajectories},
volume = {PP},
journal = {IEEE Control Systems Letters},
doi = {10.1109/LCSYS.2020.3042924}
}

@article{rivero2013stability,
  title={Stability of fractional order systems},
  author={Rivero, Margarita and Rogosin, Sergei V and Tenreiro Machado, Jose A and Trujillo, Juan J},
  journal={Mathematical Problems in Engineering},
  volume={2013},
  number={1},
  pages={356215},
  year={2013},
  publisher={Wiley Online Library}
}

@incollection{petravs2021stability,
  title={Stability of fractional-order systems},
  author={Petr{\'a}{\v{s}}, Ivo},
  booktitle={Fractional-Order Nonlinear Systems: Modeling, Analysis and Simulation},
  pages={55--101},
  year={2021},
  publisher={Springer}
}

@article{sarkar2021finite,
  title={Finite time LTI system identification},
  author={Sarkar, Tuhin and Rakhlin, Alexander and Dahleh, Munther A},
  journal={Journal of Machine Learning Research},
  volume={22},
  number={26},
  pages={1--61},
  year={2021}
}

@misc{jedra2020finitetimeidentificationstablelinear,
      title={Finite-time Identification of Stable Linear Systems: Optimality of the Least-Squares Estimator}, 
      author={Yassir Jedra and Alexandre Proutiere},
      year={2020},
      eprint={2003.07937},
      archivePrefix={arXiv},
      primaryClass={math.ST},
      url={https://arxiv.org/abs/2003.07937},
}

@misc{oymak2019nonasymptoticidentificationltisystems,
      title={Non-asymptotic Identification of LTI Systems from a Single Trajectory}, 
      author={Samet Oymak and Necmiye Ozay},
      year={2019},
      eprint={1806.05722},
      archivePrefix={arXiv},
      primaryClass={cs.LG},
      url={https://arxiv.org/abs/1806.05722}, 
}

@InProceedings{pmlr-v75-simchowitz18a,
  title = 	 {Learning Without Mixing: Towards A Sharp Analysis of Linear System Identification},
  author =       {Simchowitz, Max and Mania, Horia and Tu, Stephen and Jordan, Michael I. and Recht, Benjamin},
  booktitle = 	 {Proceedings of the 31st  Conference On Learning Theory},
  pages = 	 {439--473},
  year = 	 {2018},
  editor = 	 {Bubeck, Sébastien and Perchet, Vianney and Rigollet, Philippe},
  volume = 	 {75},
  series = 	 {Proceedings of Machine Learning Research},
  month = 	 {06--09 Jul},
  publisher =    {PMLR},
  url = 	 {https://proceedings.mlr.press/v75/simchowitz18a.html}
}

@article{flandrin2002wavelet,
  title={Wavelet analysis and synthesis of fractional Brownian motion},
  author={Flandrin, Patrick},
  journal={IEEE Transactions on information theory},
  volume={38},
  number={2},
  pages={910--917},
  year={2002},
  publisher={IEEE}
}

@Article{data4010014,
AUTHOR = {Zyma, Igor and Tukaev, Sergii and Seleznov, Ivan and Kiyono, Ken and Popov, Anton and Chernykh, Mariia and Shpenkov, Oleksii},
TITLE = {Electroencephalograms during Mental Arithmetic Task Performance},
JOURNAL = {Data},
VOLUME = {4},
YEAR = {2019},
NUMBER = {1},
ARTICLE-NUMBER = {14},
URL = {https://www.mdpi.com/2306-5729/4/1/14},
ISSN = {2306-5729},
DOI = {10.3390/data4010014}
}

@book{1971712334804565506,
author="Miller, Kenneth S. and Ross, Bertram",
title="An introduction to the fractional calculus and fractional differential equations",
publisher="Wiley",
year="1993",
series="A Wiley-Interscience publication",
URL="https://cir.nii.ac.jp/crid/1971712334804565506"
}

@article{Flajolet1990SingularityAO,
  title={Singularity Analysis of Generating Functions},
  author={Philippe Flajolet and Andrew M. Odlyzko},
  journal={SIAM Journal on Discrete Mathematics},
  year={1990},
  volume={3},
  number={2},
  pages={216-240},
  doi={10.1137/0403019},
  url={https://epubs.siam.org/doi/10.1137/0403019}
}

@article{article,
author = {Linkenkaer-Hansen, Klaus and Nikulin, Vadim and Palva, J. Matias and Ilmoniemi, Risto},
year = {2001},
month = {03},
pages = {1370-7},
title = {Long-Range Temporal Correlations and Scaling Behavior in Human Brain Oscillations},
volume = {21},
journal = {The Journal of neuroscience : the official journal of the Society for Neuroscience},
doi = {10.1523/JNEUROSCI.21-04-01370.2001}
}

@article{DING199383,
title = {A long memory property of stock market returns and a new model},
journal = {Journal of Empirical Finance},
volume = {1},
number = {1},
pages = {83-106},
year = {1993},
issn = {0927-5398},
doi = {https://doi.org/10.1016/0927-5398(93)90006-D},
url = {https://www.sciencedirect.com/science/article/pii/092753989390006D},
author = {Zhuanxin Ding and Clive W.J. Granger and Robert F. Engle}
}

@article{Leland1993OnTS,
  title={On the self-similar nature of Ethernet traffic},
  author={Will E. Leland and Walter Willinger and Murad S. Taqqu and Daniel V. Wilson},
  journal={Comput. Commun. Rev.},
  year={1993},
  volume={25},
  pages={202-213},
  url={https://api.semanticscholar.org/CorpusID:6011907}
}

@article{https://doi.org/10.1111/1468-0262.00418,
author = {Andersen, Torben G. and Bollerslev, Tim and Diebold, Francis X. and Labys, Paul},
title = {Modeling and Forecasting Realized Volatility},
journal = {Econometrica},
volume = {71},
number = {2},
pages = {579-625},
doi = {https://doi.org/10.1111/1468-0262.00418},
url = {https://onlinelibrary.wiley.com/doi/abs/10.1111/1468-0262.00418},
eprint = {https://onlinelibrary.wiley.com/doi/pdf/10.1111/1468-0262.00418},
year = {2003}
}

@article{
doi:10.1073/pnas.012579499,
author = {Ary L. Goldberger  and Luis A. N. Amaral  and Jeffrey M. Hausdorff  and Plamen Ch. Ivanov  and C.-K. Peng  and H. Eugene Stanley },
title = {Fractal dynamics in physiology: Alterations with disease and aging},
journal = {Proceedings of the National Academy of Sciences},
volume = {99},
number = {suppl\_1},
pages = {2466-2472},
year = {2002},
doi = {10.1073/pnas.012579499},
URL = {https://www.pnas.org/doi/abs/10.1073/pnas.012579499},
eprint = {https://www.pnas.org/doi/pdf/10.1073/pnas.012579499}}

@inproceedings{
alabdulmohsin2024fractal,
title={Fractal Patterns May Illuminate the Success of Next-Token Prediction},
author={Ibrahim Alabdulmohsin and Vinh Q. Tran and Mostafa Dehghani},
booktitle={The Thirty-eighth Annual Conference on Neural Information Processing Systems},
year={2024},
url={https://openreview.net/forum?id=clAFYReaYE}
}

@article{doi:10.1061/TACEAT.0006518,
author = {H. E. Hurst },
title = {Long-Term Storage Capacity of Reservoirs},
journal = {Transactions of the American Society of Civil Engineers},
volume = {116},
number = {1},
pages = {770-799},
year = {1951},
doi = {10.1061/TACEAT.0006518},

URL = {https://ascelibrary.org/doi/abs/10.1061/TACEAT.0006518},
eprint = {https://ascelibrary.org/doi/pdf/10.1061/TACEAT.0006518}
}

@article{https://doi.org/10.1029/WR005i002p00321,
author = {Mandelbrot, Benoit B. and Wallis, James R.},
title = {Some long-run properties of geophysical records},
journal = {Water Resources Research},
volume = {5},
number = {2},
pages = {321-340},
doi = {https://doi.org/10.1029/WR005i002p00321},
url = {https://agupubs.onlinelibrary.wiley.com/doi/abs/10.1029/WR005i002p00321},
eprint = {https://agupubs.onlinelibrary.wiley.com/doi/pdf/10.1029/WR005i002p00321},
year = {1969}
}

@article{hsu2012tail,
  title={A tail inequality for quadratic forms of subgaussian random vectors},
  author={Hsu, Daniel and Kakade, Sham M and Zhang, Tong},
  journal={Electronic Communications in Probability},
  volume={17},
  year={2012},
  publisher={Institute of Mathematical Statistics}
}

@InProceedings{pmlr-v178-ziemann22a,
  title = 	 {Single Trajectory Nonparametric Learning of Nonlinear Dynamics},
  author =       {Ziemann, Ingvar M and Sandberg, Henrik and Matni, Nikolai},
  booktitle = 	 {Proceedings of Thirty Fifth Conference on Learning Theory},
  pages = 	 {3333--3364},
  year = 	 {2022},
  editor = 	 {Loh, Po-Ling and Raginsky, Maxim},
  volume = 	 {178},
  series = 	 {Proceedings of Machine Learning Research},
  month = 	 {02--05 Jul},
  publisher =    {PMLR},
  url = 	 {https://proceedings.mlr.press/v178/ziemann22a.html}
}

@article{https://doi.org/10.1111/j.1467-9892.1980.tb00297.x,
author = {Granger, C. W. J. and Joyeux, Roselyne},
title = {AN INTRODUCTION TO LONG-MEMORY TIME SERIES MODELS AND FRACTIONAL DIFFERENCING},
journal = {Journal of Time Series Analysis},
volume = {1},
number = {1},
pages = {15-29},
doi = {https://doi.org/10.1111/j.1467-9892.1980.tb00297.x},
url = {https://onlinelibrary.wiley.com/doi/abs/10.1111/j.1467-9892.1980.tb00297.x},
eprint = {https://onlinelibrary.wiley.com/doi/pdf/10.1111/j.1467-9892.1980.tb00297.x},
year = {1980}
}

@INPROCEEDINGS{6659004,
  author={Ghorbani, Mahboobeh and Bogdan, Paul},
  booktitle={2013 International Conference on Hardware/Software Codesign and System Synthesis (CODES+ISSS)}, 
  title={A cyber-physical system approach to artificial pancreas design}, 
  year={2013},
  volume={},
  number={},
  pages={1-10},
  doi={10.1109/CODES-ISSS.2013.6659004}}

@INPROCEEDINGS{7459389,
  author={Xue, Yuankun and Rodriguez, Saul and Bogdan, Paul},
  booktitle={2016 Design, Automation Test in Europe Conference Exhibition (DATE)}, 
  title={A spatio-temporal fractal model for a CPS approach to brain-machine-body interfaces}, 
  year={2016},
  volume={},
  number={},
  pages={642-647},
  doi={}}

@ARTICLE{10.3389/fphys.2018.01446,
    
AUTHOR={Ghorbani, Mahboobeh  and Jonckheere, Edmond A.  and Bogdan, Paul },
           
TITLE={Gene Expression Is Not Random: Scaling, Long-Range Cross-Dependence, and Fractal Characteristics of Gene Regulatory Networks},
          
JOURNAL={Frontiers in Physiology},
          
VOLUME={Volume 9 - 2018},
  
YEAR={2018},
  
URL={https://www.frontiersin.org/journals/physiology/articles/10.3389/fphys.2018.01446},
  
DOI={10.3389/fphys.2018.01446},
  
ISSN={1664-042X}}
\clearpage
\appendix
\renewcommand{\appendixname}{Appendix}
\renewcommand{\appendixtocname}{Appendix}
\renewcommand{\appendixpagename}{Appendix}
\section*{Content}
\addcontentsline{toc}{section}{Appendix}

\startcontents[appendix]
\etocsetnexttocdepth{subsection}
\printcontents[appendix]{l}{1}{\setcounter{tocdepth}{3}}

\section{Experiment Details}
\label[appendix]{sec:experiment_detail}
The code to reproduce our experiments is available at: \url{https://github.com/Eric0312/FO-GS}.
\subsection{Synthetic Experiments}
\label{sec:synthetic_exp}
We generate trajectories from two-dimensional fractional-order LTI systems. For each synthetic system, $\alpha_{i,\star}$ are sampled independently and uniformly from $[0.1,0.5]$, and $A_\star$ is generated with eigenvalues sampled uniformly from $[-0.5,0.5]$. We compare \emph{FO-GS} with two baselines. \emph{FO-BS} uses a truncated lifted state representation with memory length $p=40$, binary search tolerance $10^{-2}$, search interval $[0.05,0.55]$, and ridge parameter $10^{-6}$. \emph{FO-WT} uses Haar wavelets with minimum level $2$, maximum level chosen automatically, linear detrending, and ridge parameter $10^{-6}$ for the subsequent OLS step. Unless otherwise specified, \emph{FO-GS} searches over $[0.05,0.55]$ using $20$ equally spaced grid points per coordinate and ridge parameter $10^{-6}$. We report the MSE of both $\boldsymbol{\alpha}_\star$ and $A_\star$, averaged over $5$ matched systems and $20$ Monte Carlo trials.

\textbf{Varying trajectory horizons.}
To study the effect of trajectory length, we vary the horizon over
$t\in\{50,100,150,200,300,400,500\}$ while fixing the noise scale to $\sigma=0.1$. The initial state is sampled from a zero mean Gaussian distribution with standard deviation $4.0$.

\textbf{Varying noise scales.}
To evaluate robustness to process noise, we fix the trajectory horizon at $t=200$ and vary the noise scale over
$\sigma\in\{0.02,0.05,0.10,0.15,0.20,0.30,0.40\}$. The initial state is sampled from a zero mean Gaussian distribution with standard deviation $4.0$.

\textbf{Varying grid sizes.}
To examine the effect of grid resolution in \emph{FO-GS}, we fix $t=100$ and $\sigma=0.01$ and vary the number of grid points over
$M\in\{3,5,10,13,16,20,25\}$. The initial state is sampled from a zero mean Gaussian distribution with standard deviation $2.0$.

\textbf{Small $\boldsymbol{\alpha}_\star$ regime.} To examine whether the slower small $\boldsymbol{\alpha}_\star$ rates are caused by finite sample effects, we further extend the horizon to $t \in \{100, 200, 400, 800, 1600, 3200, 6400, 12800, 25600\}$. Small and large regimes in \Cref{tab:small_alpha_short_horizon}  correspond to $\boldsymbol{\alpha}_\star \sim \mathcal{U}[0.01,0.20]$
and $\boldsymbol{\alpha}_\star \sim \mathcal{U}[0.50,0.99]$, respectively.

\subsection{Real-World Experiments}
\label{sec:real_exp}
\textbf{EEG preprocessing.} 
We evaluate all methods on a multi-subject EEG dataset with \(36\) subjects and \(n=19\) channels. Each subject is treated as a multivariate time series. We split each subject trajectory into non-overlapping windows of length \(W=150\) with stride \(S=150\). For each window, the first \(70\%\) of samples are used for training and the remaining \(30\%\) for testing, giving \(105\) training samples and \(45\) test samples per window. The experiments are run on the raw EEG data without additional normalization.

\textbf{Hyperparameters.} 
We fix the search interval $[0.05,0.95]$. \emph{FO-GS} uses \(50\) equally spaced grid points for each row-wise search. \emph{FO-BS} uses binary search with tolerance \(10^{-2}\) and a truncated lifted state representation with memory length \(p=40\). \emph{FO-WT} uses Haar wavelets with linear detrending and estimates the fractional order from a weighted log-variance regression over wavelet levels \(2,3,4\).

\textbf{Evaluation metric.} We evaluate methods by one-step prediction. We report training and test NMSE, where
\(
\operatorname{NMSE}
=
\frac{\sum_t\|\hat{x}_t-x_t\|_2^2}
{\sum_t\|x_t\|_2^2}.
\)

\section{Proof of \texorpdfstring{\Cref{thm:vector_alpha_error_bound}}{vaeb}}
\label[appendix]{sec:proof_thm_1}

\subsection{Controlling In-Sample Error via Martingale Offset Complexity}
\label[appendix]{sec:offset_row_search}

We now adapt the offset martingale complexity argument~\citep{pmlr-v178-ziemann22a} to the row-wise grid-search
estimator used in \Cref{alg:vector_fractional_id}. The main difference from the idealized
continuous empirical risk minimizer (ERM) is that the true parameter $\alpha_{i,\star}$ need not belong
to the finite search grid. Consequently, an additional discretization term
appears in the basic inequality~\citep{pmlr-v178-ziemann22a}.

For each coordinate $i\in[n]$, define the row-wise empirical loss
\begin{equation*}
    L^{(i)}(\alpha_i,a_i)
    :=
    \sum_{s=0}^{t-1}
    \left|
        \Delta^{\alpha_i} x_{s+1}^{(i)}
        -
        a_i x_s
    \right|^2,
\end{equation*}
where $a_i\in\mathbb R^{1\times n}$ denotes the $i$-th row of $A$.
For any candidate $\alpha_i$, let
\begin{equation*}
    \hat a_i(\alpha_i)
    \in
    \arg\min_{a_i\in\mathbb R^{1\times n}}
    L^{(i)}(\alpha_i,a_i).
\end{equation*}
The row-search estimator is 
\begin{equation*}
    \hat\alpha_i
    \in
    \arg\min_{\alpha_i\in\mathcal A_{\epsilon,i}}
    L^{(i)}
    \bigl(
        \alpha_i,
        \hat a_i(\alpha_i)
    \bigr),
    \qquad
    \hat a_i
    :=
    \hat a_i(\hat\alpha_i).
\end{equation*}
Let $\alpha_i^\circ$ denote the grid point closest to the true parameter:
\begin{equation}
    \alpha_i^\circ
    \in
    \arg\min_{\alpha_i\in\mathcal A_{\epsilon,i}}
    |\alpha_i-\alpha_{i,\star}|.
\end{equation}
In particular, if the grid resolution is $\epsilon_i$, then
\begin{equation*}
    |\alpha_i^\circ-\alpha_{i,\star}|
    \leq
    \epsilon_i.
\end{equation*}
Define
\begin{equation*}
    b_s^{(i)}(\alpha_i)
    :=
    \left(
        \Delta^{\alpha_i}
        -
        \Delta^{\alpha_{i,\star}}
    \right)
    x_{s+1}^{(i)}.
\end{equation*}
Since the true dynamics satisfy
\begin{equation}
    \Delta^{\alpha_{i,\star}}x_{s+1}^{(i)}
    =
    a_{i,\star} x_s+\eta_s^{(i)},
\end{equation}
we have
\begin{equation}
    \Delta^{\alpha_i}x_{s+1}^{(i)}
    -
    a_i x_s
    =
    \eta_s^{(i)}
    +
    b_s^{(i)}(\alpha_i)
    -
    (a_i-a_{i,\star})x_s.
    \label{eq:row_residual_decomposition}
\end{equation}

For convenience, define
\begin{equation*}
    r_s^{(i)}(\alpha_i,\Delta a_i)
    :=
    b_s^{(i)}(\alpha_i)
    -
    \Delta a_i x_s,
    \qquad
    \Delta a_i:=a_i-a_{i,\star},
\end{equation*}
and
\begin{equation*}
    \hat r_s^{(i)}
    :=
    r_s^{(i)}
    \left(
        \hat\alpha_i,
        \hat a_i-a_{i,\star}
    \right).
\end{equation*}

\begin{lemma}[Row-wise basic inequality]
\label{lem:row_basic_inequality}
For every $i\in[n]$,
\begin{align}
    \sum_{s=0}^{t-1}
    \left|
        \hat r_s^{(i)}
    \right|^2
    \leq
    &
    4
    \sum_{s=0}^{t-1}
    \left\langle
        -\eta_s^{(i)},
        \hat r_s^{(i)}
    \right\rangle
    -
    \sum_{s=0}^{t-1}
    \left|
        \hat r_s^{(i)}
    \right|^2
    \nonumber\\
    &+
    4
    \sum_{s=0}^{t-1}
    \left\langle
        \eta_s^{(i)},
        b_s^{(i)}(\alpha_i^\circ)
    \right\rangle
    +
    2
    \sum_{s=0}^{t-1}
    \left|
        b_s^{(i)}(\alpha_i^\circ)
    \right|^2.
    \label{eq:row_basic_offset}
\end{align}
\end{lemma}

\begin{proof}
By optimality of the row-wise grid-search estimator,
\begin{align*}
    L^{(i)}
    \left(
        \hat\alpha_i,
        \hat a_i
    \right)
    &\leq
    L^{(i)}
    \left(
        \alpha_i^\circ,
        \hat a_i(\alpha_i^\circ)
    \right)
    \leq
    L^{(i)}
    \left(
        \alpha_i^\circ,
        a_{i,\star}
    \right).
\end{align*}
Using \eqref{eq:row_residual_decomposition}, this gives
\begin{equation*}
    \sum_{s=0}^{t-1}
    \left|
        \eta_s^{(i)}
        +
        \hat r_s^{(i)}
    \right|^2
    \leq
    \sum_{s=0}^{t-1}
    \left|
        \eta_s^{(i)}
        +
        b_s^{(i)}(\alpha_i^\circ)
    \right|^2.
\end{equation*}
Expanding both sides and cancelling
$\sum_{s=0}^{t-1}|\eta_s^{(i)}|^2$ yields
\begin{align}
    \sum_{s=0}^{t-1}
    \left|
        \hat r_s^{(i)}
    \right|^2
    \leq
    &
    -2
    \sum_{s=0}^{t-1}
    \left\langle
        \eta_s^{(i)},
        \hat r_s^{(i)}
    \right\rangle
    \nonumber\\
    &+
    2
    \sum_{s=0}^{t-1}
    \left\langle
        \eta_s^{(i)},
        b_s^{(i)}(\alpha_i^\circ)
    \right\rangle
    +
    \sum_{s=0}^{t-1}
    \left|
        b_s^{(i)}(\alpha_i^\circ)
    \right|^2.
    \label{eq:row_basic_expand}
\end{align}
Multiplying \eqref{eq:row_basic_expand} by two and subtracting
$\sum_s|\hat r_s^{(i)}|^2$ from both sides gives
\eqref{eq:row_basic_offset}.
\end{proof}
The first line of \eqref{eq:row_basic_offset} can now be controlled by an
offset martingale complexity argument. In particular, since
$(\hat\alpha_i,\hat a_i-a_{i,\star})$ is an admissible choice,
\begin{align*}
    \sum_{s=0}^{t-1}
    |\hat r_s^{(i)}|^2
    \leq
    &
    \max_{\alpha_i\in\mathcal A_{\epsilon,i}}
    \sup_{\Delta a_i\in\mathbb R^{1\times n}}
    \Bigg\{
        4
        \sum_{s=0}^{t-1}
        \left\langle
            -\eta_s^{(i)},
            b_s^{(i)}(\alpha_i)
            -
            \Delta a_i x_s
        \right\rangle
        \nonumber\\
        &\hspace{40mm}
        -
        \sum_{s=0}^{t-1}
        \left|
            b_s^{(i)}(\alpha_i)
            -
            \Delta a_i x_s
        \right|^2
    \Bigg\}
    +
    \Gamma_{t,i}^{\mathrm{grid}},
\end{align*}
where
\begin{equation*}
    \Gamma_{t,i}^{\mathrm{grid}}
    :=
    4
    \sum_{s=0}^{t-1}
    \left\langle
        \eta_s^{(i)},
        b_s^{(i)}(\alpha_i^\circ)
    \right\rangle
    +
    2
    \sum_{s=0}^{t-1}
    \left|
        b_s^{(i)}(\alpha_i^\circ)
    \right|^2
\end{equation*}
is the additional error induced by the finite grid.

We next explicitly perform the maximization over $\Delta a_i$.
Recall from \eqref{def:Delta_X} that
\[
X_t=[x_0,\ldots,x_{t-1}],
\qquad
X_tX_t^\top=\sum_{s=0}^{t-1}x_sx_s^\top.
\]
For a fixed $\alpha_i$, let
\begin{equation*}
    y_s^{(i)}(\alpha_i)
    :=
    b_s^{(i)}(\alpha_i)
    +
    2\eta_s^{(i)}.
\end{equation*}
Then
\begin{align*}
    &
    4
    \sum_{s=0}^{t-1}
    \left\langle
        -\eta_s^{(i)},
        b_s^{(i)}(\alpha_i)
        -
        \Delta a_i x_s
    \right\rangle
    -
    \sum_{s=0}^{t-1}
    \left|
        b_s^{(i)}(\alpha_i)
        -
        \Delta a_i x_s
    \right|^2
    \nonumber\\
    &=
    -4
    \sum_{s=0}^{t-1}
    \left\langle
        \eta_s^{(i)},
        b_s^{(i)}(\alpha_i)
    \right\rangle
    -
    \sum_{s=0}^{t-1}
    \left|
        b_s^{(i)}(\alpha_i)
    \right|^2
    +
    2
    \left\langle
        \sum_{s=0}^{t-1}
        y_s^{(i)}(\alpha_i)x_s^\top,
        \Delta a_i
    \right\rangle
    -
    \Delta a_i
    X_tX_t^\top
    \Delta a_i^\top.
\end{align*}
Assuming $X_tX_t^\top\succ0$, the maximizer is
\begin{equation*}
    \Delta a_i^{*}(\alpha_i)
    =
    \left(
        \sum_{s=0}^{t-1}
        y_s^{(i)}(\alpha_i)x_s^\top
    \right)
    (X_tX_t^\top)^{-1}.
\end{equation*}
Therefore,
\begin{align*}
    &
    \sup_{\Delta a_i}
    \Bigg\{
        4
        \sum_{s=0}^{t-1}
        \left\langle
            -\eta_s^{(i)},
            b_s^{(i)}(\alpha_i)
            -
            \Delta a_i x_s
        \right\rangle
        -
        \sum_{s=0}^{t-1}
        \left|
            b_s^{(i)}(\alpha_i)
            -
            \Delta a_i x_s
        \right|^2
    \Bigg\}
    \nonumber\\
    &=
    U_{t,i}(\alpha_i)
    +
    V_{t,i}(\alpha_i),
\end{align*}
where
\begin{equation*}
    U_{t,i}(\alpha_i)
    :=
    4
    \sum_{s=0}^{t-1}
    \left\langle
        -\eta_s^{(i)},
        b_s^{(i)}(\alpha_i)
    \right\rangle
    -
    \sum_{s=0}^{t-1}
    \left|
        b_s^{(i)}(\alpha_i)
    \right|^2,
\end{equation*}
and
\begin{equation*}
    V_{t,i}(\alpha_i)
    :=
    \left\|
        \left(
            \sum_{s=0}^{t-1}
            y_s^{(i)}(\alpha_i)x_s^\top
        \right)
        (X_tX_t^\top)^{-1/2}
    \right\|_2^2.
\end{equation*}
Combining everything together gives the following row-wise
in-sample error bound:
\begin{equation*}
    \begin{aligned}
    \sum_{s=0}^{t-1}
    \left|
        b_s^{(i)}(\hat\alpha_i)
        -
        (\hat a_i-a_{i,\star})x_s
    \right|^2
    \leq
    \max_{\alpha_i\in\mathcal A_{\epsilon,i}}
    \left\{
        U_{t,i}(\alpha_i)
        +
        V_{t,i}(\alpha_i)
    \right\}
    +
    \Gamma_{t,i}^{\mathrm{grid}}.
    \end{aligned}
\end{equation*}

Finally, summing over all coordinates $i\in[n]$ gives
\begin{equation*}
    \begin{aligned}
    \sum_{s=0}^{t-1}
    \left\|
        \left(
            \Delta^{\hat{\boldsymbol\alpha}}
            -
            \Delta^{\boldsymbol\alpha^\star}
        \right)x_{s+1}
        -
        (\hat A-A_\star)x_s
    \right\|_2^2
    \leq
    \sum_{i=1}^n
    \max_{\alpha_i\in\mathcal A_{\epsilon,i}}
    \left\{
        U_{t,i}(\alpha_i)
        +
        V_{t,i}(\alpha_i)
    \right\}
    +
    \Gamma_t^{\mathrm{grid}},
    \end{aligned}
\end{equation*}
where
\begin{equation*}
    \Gamma_t^{\mathrm{grid}}
    :=
    \sum_{i=1}^n
    \Gamma_{t,i}^{\mathrm{grid}}.
\end{equation*}
Equivalently, letting
\begin{equation*}
    \boldsymbol\alpha^\circ
    :=
    (\alpha_1^\circ,\ldots,\alpha_n^\circ),
\end{equation*}
and defining the matrices
\begin{equation*}
    B_t(\boldsymbol\alpha^\circ)
    :=
    \begin{bmatrix}
        b_0(\boldsymbol\alpha^\circ),
        &
        \cdots,
        &
        b_{t-1}(\boldsymbol\alpha^\circ)
    \end{bmatrix},
    \qquad
    W_t
    :=
    \begin{bmatrix}
        \eta_0,
        &
        \cdots,
        &
        \eta_{t-1}
    \end{bmatrix},
\end{equation*}
the grid-discretization contribution can be written compactly as
\begin{equation*}
    \Gamma_t^{\mathrm{grid}}
    =
    4
    \left\langle
        W_t,
        B_t(\boldsymbol\alpha^\circ)
    \right\rangle_F
    +
    2
    \left\|
        B_t(\boldsymbol\alpha^\circ)
    \right\|_F^2.
\end{equation*}

\begin{remark}
If the true parameter lies exactly on the search grid, i.e.,
$\alpha_{i,\star}\in\mathcal A_{\epsilon,i}$ for every $i$, then we may
take $\alpha_i^\circ=\alpha_{i,\star}$. In this case
\begin{equation*}
    b_s^{(i)}(\alpha_i^\circ)=0
\end{equation*}
for every $i$ and $s$, and hence
\begin{equation*}
    \Gamma_t^{\mathrm{grid}}=0.
\end{equation*}
Thus, the continuous ERM offset inequality is recovered as a special
case. For a finite grid, however, the additional term
$\Gamma_t^{\mathrm{grid}}$ must be retained.
\end{remark}
\begin{remark}
Since the fractional-order operator is diagonal across state
coordinates and the estimator searches for each $\alpha_i$ separately,
the relevant offset complexity decomposes as
\begin{equation*}
    \sum_{i=1}^n
    \max_{\alpha_i\in\mathcal A_{\epsilon,i}}
    \left\{
        U_{t,i}(\alpha_i)
        +
        V_{t,i}(\alpha_i)
    \right\},
\end{equation*}
rather than requiring a supremum over the full Cartesian grid
$\mathcal A_{\epsilon,1}\times\cdots\times
\mathcal A_{\epsilon,n}$.
This row-wise decomposition is important for obtaining the sharp
complexity dependence of the grid-search estimator.
\end{remark}

\begin{lemma}[Parameterized bound for the offset term]
\label{lem:offset_fast_rate}
Fix $\boldsymbol{\alpha}$ and assume $X_tX_t^\top\succ0$. For an arbitrary auxiliary matrix
perturbation $\Delta A\in\mathbb R^{n\times n}$, where
$\Delta A:=A-A_\star$, consider
\begin{equation}
\label{eq:Phi_alpha_def}
\Phi_{\boldsymbol{\alpha}}(\Delta A)
=
4\sum_{s=0}^{t-1}
\left\langle
    -\eta_s,
    b_s(\boldsymbol{\alpha})-\Delta A x_s
\right\rangle
-
\sum_{s=0}^{t-1}
\left\|
    b_s(\boldsymbol{\alpha})-\Delta A x_s
\right\|_2^2.
\end{equation}

For any $\tau\in(0,1]$, define
\begin{equation}
\label{eq:Ut_tau_def}
U_{t,\tau}(\boldsymbol{\alpha})
:=-4\langle W_t,B_t(\boldsymbol{\alpha})\rangle_F
-
\tau\|B_t(\boldsymbol{\alpha})\|_F^2.
\end{equation}
Then
\begin{equation}
\label{eq:Ut_tau_relation}
U_t(\boldsymbol{\alpha})
=
U_{t,\tau}(\boldsymbol{\alpha})
-(1-\tau)\|B_t(\boldsymbol{\alpha})\|_F^2.
\end{equation}
Moreover, for every fixed deterministic $\boldsymbol{\alpha}$ and
$\delta\in(0,1)$,
\begin{equation}
\label{eq:fixed_alpha_Ut_tau_bound}
\Pr\left(
    U_{t,\tau}(\boldsymbol{\alpha})
    \le
    \frac{8\sigma^2}{\tau}\log\frac1\delta
\right)
\ge1-\delta.
\end{equation}
For the row-wise grid search, define
\begin{equation}
\label{eq:row_Ut_tau_def}
U_{t,\tau,i}(\alpha_i)
:=-4\sum_{s=0}^{t-1}
\eta_s^{(i)}b_s^{(i)}(\alpha_i)
-
\tau\sum_{s=0}^{t-1}|b_s^{(i)}(\alpha_i)|^2.
\end{equation}
Then, with probability at least $1-\delta$, simultaneously for all
$i\in[n]$,
\begin{equation}
\label{eq:grid_Ut_bound}
\max_{\alpha_i\in\mathcal A_{\epsilon,i}}
U_{t,\tau,i}(\alpha_i)
\le
\frac{8\sigma^2}{\tau}
\log\left(\frac{M_i n}{\delta}\right).
\end{equation}
\end{lemma}

\begin{proof}
Expanding \eqref{eq:Phi_alpha_def} gives
\begin{align}
\Phi_{\boldsymbol{\alpha}}(\Delta A)
={}&
-4\langle W_t,B_t(\boldsymbol{\alpha})\rangle_F
-\|B_t(\boldsymbol{\alpha})\|_F^2
\nonumber\\
&+
2\left\langle
    \bigl(B_t(\boldsymbol{\alpha})+2W_t\bigr)X_t^\top,
    \Delta A
\right\rangle_F
-
\operatorname{tr}\!\left(
    \Delta A X_tX_t^\top\Delta A^\top
\right).
\label{eq:Phi_expand}
\end{align}
The maximizer of this auxiliary offset objective is
\[
    \Delta A_{\mathrm{opt}}
    =
    \bigl(B_t(\boldsymbol{\alpha})+2W_t\bigr)
    X_t^\top(X_tX_t^\top)^{-1}.
\]
Substitution gives
\begin{equation*}
\sup_{\Delta A}
\Phi_{\boldsymbol{\alpha}}(\Delta A)
=
U_t(\boldsymbol{\alpha})
+
V_t(\boldsymbol{\alpha}),
\end{equation*}
where
\begin{equation*}
U_t(\boldsymbol{\alpha})
=
-4\langle W_t,B_t(\boldsymbol{\alpha})\rangle_F
-
\|B_t(\boldsymbol{\alpha})\|_F^2,
\end{equation*}
and
\begin{equation*}
V_t(\boldsymbol{\alpha})
=
\left\|
    \bigl(B_t(\boldsymbol{\alpha})+2W_t\bigr)
    X_t^\top
    (X_tX_t^\top)^{-1/2}
\right\|_F^2.
\end{equation*}
Then \eqref{eq:Ut_tau_relation} follows directly from the definitions.

For fixed deterministic $\boldsymbol{\alpha}$,
$b_s(\boldsymbol{\alpha})$ is $\mathcal F_s$-measurable. Hence, for
$\lambda>0$,
\begin{align*}
\mathbb E\left[
\exp\left(
    \lambda\left[
        -4\langle\eta_s,b_s(\boldsymbol{\alpha})\rangle
        -\tau\|b_s(\boldsymbol{\alpha})\|_2^2
    \right]
\right)
\middle|\mathcal F_s
\right]
=
\exp\left(
    (8\sigma^2\lambda^2-\tau\lambda)
    \|b_s(\boldsymbol{\alpha})\|_2^2
\right).
\end{align*}
The right-hand side is at most one whenever
$0<\lambda\le\tau/(8\sigma^2)$. Iterating conditional expectations,
taking $\lambda=\tau/(8\sigma^2)$, and applying Markov's inequality
proves \eqref{eq:fixed_alpha_Ut_tau_bound}. The row-wise statement
follows by taking failure probability $\delta/(nM_i)$ for each fixed
grid point and applying a union bound over all rows and grid points.
\end{proof}
\begin{lemma}[High-probability bound for $V_t(\boldsymbol{\alpha})$]
\label{lem:VT_bound}
Define
\[
\mathfrak r_t(\delta,k)
:=
\frac{C}{\sqrt{t\lambda_{\min}(\Gamma_k)}}
\left(
    n\log\frac{n}{\delta}
    +
    \log\det(\Gamma_t\Gamma_k^{-1})
\right)^{1/2},
\]
where $C>0$ is the universal constant in \Cref{lem:bound_for_noise_term}. Suppose that
\[
\frac{t}{k}
\ge
c\left(
    n\log\frac{n}{\delta}
    +
    \log\det(\Gamma_t\Gamma_k^{-1})
\right).
\]
For the row-wise search intervals
$\mathcal A_i=[\underline\alpha_i,\overline\alpha_i]$, define
\begin{equation}
S_{1,i}
:=
\sup_{u\in\mathcal A_i}
\sum_{j\ge1}|\partial_u\psi(u,j)|,
\qquad
S_1:=\max_{1\le i\le n}S_{1,i}.
\label{eq:S1_original_notation}
\end{equation}
Then, with probability at least $1-\delta$, simultaneously for all
$\boldsymbol{\alpha}$ in the search set,
\begin{equation}
\label{eq:VT_bound}
V_t(\boldsymbol{\alpha})
\le
\operatorname{tr}(X_tX_t^\top)
\left(
    S_1\|\boldsymbol{\alpha}-\boldsymbol{\alpha}_\star\|_\infty
    +2\mathfrak r_t(\delta,k)
\right)^2.
\end{equation}
In particular,
\begin{align*}
V_t(\boldsymbol{\alpha})
\le{}&
2S_1^2
\|\boldsymbol{\alpha}-\boldsymbol{\alpha}_\star\|_\infty^2
\operatorname{tr}(X_tX_t^\top)
\nonumber\\
&+
\frac{8C^2\operatorname{tr}(X_tX_t^\top)}
{t\lambda_{\min}(\Gamma_k)}
\left(
    n\log\frac{n}{\delta}
    +
    \log\det(\Gamma_t\Gamma_k^{-1})
\right).
\end{align*}
\end{lemma}

\begin{proof}
Since
\[
\sum_{s=0}^{t-1}
\bigl(b_s(\boldsymbol{\alpha})+2\eta_s\bigr)x_s^\top
=
\bigl(B_t(\boldsymbol{\alpha})+2W_t\bigr)X_t^\top,
\]
we have
\begin{align}
V_t(\boldsymbol{\alpha})^{1/2}
\le
\left\|
B_t(\boldsymbol{\alpha})X_t^\top(X_tX_t^\top)^{-1/2}
\right\|_F
+
2\left\|
W_tX_t^\top(X_tX_t^\top)^{-1/2}
\right\|_F.
\label{eq:VT_decomp}
\end{align}
Let
\[
P_X:=X_t^\top(X_tX_t^\top)^{-1}X_t.
\]
Then $P_X$ is an orthogonal projection, so
\begin{align*}
\left\|
B_t(\boldsymbol{\alpha})X_t^\top(X_tX_t^\top)^{-1/2}
\right\|_F^2
=
\operatorname{tr}\!\left(
B_t(\boldsymbol{\alpha})P_XB_t(\boldsymbol{\alpha})^\top
\right)
\le
\|B_t(\boldsymbol{\alpha})\|_F^2.
\end{align*}
The convolution bound established in the proof of \Cref{lem:bound_for_bias_term} gives
\[
\|B_t(\boldsymbol{\alpha})\|_F^2
\le
S_1^2
\|\boldsymbol{\alpha}-\boldsymbol{\alpha}_\star\|_\infty^2
\operatorname{tr}(X_tX_t^\top).
\]
Hence
\begin{equation}
\label{eq:VT_bias_bound}
\left\|
B_t(\boldsymbol{\alpha})X_t^\top(X_tX_t^\top)^{-1/2}
\right\|_F
\le
S_1\|\boldsymbol{\alpha}-\boldsymbol{\alpha}_\star\|_\infty
\sqrt{\operatorname{tr}(X_tX_t^\top)}.
\end{equation}
Moreover,
\[
W_tX_t^\top(X_tX_t^\top)^{-1/2}
=
\left(W_tX_t^\top(X_tX_t^\top)^{-1}\right)
(X_tX_t^\top)^{1/2},
\]
and therefore \Cref{lem:bound_for_noise_term} implies
\begin{equation}
\label{eq:VT_noise_bound}
\left\|
W_tX_t^\top(X_tX_t^\top)^{-1/2}
\right\|_F
\le
\mathfrak r_t(\delta,k)
\sqrt{\operatorname{tr}(X_tX_t^\top)}
\end{equation}
with probability at least $1-\delta$. Substituting
\eqref{eq:VT_bias_bound} and \eqref{eq:VT_noise_bound} into
\eqref{eq:VT_decomp} proves \eqref{eq:VT_bound}; the expanded form
follows from $(a+b)^2\le2a^2+2b^2$.
\end{proof}

\begin{lemma}[Absorbable high-probability in-sample error bound]
\label{lem:insample_fast_rate}
For $\tau\in(0,1]$ and $\rho>0$,
suppose that
\begin{equation*}
\frac{t}{k}
\ge
c\left(
    n\log\frac{3n}{\delta}
    +
    \log\det(\Gamma_t\Gamma_k^{-1})
\right).
\end{equation*}
Then, with probability at least $1-\delta$,
\begin{align}
&\sum_{s=0}^{t-1}
\left\|
    \left(
        \Delta^{\hat{\boldsymbol{\alpha}}}
        -\Delta^{\boldsymbol{\alpha}_\star}
    \right)x_{s+1}
    -(\hat A-A_\star)x_s
\right\|_2^2
\nonumber\\
&\quad\le
(\tau+\rho)S_1^2
\|\hat{\boldsymbol{\alpha}}-\boldsymbol{\alpha}_\star\|_\infty^2
\operatorname{tr}(X_tX_t^\top)
+
3S_1^2\epsilon_{\max}^2\operatorname{tr}(X_tX_t^\top)
\nonumber\\
&\qquad+
\frac{4(1+\rho^{-1})C^2\operatorname{tr}(X_tX_t^\top)}
{t\lambda_{\min}(\Gamma_k)}
\left(
    n\log\frac{3n}{\delta}
    +
    \log\det(\Gamma_t\Gamma_k^{-1})
\right)
\nonumber\\
&\qquad+
\frac{8\sigma^2}{\tau}
\sum_{i=1}^n\log\left(\frac{3M_i n}{\delta}\right)
+
8\sigma^2\log\left(\frac3\delta\right).
\label{eq:insample_fast_rate}
\end{align}
\end{lemma}

\begin{proof}
For each row, recall
\[
\Delta a_i:=\hat a_i-a_{i,\star},
\qquad
\hat r_s^{(i)}
:=b_s^{(i)}(\hat\alpha_i)-\Delta a_i x_s.
\]
The row-wise basic inequality in \Cref{lem:row_basic_inequality} and the quadratic maximization in
Lemma~\ref{lem:offset_fast_rate} give
\[
\sum_{s=0}^{t-1}|\hat r_s^{(i)}|^2
\le
U_{t,i}(\hat\alpha_i)
+V_{t,i}(\hat\alpha_i)
+\Gamma_{t,i}^{\mathrm{grid}}.
\]
Summing over rows,
\begin{align}
\mathcal E_t
:=
\sum_{s=0}^{t-1}
\left\|
    \left(
        \Delta^{\hat{\boldsymbol{\alpha}}}
        -\Delta^{\boldsymbol{\alpha}_\star}
    \right)x_{s+1}
    -(\hat A-A_\star)x_s
\right\|_2^2
\le
U_t(\hat{\boldsymbol{\alpha}})
+V_t(\hat{\boldsymbol{\alpha}})
+\Gamma_t^{\mathrm{grid}}.
\label{eq:actual_alpha_decomposition}
\end{align}
By \eqref{eq:Ut_tau_relation},
\[
U_t(\hat{\boldsymbol{\alpha}})
=
U_{t,\tau}(\hat{\boldsymbol{\alpha}})
-(1-\tau)\|B_t(\hat{\boldsymbol{\alpha}})\|_F^2.
\]
Young's inequality and the projection inequality give
\begin{align*}
U_t(\hat{\boldsymbol{\alpha}})
+V_t(\hat{\boldsymbol{\alpha}})
\le
U_{t,\tau}(\hat{\boldsymbol{\alpha}})
+(\tau+\rho)\|B_t(\hat{\boldsymbol{\alpha}})\|_F^2
+4(1+\rho^{-1})\|W_tX_t^\top(X_tX_t^\top)^{-1/2}\|_F^2.
\end{align*}
Lemma~\ref{lem:offset_fast_rate}, with failure probability $\delta/3$,
gives
\[
U_{t,\tau}(\hat{\boldsymbol{\alpha}})
\le
\frac{8\sigma^2}{\tau}
\sum_{i=1}^n\log\left(\frac{3M_i n}{\delta}\right).
\]
Also, the convolution bound established in the proof of \Cref{lem:bound_for_bias_term} gives
\begin{equation*}
\|B_t(\hat{\boldsymbol{\alpha}})\|_F^2
\le
S_1^2
\|\hat{\boldsymbol{\alpha}}-\boldsymbol{\alpha}_\star\|_\infty^2
\operatorname{tr}(X_tX_t^\top).
\end{equation*}

Next, retain the full good event used in the proof of \Cref{lem:bound_for_noise_term}. With
failure probability $\delta/3$, this event simultaneously gives
\[
\left\|
W_tX_t^\top(X_tX_t^\top)^{-1}
\right\|_{\mathrm{op}}
\le
\mathfrak r_t\!\left(\frac{\delta}{3},k\right)
\]
and the complement of the event $\mathcal E_3$ from that proof,
namely 
\begin{equation}
\label{eq:E3_joint_upper_gram}
X_tX_t^\top
\preceq
t\,\overline\Gamma_t\!\left(\frac{\delta}{3}\right),
\qquad
\overline\Gamma_t(\delta_0)
:=\frac{\sigma^2n}{\delta_0}\Gamma_t.
\end{equation}
Consequently,
\begin{align*}
\|W_tX_t^\top(X_tX_t^\top)^{-1/2}\|_F^2
&\le
\mathfrak r_t^2\!\left(\frac{\delta}{3},k\right)
\operatorname{tr}(X_tX_t^\top).
\end{align*}
Thus the upper empirical state energy control in
\eqref{eq:E3_joint_upper_gram} does not require a fourth event or an
additional allocation of the failure probability.

Finally, let $B_t^\circ:=B_t(\boldsymbol{\alpha}^\circ)$. Then
\[
\Gamma_t^{\mathrm{grid}}
=
\left(
4\langle W_t,B_t^\circ\rangle_F-\|B_t^\circ\|_F^2
\right)
+3\|B_t^\circ\|_F^2.
\]
The same conditional moment generating function argument in \Cref{lem:offset_fast_rate}, with failure probability
$\delta/3$, yields
\[
4\langle W_t,B_t^\circ\rangle_F-\|B_t^\circ\|_F^2
\le8\sigma^2\log\left(\frac3\delta\right),
\]
while the convolution bound gives
\[
\|B_t^\circ\|_F^2
\le
S_1^2\epsilon_{\max}^2\operatorname{tr}(X_tX_t^\top).
\]
Substituting these three bounds into
\eqref{eq:actual_alpha_decomposition} and applying a union bound proves
\eqref{eq:insample_fast_rate} and \eqref{eq:E3_joint_upper_gram}.
\end{proof}

\subsection{Lower Isometry for the Row-Wise Grid-Search Estimator}
\label[appendix]{sec:Lower Isometry for the Row-Wise Grid-Search Estimator}
The row-wise structure of \textit{FO-GS} allows the lower isometry analysis to be carried out coordinate-wise. For each row $i$, we profile out the corresponding row of $A$ and study the resulting noiseless prediction error as a function of $\alpha_i$. This quantity captures the curvature in the fractional-order parameter and will be used to relate the in-sample prediction error to the estimation error in $\alpha_i$.

For each $i\in[n]$, recall
\[
b_s^{(i)}(\alpha_i)
:=
\left(
    \Delta^{\alpha_i}
    -
    \Delta^{\alpha_{i,\star}}
\right)x_{s+1}^{(i)},
\qquad
\Delta a_i:=a_i-a_{i,\star},
\]
and define
\[
Q_{t,i}(\alpha_i,\Delta a_i)
:=
\sum_{s=0}^{t-1}
\left|
    b_s^{(i)}(\alpha_i)
    -
    \Delta a_i x_s
\right|^2.
\]
Since \textit{FO-GS} profiles out $a_i$ for every candidate $\alpha_i$,
it is natural to introduce the profiled noiseless error
\[
\underline Q_{t,i}(\alpha_i)
:=
\inf_{\Delta a_i\in\mathbb R^{1\times n}}
Q_{t,i}(\alpha_i,\Delta a_i).
\]
In particular,
\[
Q_{t,i}
\left(
    \hat\alpha_i,
    \hat a_i-a_{i,\star}
\right)
\ge
\underline Q_{t,i}(\hat\alpha_i).
\]
\begin{lemma}
\label{lem:taylor_remainder}
Suppose \Cref{ass:stable_A} holds. For each $i\in[n]$, define
\begin{align*}
\underline\alpha_{i,\mathrm{loc}}
&:=\frac{\alpha_{i,\star}}{2}, \notag\\
\mathfrak Z_i
&:=
\zeta\!\left(1+\frac{\alpha_{i,\star}}{2}\right)
-2\zeta'\!\left(1+\frac{\alpha_{i,\star}}{2}\right)
+\zeta''\!\left(1+\frac{\alpha_{i,\star}}{2}\right), \notag\\
K_i
&:=
\frac{9e}{4}\sigma\widetilde C_G\sqrt{\zeta(2)}\,\mathfrak Z_i.
\end{align*}
Then, for every
$\alpha_i\in[\underline\alpha_{i,\mathrm{loc}},1]$,
\begin{align*}
\frac1t\sum_{s=0}^{t-1}
\mathbb E\left|
 b_s^{(i)}(\alpha_i)
 -(\alpha_i-\alpha_{i,\star})
 \sum_{j\ge1}
 \partial_\alpha\psi(\alpha_{i,\star},j)
 x_{s+1-j}^{(i)}
\right|^2
\le
\frac{K_i^2}{4}
|\alpha_i-\alpha_{i,\star}|^4.
\end{align*}
\end{lemma}
\begin{proof}
For row $i$, let
\[S_{2}
:=
\sup_{u\in[a,1]}
\sum_{j=1}^{\infty}
|\partial_u^2\psi(u,j)|, \qquad a>0
\] and
\[
\rho_{s}^{(i)}(\alpha)
:=
b_{s}^{(i)}(\alpha)
-
(\alpha - \alpha_{i, \star})
\sum_{j=1}^{s+1}
\partial_\alpha \psi(\alpha_{i,\star},j)
x_{s+1-j}^{(i)}.
\]
Taylor's theorem in integral form gives
\[
\psi(\alpha,j)
-
\psi(\alpha_{i,\star},j)
-
(\alpha - \alpha_{i, \star})\,\partial_\alpha \psi(\alpha_{i,\star},j)
=
(\alpha - \alpha_{i, \star})^2
\int_0^1
(1-\tau)
\partial_\alpha^2
\psi(\alpha_{i,\star}+\tau(\alpha - \alpha_{i, \star}),j)
\,d\tau.
\]
Therefore,
\[
\rho_{s}^{(i)}(\alpha)
=
(\alpha - \alpha_{i, \star})^2
\int_0^1
(1-\tau)
\sum_{j=1}^{s+1}
\partial_\alpha^2
\psi(\alpha_{i,\star}+\tau(\alpha - \alpha_{i, \star}),j)
x_{s+1-j}^{(i)}
\,d\tau.
\]
For $j \geq 2$, let
$
\psi(u,j)
=
-\frac{u(1-u)}{j}Q_j(u),$ and
$
Q_j(u)
:=
\prod_{m=2}^{j-1}
\left(1-\frac{u}{m}\right).
$

Define
\[
H_{1,j}(u)
=
\sum_{m=2}^{j-1}\frac{1}{m-u},
\qquad
H_{2,j}(u)
=
\sum_{m=2}^{j-1}\frac{1}{(m-u)^2}.
\]

Then
\[
Q_j'(u)
=
-Q_j(u)H_{1,j}(u),
\]
and
\[
\left|Q_j''(u)\right|
=
Q_j(u)
\left(
H_{1,j}(u)^2-H_{2,j}(u)
\right)
\le Q_j(u) H_{1,j}(u)^2.
\]

For $u\in[a,1]$, where $a>0$,
$
Q_j(u)
\leq
e\,j^{-u}
\leq
e\,j^{-a},
$
and
\(
H_{1,j}(u)
\leq
1+\log j,
\)

Consequently,
\begin{align*}
\left|
\partial_u^2\psi(u,j)
\right|
&\leq 
\frac{Q_j(u)}{j} \Bigl[2 + 2H_{1,j} (u) + \frac{1}{4}H_{1,j}(u)^2  \Bigr] \notag \\
&\leq 
\frac{e}{j^{1+a}}
\left[
2
+
2(1+\log j)
+
\frac{1}{4}
(1+\log j)^2
\right].
\end{align*}
By simplification, we get
\[
\left|
\partial_u^2\psi(u,j)
\right|
\leq
\frac{9}{4}e(1+\log j)^2j^{-(1+a)}
\]
for every $u\in[a,1]$.
Hence
\begin{equation}
S_2(a)
\leq
\frac{9}{4}e
\left[
\zeta(1+a)
-
2\zeta'(1+a)
+
\zeta''(1+a)
\right].
\label{eq:S2_explicit_bound}
\end{equation}
For coordinate $i$, simply choose $\underline\alpha_{i,\mathrm{loc}} = \frac{\alpha_{i,\star}}{2}$. Then the result holds for all $\alpha \in [\underline\alpha_{i,\mathrm{loc}},1]$. Using $\mathbb{E}|x_s^{(i)}|^2 \le \sigma^2 \sum_{m=0}^{\infty}\|e_i^\top G_m\|_2^2 \le \sigma^2 \tilde{C}_G^2\zeta(2)$ and Minkowski's inequality
\begin{align*} 
\left\| \rho_{s}^{(i)}(\alpha) \right\|_{L^2} 
&\leq (\alpha - \alpha_{i, \star})^2 \int_0^1 (1-\tau) \sum_{j=1}^{s+1} \left| \partial_\alpha^2 \psi(\alpha_{i,*}+\tau(\alpha - \alpha_{i, \star}),j) \right| \left\| x_{s+1-j}^{(i)} \right\|_{L^2} \,d\tau \nonumber \\ 
&\leq \frac{(\alpha - \alpha_{i, \star})^2}{2} S_2(a)\sqrt{ \sigma^2 \tilde{C}_G^2\zeta(2)}. 
\end{align*} 
Squaring, we have
\begin{equation*}
\mathbb{E}| \rho_{s}^{(i)}(\alpha)|^2 \le \frac{(\alpha - \alpha_{i, \star})^4}{4}S_2(\frac{\alpha_{i,\star}}{2})^2 \sigma^2 \tilde{C}_G^2 \zeta(2)    
\end{equation*}
Hence, we have
\begin{equation*}
\frac1t\sum_{s=0}^{t-1}
\mathbb E\left|
b_{s}^{(i)}(\alpha)
-
(\alpha-\alpha_{i,\star})
\sum_{j\ge1}
\partial_\alpha\psi(\alpha_{i,\star},j)
x_{s+1-j}^{(i)}
\right|^2
\le
\frac{K_i^2}{4}
(\alpha-\alpha_{i,\star})^4
\end{equation*}
with \(K_i = \frac{9}{4}e \sigma \tilde{C}_G \sqrt{ \zeta(2)} 
\left[
\zeta(1+\frac{\alpha_{i,\star}}{2})
-
2\zeta'(1+\frac{\alpha_{i,\star}}{2})
+
\zeta''(1+\frac{\alpha_{i,\star}}{2}) \right]  \)
\end{proof}

Recall the row-wise population risk
\[
R_t^{(i)}(\alpha_i)
:=
\frac{1}{t}
\inf_{a_i}
\sum_{s=0}^{t-1}
\mathbb E
\left[
    \left(
        \Delta^{\alpha_i}x_{s+1}^{(i)}
        -
        a_i x_s
    \right)^2
\right].
\]
Using the true dynamics,
\[
\Delta^{\alpha_i}x_{s+1}^{(i)}-a_i x_s
=
\eta_s^{(i)}
+
b_s^{(i)}(\alpha_i)
-
\Delta a_i x_s.
\]
Since $b_s^{(i)}(\alpha_i)$ and $x_s$ are
$\mathcal F_s$-measurable and
$\mathbb E[\eta_s^{(i)}\mid\mathcal F_s]=0$,
\[
R_t^{(i)}(\alpha_i)-R_t^{(i)}(\alpha_{i,\star})
=
\frac{1}{t}
\inf_{\Delta a_i}
\sum_{s=0}^{t-1}
\mathbb E
\left[
    \left|
        b_s^{(i)}(\alpha_i)
        -
        \Delta a_i x_s
    \right|^2
\right].
\]

Then we use the lower isometry argument in two stages. First, a global
profiled lower isometry bound in \Cref{lem:global_profiled_lower_isometry} shows that, with high probability,
simultaneously over all rows and all grid points outside the
separation neighborhood of $\alpha_{i,\star}$,
\[
\underline{\mathcal Q}_{t,i}(\alpha_i)
\ge
\frac{t}{2}
\left(
R_t^{(i)}(\alpha_i)
-
R_t^{(i)}(\alpha_{i,\star})
\right).
\]
Since each $\mathcal A_{\epsilon,i}$ is finite, the uniform statement
is obtained by establishing the bound for a fixed $\alpha_i$ and
taking a union bound over the grid points and rows. Combined with the
in-sample upper bound and the population separation gap
$\gamma$, this global bound rules out grid points
outside the separation neighborhood and localizes the estimator to
the set
\begin{equation}
\mathcal G_i
:=
\left\{
\alpha_i\in
\mathcal A_{\epsilon,i}\cap
[\alpha_{i,\mathrm{loc}},1]:
|\alpha_i-\alpha_{i,\star}|
\le
\frac{\sqrt{\mu_{t,i}}}{2K_i}
\right\}.
\label{eq:def_mathG_I}
\end{equation}
Within $\mathcal G_i$, the population risk has a quadratic local
curvature. In particular, by Lemma~\ref{lem:taylor_remainder} and the
reverse triangle inequality,
\[
\sqrt{
R_t^{(i)}(\alpha_i)-R_t^{(i)}(\alpha_{i,\star})
}
\ge
|\alpha_i-\alpha_{i,\star}|\sqrt{\mu_{t,i}}
-
\frac{K_i}{2}
|\alpha_i-\alpha_{i,\star}|^2.
\]
Hence, whenever
$|\alpha_i-\alpha_{i,\star}|
\le
\sqrt{\mu_{t,i}}/(2K_i)$,
\[
R_t^{(i)}(\alpha_i)-R_t^{(i)}(\alpha_{i,\star})
\ge
\frac{\mu_{t,i}}{2}
|\alpha_i-\alpha_{i,\star}|^2.
\]
The corresponding local lower isometry argument in \Cref{lem:Lower_isometry} transfers this
curvature to the empirical profiled error and gives, uniformly over
$\alpha_i\in\mathcal G_i$,
\[
\underline{\mathcal Q}_{t,i}(\alpha_i)
\ge
\frac{t\mu_{t,i}}{8}
|\alpha_i-\alpha_{i,\star}|^2.
\]
Therefore, once $\hat\alpha_i\in\mathcal G_i$ for every $i$,
\[
\mathcal Q_{t,i}
\left(
\hat\alpha_i,
\hat a_i-a_{i,\star}
\right)
\ge
\underline{\mathcal Q}_{t,i}(\hat\alpha_i)
\ge
\frac{t\mu_{t,i}}{8}
|\hat\alpha_i-\alpha_{i,\star}|^2.
\]
Summing over the rows yields
\[
\mathcal E_t
\ge
\frac{t}{8}
\sum_{i=1}^n
\mu_{t,i}
|\hat\alpha_i-\alpha_{i,\star}|^2
\ge
\frac{t\mu_{\min}}{8}
\left\|
\hat{\boldsymbol\alpha}
-
\boldsymbol\alpha^\star
\right\|_\infty^2.
\]
Combining this lower bound with the in-sample upper bound and choosing
the tunable coefficient sufficiently small allows the quadratic
estimation-error term on the upper-bound side to be absorbed. This
gives
\[
\left\|
\hat{\boldsymbol\alpha}
-
\boldsymbol\alpha_\star
\right\|_\infty^2
\lesssim
\epsilon_{\max}^2
+
\mathcal O(t^{-1}).
\]

\textbf{Common notation for Lemmas~\ref{lem:row_lower_isometry}--\ref{lem:E2}.}
For each row $i\in[n]$ and candidate $\alpha_i$, let
\begin{align}
h_i(\alpha_i)&:=\alpha_i-\alpha_{i,\star},\\
b_s^{(i)}(\alpha_i)
&:=\bigl(\Delta^{\alpha_i}-\Delta^{\alpha_{i,\star}}\bigr)x_{s+1}^{(i)},\\
g_s^{(i)}
&:=\sum_{j\ge1}d_{i,j}x_{s+1-j}^{(i)},
\qquad
 d_{i,j}:=\partial_\alpha\psi(\alpha_{i,\star},j),\\
\rho_s^{(i)}(\alpha_i)
&:=b_s^{(i)}(\alpha_i)-h_i(\alpha_i)g_s^{(i)}.
\end{align}
Recall the unprofiled and profiled noiseless errors
\begin{align}
\mathcal Q_{t,i}(\alpha_i,v)
&:=\sum_{s=0}^{t-1}
\bigl|b_s^{(i)}(\alpha_i)-v x_s\bigr|^2,\\
\underline{\mathcal Q}_{t,i}(\alpha_i)
&:=\inf_{v\in\mathbb R^{1\times n}}
\mathcal Q_{t,i}(\alpha_i,v).
\end{align}
Define the empirical and population profiled derivative
curvatures using the same profiling operation:
\begin{align}
\hat\mu_{t,i}
&:=\inf_{v\in\mathbb R^{1\times n}}
\frac1t\sum_{s=0}^{t-1}
\bigl(g_s^{(i)}-v x_s\bigr)^2,\label{eq:muhat_common}\\
\mu_{t,i}
&:=\inf_{v\in\mathbb R^{1\times n}}
\frac1t\sum_{s=0}^{t-1}
\mathbb E\bigl(g_s^{(i)}-v x_s\bigr)^2.\label{eq:mu_common}
\end{align}
Define
\begin{equation}
r_{i,\mathrm{sep}}
:=
\frac{\sqrt{\mu_{i, \mathrm{lb}}}}
{2K_i}.
\label{eq:separation_radius}
\end{equation}
\begin{lemma}[Row-wise lower isometry]
\label{lem:row_lower_isometry}
Suppose that the following two events hold simultaneously:
\begin{align}
\hat\mu_{t,i}&\ge \frac12\mu_{t,i},
&&i\in[n],\tag{E1}\label{eq:E1_event_final}\\
\frac1t\sum_{s=0}^{t-1}
|\rho_s^{(i)}(\alpha_i)|^2
&\le \frac{K_i^2}{2}
|\alpha_i-\alpha_{i,\star}|^4,
&&i\in[n],\quad
\alpha_i\in\mathcal A_{\epsilon,i}
\cap[\underline\alpha_{i,\mathrm{loc}},1].
\tag{E2}\label{eq:E2_event_final}
\end{align}
Then, simultaneously for every $i\in[n]$ and every
$\alpha_i\in\mathcal G_i$,
\begin{equation}
\underline{\mathcal Q}_{t,i}(\alpha_i)
\ge
\frac{t\mu_{t,i}}{8}
|\alpha_i-\alpha_{i,\star}|^2.
\label{eq:row_lower_isometry_final}
\end{equation}
Consequently, on any event on which
$\hat\alpha_i\in\mathcal G_i$,
\begin{equation}
\mathcal Q_{t,i}
\bigl(\hat\alpha_i,\hat a_i-a_{i,\star}\bigr)
\ge
\frac{t\mu_{t,i}}{8}
|\hat\alpha_i-\alpha_{i,\star}|^2.
\label{eq:row_lower_isometry_estimator_final}
\end{equation}
\end{lemma}

\begin{proof}
Fix $i$ and $\alpha_i\in\mathcal G_i$, and write
$h_i:=\alpha_i-\alpha_{i,\star}$. Let
\[
X_t=[x_0,\ldots,x_{t-1}],
\qquad
P_X:=X_t^\top(X_tX_t^\top)^\dagger X_t,
\]
where $\dagger$ denotes the Moore--Penrose pseudoinverse. Define the
corresponding time-stacked vectors
\[
\begin{aligned}
b_i(\alpha_i)&:=
\bigl(b_0^{(i)}(\alpha_i),\ldots,b_{t-1}^{(i)}(\alpha_i)\bigr)^\top,\\
g_i&:=\bigl(g_0^{(i)},\ldots,g_{t-1}^{(i)}\bigr)^\top,\\
\rho_i(\alpha_i)&:=
\bigl(\rho_0^{(i)}(\alpha_i),\ldots,
\rho_{t-1}^{(i)}(\alpha_i)\bigr)^\top.
\end{aligned}
\]
Then $P_X$ is the orthogonal projection onto the row space of $X_t$,
and profiling gives
\[
\underline{\mathcal Q}_{t,i}(\alpha_i)
=
\|(I-P_X)b_i(\alpha_i)\|_2^2.
\]
Since $b_i(\alpha_i)=h_i g_i+\rho_i(\alpha_i)$, the reverse triangle
inequality and $\|I-P_X\|_{op}\le1$ imply
\begin{align*}
\sqrt{\frac{\underline{\mathcal Q}_{t,i}(\alpha_i)}{t}}
&\ge
|h_i|\sqrt{\hat\mu_{t,i}}
-
\left(
\frac1t\sum_{s=0}^{t-1}|\rho_s^{(i)}(\alpha_i)|^2
\right)^{1/2}.
\end{align*}
On \eqref{eq:E1_event_final}--\eqref{eq:E2_event_final},
\[
\sqrt{\frac{\underline{\mathcal Q}_{t,i}(\alpha_i)}{t}}
\ge
|h_i|\sqrt{\frac{\mu_{t,i}}2}
-
\frac{K_i}{\sqrt2}|h_i|^2.
\]
Because $\alpha_i\in\mathcal G_i$,
$K_i|h_i|\le\sqrt{\mu_{t,i}}/2$, and hence
\[
\sqrt{\frac{\underline{\mathcal Q}_{t,i}(\alpha_i)}{t}}
\ge
\frac{\sqrt{\mu_{t,i}}}{2\sqrt2}|h_i|.
\]
Squaring proves \eqref{eq:row_lower_isometry_final}. Finally,
\[
\mathcal Q_{t,i}
\bigl(\hat\alpha_i,\hat a_i-a_{i,\star}\bigr)
\ge
\underline{\mathcal Q}_{t,i}(\hat\alpha_i),
\]
which proves \eqref{eq:row_lower_isometry_estimator_final} whenever
$\hat\alpha_i\in\mathcal G_i$.
\end{proof}

\begin{lemma}[Relative concentration of the empirical derivative curvature]
\label{lem:E1}
Suppose \Cref{ass:stable_A} holds. Let
\[
D_i:=\sum_{j\ge1}|d_{i,j}|,
\qquad
D_{2,i}:=
\left(\sum_{j\ge1}d_{i,j}^2\right)^{1/2},
\qquad
\mathcal L_x
:=
\widetilde C_x
\left(1+\frac{2}{\alpha_{\min}}\right).
\]
Define
\[
z_s^{(i)}
:=
\begin{bmatrix}
x_s\\
g_s^{(i)}
\end{bmatrix},
\qquad
\hat\Sigma_{z,i}
:=
\frac1t\sum_{s=0}^{t-1}
z_s^{(i)}z_s^{(i)\top},
\qquad
\Sigma_{z,i}
:=
\mathbb E\hat\Sigma_{z,i}.
\]
Set
\[
\kappa_{i,\star}:=d_{i,2}^2+d_{i,3}^2,
\]
\[
C_{\mathrm{der}}
:=
\max\left\{
2\|G_1\|_{\mathrm{op}}^2
+3\bigl(\|G_2\|_{\mathrm{op}}
+\|G_1\|_{\mathrm{op}}^2\bigr)^2,
\;
2+3\|G_1\|_{\mathrm{op}}^2,
\;
3
\right\},
\]
and
\[
\mu_{i,\mathrm{lb}}
:=
\frac{\sigma^2\kappa_{i,\star}}
{2C_{\mathrm{der}}}.
\]
For each \(i\), define
\[
\overline{\mathfrak R}_i
:=
\mathcal L_x
\left[
\frac{\sqrt2}{\sigma}
+
\frac{
D_i+\frac{\sqrt{2\mathcal L_x}}{\sigma}D_{2,i}
}{
\sqrt{\mu_{i,\mathrm{lb}}}
}
\right]^2,
\]
and let
\[
q_1(\delta)
:=
(n+1)\log 9+\log\frac{4n}{\delta}.
\]

Then, for \(t\ge6\), \(\Sigma_{z,i}\succ0\) and
\(\mu_{t,i}\ge\mu_{i,\mathrm{lb}}>0\). Moreover, with probability at
least \(1-\delta/2\), simultaneously for all \(i\in[n]\),
\[
\left\|
\Sigma_{z,i}^{-1/2}
\bigl(\hat\Sigma_{z,i}-\Sigma_{z,i}\bigr)
\Sigma_{z,i}^{-1/2}
\right\|_{\mathrm{op}}
\le
4\left(
\sqrt{
\frac{
\overline{\mathfrak R}_i q_1(\delta)
}{t}
}
+
\frac{
\overline{\mathfrak R}_i q_1(\delta)
}{t}
\right).
\]

Consequently, if
\[
t\ge
t_{\mathrm{E1}}(\delta)
:=
256q_1(\delta)
\max_{i\in[n]}\overline{\mathfrak R}_i,
\]
then, with probability at least \(1-\delta/2\),
\[
\hat\Sigma_{z,i}
\succeq
\frac12\Sigma_{z,i},
\qquad i\in[n],
\]
and hence
\[
\hat\mu_{t,i}
\ge
\frac12\mu_{t,i},
\qquad i\in[n].
\]
\end{lemma}

\begin{proof}
For
\[
h_{i,\ell}^{\top}
:=
\sum_{j=1}^{\ell}
d_{i,j}e_i^\top G_{\ell-j},
\]
independence of the innovations gives
\[
\frac1t\sum_{s=0}^{t-1}
\mathbb E
\bigl(g_s^{(i)}-vx_s\bigr)^2
=
\frac{\sigma^2}{t}
\sum_{\ell=1}^{t-1}
(t-\ell)
\|h_{i,\ell}^{\top}-vG_{\ell-1}\|_2^2.
\]
For \(t\ge6\), retaining only \(\ell=1,2,3\) yields
\[
\mu_{t,i}
\ge
\frac{\sigma^2}{2}
\inf_v
\sum_{\ell=1}^3
\|h_{i,\ell}^{\top}-vG_{\ell-1}\|_2^2.
\]
If
\[
r=d_{i,1}e_i^\top-v,
\]
\[
u_2=rG_1+d_{i,2}e_i^\top,
\qquad
u_3=rG_2+d_{i,2}e_i^\top G_1+d_{i,3}e_i^\top,
\]
then
\[
\kappa_{i,\star}
\le
C_{\mathrm{der}}
\left(
\|r\|_2^2+\|u_2\|_2^2+\|u_3\|_2^2
\right).
\]
Therefore,
\[
\mu_{t,i}
\ge
\frac{\sigma^2\kappa_{i,\star}}
{2C_{\mathrm{der}}}
=
\mu_{i,\mathrm{lb}}.
\]

Next write
\[
\Sigma_{z,i}
=
\begin{bmatrix}
\Sigma_{x,t}&c_{t,i}\\
c_{t,i}^\top&q_{t,i}
\end{bmatrix},
\]
where
\[
\Sigma_{x,t}
:=
\frac1t\sum_{s=0}^{t-1}
\mathbb E[x_sx_s^\top],
\qquad
q_{t,i}
:=
\frac1t\sum_{s=0}^{t-1}
\mathbb E\bigl[(g_s^{(i)})^2\bigr].
\]
Since \(x_s\) contains the fresh innovation \(\eta_{s-1}\) for
\(s\ge1\),
\[
\Sigma_{x,t}\succeq\frac{\sigma^2}{2}I_n.
\]
The Schur complement of \(\Sigma_{x,t}\) in \(\Sigma_{z,i}\) is
\(\mu_{t,i}>0\). Hence \(\Sigma_{z,i}\succ0\).

We now derive a deterministic bound on the temporal dependence of the
whitened process. Let
\[
a_{t,i}:=\Sigma_{x,t}^{-1}c_{t,i}.
\]
The Schur-complement identity gives
\[
q_{t,i}
=
\mu_{t,i}
+
a_{t,i}^\top\Sigma_{x,t}a_{t,i}.
\]
In particular,
\[
a_{t,i}^\top\Sigma_{x,t}a_{t,i}
\le q_{t,i}.
\]

We first bound \(q_{t,i}\). The covariance decay in \Cref{lem:state_cov_decay} gives, for every \(s\),
\[
\begin{aligned}
\mathbb E\bigl[(g_s^{(i)})^2\bigr]
&\le
\widetilde C_x
\sum_{j,k\ge1}
|d_{i,j}||d_{i,k}|
\bigl(|j-k|+1\bigr)^{-(1+\alpha_{\min})}.
\end{aligned}
\]
Let
\[
\gamma_k:=(|k|+1)^{-(1+\alpha_{\min})},
\qquad k\in\mathbb Z.
\]
Young's convolution inequality gives
\[
\sum_{j,k\ge1}
|d_{i,j}||d_{i,k}|\gamma_{j-k}
\le
\|\gamma\|_{\ell_1(\mathbb Z)}
\sum_{j\ge1}d_{i,j}^2.
\]
Moreover,
\[
\|\gamma\|_{\ell_1(\mathbb Z)}
\le
1+\frac{2}{\alpha_{\min}}.
\]
Consequently,
\[
q_{t,i}
\le
\mathcal L_x D_{2,i}^2.
\]
Since
\(\Sigma_{x,t}\succeq\sigma^2 I_n/2\), it follows that
\[
\|a_{t,i}\|_2^2
\le
\frac{2q_{t,i}}{\sigma^2}
\le
\frac{2\mathcal L_xD_{2,i}^2}{\sigma^2},
\]
and therefore
\[
\|a_{t,i}\|_2
\le
\frac{\sqrt{2\mathcal L_x}}{\sigma}D_{2,i}.
\]

Fix \(i\) and a unit vector \(u\in\mathbb R^{n+1}\), and write
\[
w
:=
\Sigma_{z,i}^{-1/2}u
=
\begin{bmatrix}
\beta\\
b
\end{bmatrix}.
\]
Since \(w^\top\Sigma_{z,i}w=1\), the factorization
\[
\Sigma_{z,i}
=
\begin{bmatrix}
I&a_{t,i}\\
0&1
\end{bmatrix}^{\!\top}
\begin{bmatrix}
\Sigma_{x,t}&0\\
0&\mu_{t,i}
\end{bmatrix}
\begin{bmatrix}
I&a_{t,i}\\
0&1
\end{bmatrix}
\]
implies
\[
1
=
(\beta+ba_{t,i})^\top
\Sigma_{x,t}
(\beta+ba_{t,i})
+
\mu_{t,i}b^2.
\]
Therefore,
\[
\|\beta+ba_{t,i}\|_2
\le
\frac{\sqrt2}{\sigma},
\qquad
|b|
\le
\frac1{\sqrt{\mu_{i,\mathrm{lb}}}}.
\]
It follows that
\[
\begin{aligned}
\|\beta\|_2+D_i|b|
&\le
\|\beta+ba_{t,i}\|_2
+
\bigl(\|a_{t,i}\|_2+D_i\bigr)|b|\\
&\le
\frac{\sqrt2}{\sigma}
+
\frac{
D_i+\frac{\sqrt{2\mathcal L_x}}{\sigma}D_{2,i}
}{
\sqrt{\mu_{i,\mathrm{lb}}}
}.
\end{aligned}
\]

Now define the scalar whitened process
\[
r_s^{(i,u)}
:=
u^\top\Sigma_{z,i}^{-1/2}z_s^{(i)}
=
\beta^\top x_s+bg_s^{(i)}.
\]
It is a linear filter of \(x_s\) whose coefficient \(\ell_1\) norm is
at most
\[
F_i
:=
\frac{\sqrt2}{\sigma}
+
\frac{
D_i+\frac{\sqrt{2\mathcal L_x}}{\sigma}D_{2,i}
}{
\sqrt{\mu_{i,\mathrm{lb}}}
}.
\]
The covariance decay convolution bound in \Cref{lem:state_cov_decay} therefore gives
\[
\sup_p\sum_{k\in\mathbb Z}
\left|
\operatorname{Cov}
\bigl(r_p^{(i,u)},r_{p+k}^{(i,u)}\bigr)
\right|
\le
F_i^2\mathcal L_x
=
\overline{\mathfrak R}_i.
\]
If
\[
\mathcal T_{i,u}
:=
\operatorname{Cov}
\left(
r_0^{(i,u)},\ldots,r_{t-1}^{(i,u)}
\right),
\]
then
\[
\|\mathcal T_{i,u}\|_{\mathrm{op}}
\le
\overline{\mathfrak R}_i.
\]

Furthermore,
\[
\begin{aligned}
\operatorname{tr}(\mathcal T_{i,u})
&=
\sum_{s=0}^{t-1}
\mathbb E
\left[
\bigl(r_s^{(i,u)}\bigr)^2
\right]\\
&=
t\,u^\top
\Sigma_{z,i}^{-1/2}
\Sigma_{z,i}
\Sigma_{z,i}^{-1/2}u\\
&=t.
\end{aligned}
\]
Since \(\mathcal T_{i,u}\succeq0\),
\[
\|\mathcal T_{i,u}\|_F^2
\le
\|\mathcal T_{i,u}\|_{\mathrm{op}}
\operatorname{tr}(\mathcal T_{i,u})
\le t\overline{\mathfrak R}_i.
\]

Let
\[
r^{(i,u)}
:=
\left(
r_0^{(i,u)},\ldots,r_{t-1}^{(i,u)}
\right)^\top.
\]
The Gaussian quadratic-form inequality~\citep{hsu2012tail} gives, for every \(x>0\),
\[
\left|
\frac1t\|r^{(i,u)}\|_2^2-1
\right|
\le
2\sqrt{
\frac{\overline{\mathfrak R}_i x}{t}
}
+
2\frac{\overline{\mathfrak R}_i x}{t}
\]
with probability at least \(1-2e^{-x}\).

Equivalently,
\[
\left|
u^\top
\Sigma_{z,i}^{-1/2}
\bigl(\hat\Sigma_{z,i}-\Sigma_{z,i}\bigr)
\Sigma_{z,i}^{-1/2}
u
\right|
\le
2\sqrt{
\frac{\overline{\mathfrak R}_i x}{t}
}
+
2\frac{\overline{\mathfrak R}_i x}{t}.
\]

Apply this inequality to a \(1/4\)-net of the unit sphere in
\(\mathbb R^{n+1}\), whose cardinality is at most \(9^{n+1}\), and
take a union bound over \(i\in[n]\). With
\[
x=q_1(\delta)
=
(n+1)\log9+\log\frac{4n}{\delta},
\]
the total failure probability is at most
\[
2n9^{n+1}e^{-q_1(\delta)}
=
\frac{\delta}{2}.
\]
The standard \(1/4\)-net bound for symmetric matrices then yields
\[
\left\|
\Sigma_{z,i}^{-1/2}
\bigl(\hat\Sigma_{z,i}-\Sigma_{z,i}\bigr)
\Sigma_{z,i}^{-1/2}
\right\|_{\mathrm{op}}
\le
4\left(
\sqrt{
\frac{
\overline{\mathfrak R}_iq_1(\delta)
}{t}
}
+
\frac{
\overline{\mathfrak R}_iq_1(\delta)
}{t}
\right)
\]
simultaneously for all \(i\).

If
\[
t\ge
256q_1(\delta)
\max_i\overline{\mathfrak R}_i,
\]
then
\[
\sqrt{
\frac{
\overline{\mathfrak R}_iq_1(\delta)
}{t}
}
\le\frac1{16},
\qquad
\frac{
\overline{\mathfrak R}_iq_1(\delta)
}{t}
\le\frac1{256},
\]
so the preceding right-hand side is at most
\[
4\left(\frac1{16}+\frac1{256}\right)
=
\frac{17}{64}
<
\frac12.
\]
Hence
\[
\Sigma_{z,i}^{-1/2}
\hat\Sigma_{z,i}
\Sigma_{z,i}^{-1/2}
\succeq
\frac12I,
\]
and therefore
\[
\hat\Sigma_{z,i}
\succeq
\frac12\Sigma_{z,i}.
\]

Finally, for
\[
w(v):=(-v^\top,1)^\top,
\]
we have
\[
\hat\mu_{t,i}
=
\inf_v
w(v)^\top\hat\Sigma_{z,i}w(v)
\]
and
\[
\mu_{t,i}
=
\inf_v
w(v)^\top\Sigma_{z,i}w(v).
\]
Thus
\[
\hat\mu_{t,i}
\ge
\frac12\mu_{t,i},
\]
which proves event \(\mathrm{(E1)}\).
\end{proof}

\begin{lemma}[Global profiled lower isometry]
\label{lem:global_profiled_lower_isometry}
Suppose \Cref{ass:stable_A} holds. Define
\[
\gamma
:=
\min_{i\in[n]}
\inf_{\substack{
\alpha_i\in\mathcal A_{\epsilon,i}\\
|h_i(\alpha_i)|>r_{i,\mathrm{sep}}
}}
\left\{
R_t^{(i)}(\alpha_i)
-
R_t^{(i)}(\alpha_{i}^{\circ})
\right\},
\]
and suppose that
\(
\gamma>0.
\)
Let
$
F_{\mathrm G}
:=
\frac{\sqrt{2}}{\sigma}
+
\frac{S_1}{\sqrt{\gamma}}
\left(
1+\frac{\sqrt{2L_x}}{\sigma}
\right),$
and
\[
q_{\mathrm G}(\delta)
:=
(n+1)\log 9
+
\log\frac{2\sum_{i=1}^nM_i}{\delta}.
\]
If
\(
t\ge
\max\left\{
2,\,
256L_xF_{\mathrm G}^2q_{\mathrm G}(\delta)
\right\},
\)
then with probability at least $1-\delta$, simultaneously for every
$i\in[n]$ and every $\alpha_i\in\mathcal A_{\epsilon,i}$ satisfying
$|h_i(\alpha_i)|>r_{i,\mathrm{sep}}$,
\[
\underline{\mathcal Q}_{t,i}(\alpha_i)
\ge
\frac t2
\left[
R_t^{(i)}(\alpha_i)
-
R_t^{(i)}(\alpha_{i,\star})
\right]
\ge
\frac{t\gamma}{2}.
\]
\end{lemma}

\begin{proof}
Fix $i\in[n]$ and
$\alpha_i\in\mathcal A_{\epsilon,i}$ satisfying
$|h_i(\alpha_i)|>r_{i,\mathrm{sep}}$. Define
$
z_s^{(i,\alpha_i)}
:=
\begin{bmatrix}
x_s\\
b_s^{(i)}(\alpha_i)
\end{bmatrix},
$
$
\hat\Sigma_{i,\alpha_i}
:=
\frac1t
\sum_{s=0}^{t-1}
z_s^{(i,\alpha_i)}
z_s^{(i,\alpha_i)\top},$ and 
$\Sigma_{i,\alpha_i}
:=
\mathbb E\hat\Sigma_{i,\alpha_i}.$
For
$
w(v):=
\begin{bmatrix}
-v^\top\\
1
\end{bmatrix},
$
we have
\[
w(v)^\top z_s^{(i,\alpha_i)}
=
b_s^{(i)}(\alpha_i)-vx_s.
\]
Hence, by the definition of the profiled noiseless error,
\begin{equation}
\frac1t
\underline{\mathcal Q}_{t,i}(\alpha_i)
=
\inf_v
w(v)^\top
\hat\Sigma_{i,\alpha_i}
w(v).
\label{eq:global-li-empirical-profile}
\end{equation}
Write
\(
\Sigma_{i,\alpha_i}
=
\begin{bmatrix}
\Sigma_{x,t} & c_{i,\alpha_i}\\
c_{i,\alpha_i}^\top&q_{i,\alpha_i}
\end{bmatrix},
\)
where
\(
c_{i,\alpha_i}
=
\frac1t\sum_{s=0}^{t-1}
\mathbb E
\left[
x_sb_s^{(i)}(\alpha_i)
\right],\) and
\(
q_{i,\alpha_i}
=
\frac1t
\sum_{s=0}^{t-1}
\mathbb E
\left[
|b_s^{(i)}(\alpha_i)|^2
\right].
\)
The population-risk identity gives
\[
R_t^{(i)}(\alpha_i)
-
R_t^{(i)}(\alpha_{i,\star})
=
\inf_v
\frac1t
\sum_{s=0}^{t-1}
\mathbb E
\left[
|b_s^{(i)}(\alpha_i)-vx_s|^2
\right].
\]
Expanding the square,
\[
\frac1t
\sum_{s=0}^{t-1}
\mathbb E
|b_s^{(i)}(\alpha_i)-vx_s|^2
=
q_{i,\alpha_i}
-
2vc_{i,\alpha_i}
+
v\Sigma_{x,t}v^\top.
\]
For $t\ge2$, the innovation gives
\begin{equation}
\Sigma_{x,t}
\succeq
\frac{\sigma^2}{2}I_n.
\label{eq:global-li-state-lb}
\end{equation}
Thus $\Sigma_{x,t}\succ0$, and minimizing the preceding quadratic
function over $v$ yields
\(
v^\star
=
c_{i,\alpha_i}^\top
\Sigma_{x,t}^{-1}.
\)
Consequently,
\begin{equation}
R_t^{(i)}(\alpha_i)
-
R_t^{(i)}(\alpha_{i,\star})
=
q_{i,\alpha_i}
-
c_{i,\alpha_i}^\top
\Sigma_{x,t}^{-1}
c_{i,\alpha_i}.
\label{eq:global-li-schur}
\end{equation}
Equivalently,
\begin{equation}
R_t^{(i)}(\alpha_i)
-
R_t^{(i)}(\alpha_{i,\star})
=
\inf_v
w(v)^\top
\Sigma_{i,\alpha_i}
w(v).
\label{eq:global-li-pop-profile}
\end{equation}

Since
\[
R_t^{(i)}(\alpha_{i}^{\circ})
-
R_t^{(i)}(\alpha_{i,\star})
\ge0,
\]
the definition of $\gamma$ implies
\begin{equation}
R_t^{(i)}(\alpha_i)
-
R_t^{(i)}(\alpha_{i,\star})
\ge
\gamma.
\label{eq:global-li-gap}
\end{equation}
Hence the Schur complement in
\eqref{eq:global-li-schur} is strictly positive and
$\Sigma_{i,\alpha_i}\succ0$.

Let
\[
\Delta\psi_j
:=
\psi(\alpha_i,j)
-
\psi(\alpha_{i,\star},j).
\]
Then
\[
b_s^{(i)}(\alpha_i)
=
\sum_{j\ge1}
\Delta\psi_j x_{s+1-j}^{(i)}.
\]
By the fundamental theorem of calculus and the definition of $S_1$,
\begin{equation}
\sum_{j\ge1}
|\Delta\psi_j|
\le
S_1|h_i(\alpha_i)|.
\label{eq:global-li-l1-filter}
\end{equation}
Therefore,
\begin{equation}
\left(
\sum_{j\ge1}
|\Delta\psi_j|^2
\right)^{1/2}
\le
S_1|h_i(\alpha_i)|.
\label{eq:global-li-l2-filter}
\end{equation}
Let
\(
\gamma_k
:=
(|k|+1)^{-(1+\alpha_{\min})},\) where \(k\in\mathbb Z.
\)
The covariance-decay bound in \Cref{lem:state_cov_decay} gives
\[
\left|
\operatorname{Cov}
\left(
x_p^{(i)},x_q^{(i)}
\right)
\right|
\le
\widetilde C_x\gamma_{p-q}.
\]
Hence, for every $s$,
\begin{align*}
\mathbb E
|b_s^{(i)}(\alpha_i)|^2
\le
\widetilde C_x
\sum_{j,k\ge1}
|\Delta\psi_j|
|\Delta\psi_k|
\gamma_{j-k}
\le
\widetilde C_x
\|\gamma\|_{\ell_1(\mathbb Z)}
\sum_{j\ge1}
|\Delta\psi_j|^2,
\end{align*}
where the second inequality follows from Young's convolution inequality.
Since
\[
\|\gamma\|_{\ell_1(\mathbb Z)}
\le
1+\frac{2}{\alpha_{\min}},
\]
the definition of $L_x$ and
\eqref{eq:global-li-l2-filter} imply
\begin{equation}
q_{i,\alpha_i}
\le
L_xS_1^2|h_i(\alpha_i)|^2.
\label{eq:global-li-q-bound}
\end{equation}

Let
\(
m_{i,\alpha_i}
:=
\Sigma_{x,t}^{-1}c_{i,\alpha_i}.
\)
By \eqref{eq:global-li-schur},
\[
q_{i,\alpha_i}
=
\left[
R_t^{(i)}(\alpha_i)
-
R_t^{(i)}(\alpha_{i,\star})
\right]
+
m_{i,\alpha_i}^\top
\Sigma_{x,t}
m_{i,\alpha_i}.
\]
Thus
\[
m_{i,\alpha_i}^\top
\Sigma_{x,t}
m_{i,\alpha_i}
\le
q_{i,\alpha_i}.
\]
Combining this with
\eqref{eq:global-li-state-lb} and
\eqref{eq:global-li-q-bound} yields
\begin{equation}
\|m_{i,\alpha_i}\|_2
\le
\frac{\sqrt{2L_x}}{\sigma}
S_1|h_i(\alpha_i)|.
\label{eq:global-li-m-bound}
\end{equation}

Fix a unit vector $u\in\mathbb R^{n+1}$ and write
\(
\Sigma_{i,\alpha_i}^{-1/2}u
=
\begin{bmatrix}
\beta\\
b
\end{bmatrix}.
\)
The block factorization
\[
\Sigma_{i,\alpha_i}
=
\begin{bmatrix}
I&m_{i,\alpha_i}\\
0&1
\end{bmatrix}^{\!\top}
\begin{bmatrix}
\Sigma_{x,t}&0\\
0&
R_t^{(i)}(\alpha_i)-R_t^{(i)}(\alpha_{i,\star})
\end{bmatrix}
\begin{bmatrix}
I&m_{i,\alpha_i}\\
0&1
\end{bmatrix}
\]
and $u^\top u=1$ give
\[
1
=
(\beta+bm_{i,\alpha_i})^\top
\Sigma_{x,t}
(\beta+bm_{i,\alpha_i})
+
\left[
R_t^{(i)}(\alpha_i)-R_t^{(i)}(\alpha_{i,\star})
\right]b^2.
\]
Hence, by
\eqref{eq:global-li-state-lb} and
\eqref{eq:global-li-gap},
\begin{equation}
\|\beta+bm_{i,\alpha_i}\|_2
\le
\frac{\sqrt2}{\sigma},
\qquad
|b|
\le
\frac1{\sqrt{\gamma}}.
\label{eq:global-li-whitened-components}
\end{equation}

Define the scalar whitened process
\[
r_s
:=
u^\top
\Sigma_{i,\alpha_i}^{-1/2}
z_s^{(i,\alpha_i)}
=
\beta^\top x_s
+
b\,b_s^{(i)}(\alpha_i).
\] 
It is a linear filter of the state process. By
\eqref{eq:global-li-l1-filter},
\eqref{eq:global-li-m-bound}, and
\eqref{eq:global-li-whitened-components}, the sum of the Euclidean
norms of its filter coefficients is at most
\begin{align*}
\|\beta\|_2
+
|b|
\sum_{j\ge1}|\Delta\psi_j|
&\le
\|\beta+bm_{i,\alpha_i}\|_2
+
|b|
\left(
\|m_{i,\alpha_i}\|_2
+
\sum_{j\ge1}|\Delta\psi_j|
\right)\\
&\le
\frac{\sqrt2}{\sigma}
+
\frac{S_1|h_i(\alpha_i)|}{\sqrt{\gamma}}
\left(
1+\frac{\sqrt{2L_x}}{\sigma}
\right).
\end{align*}
Since $\alpha_i,\alpha_{i,\star}\in(0,1]$,
$|h_i(\alpha_i)|\le1$, and therefore this coefficient norm is at most
$F_{\mathrm G}$.

Writing the filter representation as
\[
r_s
=
\sum_{\ell\ge0}f_\ell^\top x_{s-\ell},
\qquad
\sum_{\ell\ge0}\|f_\ell\|_2
\le
F_{\mathrm G},
\]
we obtain
\begin{align}
\sum_{k\in\mathbb Z}
|\operatorname{Cov}(r_p,r_{p+k})|
\le
\widetilde C_x
\sum_{\ell,m\ge0}
\|f_\ell\|_2
\|f_m\|_2
\sum_{k\in\mathbb Z}
\gamma_{k+\ell-m}
\le
L_xF_{\mathrm G}^2.
\label{eq:T_i_alpha_bound}
\end{align}
Let $T_{i,\alpha_i,u}$ denote the covariance matrix of
$(r_0,\ldots,r_{t-1})^\top$. The preceding covariance-sum bound \eqref{eq:T_i_alpha_bound} yields
\begin{equation}
\|T_{i,\alpha_i,u}\|_{\mathrm{op}}
\le
L_xF_{\mathrm G}^2.
\label{eq:global-li-T-op}
\end{equation}
Moreover,
\begin{align*}
\operatorname{tr}(T_{i,\alpha_i,u})
=
\sum_{s=0}^{t-1}\mathbb E[r_s^2]
=
t\,u^\top
\Sigma_{i,\alpha_i}^{-1/2}
\Sigma_{i,\alpha_i}
\Sigma_{i,\alpha_i}^{-1/2}
u
=
t.
\end{align*}
Since $T_{i,\alpha_i,u}\succeq0$,
\begin{equation}
\|T_{i,\alpha_i,u}\|_F^2
\le
\|T_{i,\alpha_i,u}\|_{\mathrm{op}}
\operatorname{tr}(T_{i,\alpha_i,u})
\le
tL_xF_{\mathrm G}^2.
\label{eq:global-li-T-frob}
\end{equation}

The vector $(r_0,\ldots,r_{t-1})^\top$ is jointly Gaussian.
Thus, for every $x>0$, the Gaussian quadratic-form inequality with
\eqref{eq:global-li-T-op} and \eqref{eq:global-li-T-frob} gives
\[
\left|
\frac1t\sum_{s=0}^{t-1}r_s^2-1
\right|
\le
2\sqrt{\frac{L_xF_{\mathrm G}^2x}{t}}
+
2\frac{L_xF_{\mathrm G}^2x}{t}
\]
with probability at least $1-2e^{-x}$. Equivalently,
\[
\left|
u^\top
\Sigma_{i,\alpha_i}^{-1/2}
\left(
\hat\Sigma_{i,\alpha_i}
-
\Sigma_{i,\alpha_i}
\right)
\Sigma_{i,\alpha_i}^{-1/2}
u
\right|
\le
2\sqrt{\frac{L_xF_{\mathrm G}^2x}{t}}
+
2\frac{L_xF_{\mathrm G}^2x}{t}.
\]

Apply this inequality to a $1/4$-net of the unit sphere in
$\mathbb R^{n+1}$, whose cardinality is at most $9^{n+1}$, and take a
union bound over all grid points. Their total
number is at most $\sum_{i=1}^nM_i$. Taking
$x=q_{\mathrm G}(\delta)$ gives, with probability at least $1-\delta$,
simultaneously for all $i$ and $\alpha_i \in \mathcal{A}_{\epsilon,i}$ satisfying $|h_i(\alpha_i)| > r_\mathrm{i, sep}$,
\begin{equation}
\left\|
\Sigma_{i,\alpha_i}^{-1/2}
\left(
\hat\Sigma_{i,\alpha_i}
-
\Sigma_{i,\alpha_i}
\right)
\Sigma_{i,\alpha_i}^{-1/2}
\right\|_{\mathrm{op}}
\le
4
\left(
\sqrt{\frac{L_xF_{\mathrm G}^2q_{\mathrm G}(\delta)}{t}}
+
\frac{L_xF_{\mathrm G}^2q_{\mathrm G}(\delta)}{t}
\right).
\label{eq:global-li-relative-concentration}
\end{equation}
If
\[
t\ge
256L_xF_{\mathrm G}^2q_{\mathrm G}(\delta),
\]
then the right-hand side of
\eqref{eq:global-li-relative-concentration} is at most
\[
4\left(
\frac1{16}+\frac1{256}
\right)
=
\frac{17}{64}
<
\frac12.
\]
Therefore
\begin{equation}
\hat\Sigma_{i,\alpha_i}
\succeq
\frac12\Sigma_{i,\alpha_i}.
\label{eq:global-li-psd}
\end{equation}

Finally, \eqref{eq:global-li-psd} implies, for every $v$,
\[
w(v)^\top
\hat\Sigma_{i,\alpha_i}
w(v)
\ge
\frac12
w(v)^\top
\Sigma_{i,\alpha_i}
w(v).
\]
Taking the infimum over $v$ and using
\eqref{eq:global-li-empirical-profile} and
\eqref{eq:global-li-pop-profile},
\[
\frac1t
\underline{\mathcal Q}_{t,i}(\alpha_i)
\ge
\frac12
\left[
R_t^{(i)}(\alpha_i)
-
R_t^{(i)}(\alpha_{i,\star})
\right].
\]
Thus
\[
\underline{\mathcal Q}_{t,i}(\alpha_i)
\ge
\frac t2
\left[
R_t^{(i)}(\alpha_i)
-
R_t^{(i)}(\alpha_{i,\star})
\right].
\]
Finally, \eqref{eq:global-li-gap} gives
\[
\underline{\mathcal Q}_{t,i}(\alpha_i)
\ge
\frac{t\gamma}{2}.
\]
\end{proof}

\begin{lemma}[Concentration of the Taylor remainder]
\label{lem:E2}
Suppose \Cref{ass:stable_A} holds. Let
\begin{equation*}
\mathcal L_{\rho,i}
:=
\frac{1}{4}
\left(\frac{9e}{4}\,\mathfrak Z_i\right)^2
\mathcal L_x.
\end{equation*}
Define
\[
M_{\mathrm{loc}}
:=
\sum_{i=1}^n
\left|
\mathcal A_{\epsilon,i}
\cap[\underline\alpha_{i,\mathrm{loc}},1]
\right|,
\qquad
q_2(\delta):=\log\frac{2M_{\mathrm{loc}}}{\delta}.
\]
Then, with probability at least $1-\delta/2$, simultaneously for all
$i\in[n]$ and all
$\alpha_i\in\mathcal A_{\epsilon,i}
\cap[\underline\alpha_{i,\mathrm{loc}},1]$,
\begin{align}
\frac1t\sum_{s=0}^{t-1}|\rho_s^{(i)}(\alpha_i)|^2
\le
|h_i(\alpha_i)|^4
\left[
\frac{K_i^2}{4}
+2\mathcal L_{\rho,i}
\left(
\sqrt{\frac{q_2(\delta)}{t}}
+\frac{q_2(\delta)}{t}
\right)
\right].
\label{eq:E2_general_clean}
\end{align}
Consequently, if
\begin{equation}
t\ge t_{\mathrm{E2}}(\delta)
:=
q_2(\delta)
\max\left\{
\frac{16\mathcal L_x^2}
{\sigma^4\widetilde C_G^4\zeta(2)^2},
\frac{4\mathcal L_x}
{\sigma^2\widetilde C_G^2\zeta(2)}
\right\},
\label{eq:TE2_clean}
\end{equation}
then event \textnormal{(E2)} holds simultaneously over all local grid
points with probability at least $1-\delta/2$.
\end{lemma}

\begin{proof}
Fix $i$ and
$\alpha_i\in\mathcal A_{\epsilon,i}
\cap[\underline\alpha_{i,\mathrm{loc}},1]$.
For
\[
q_{i,j}(\alpha_i)
:=
\psi(\alpha_i,j)-\psi(\alpha_{i,\star},j)
-h_i(\alpha_i)\partial_\alpha\psi(\alpha_{i,\star},j),
\]
the integral Taylor formula gives
\[
q_{i,j}(\alpha_i)
=
\int_{\alpha_{i,\star}}^{\alpha_i}
(\alpha_i-u)\partial_u^2\psi(u,j)\,\mathrm du,
\qquad
\rho_s^{(i)}(\alpha_i)
=
\sum_{j\ge1}q_{i,j}(\alpha_i)x_{s+1-j}^{(i)}.
\]
Therefore, by \eqref{eq:S2_explicit_bound},
\begin{equation}
\sum_{j\ge1}|q_{i,j}(\alpha_i)|
\le
\frac{9e}{8}\,\mathfrak Z_i
|h_i(\alpha_i)|^2.
\label{eq:remainder_filter_l1_clean}
\end{equation}
Using \eqref{eq:remainder_filter_l1_clean} and the same convolution
calculation as in Lemma~\ref{lem:E1},
\[
\sup_p\sum_{k\in\mathbb Z}
|\operatorname{Cov}(\rho_p^{(i)}(\alpha_i),
\rho_{p+k}^{(i)}(\alpha_i))|
\le
\mathcal L_{\rho,i}|h_i(\alpha_i)|^4.
\]
Hence, for
$\rho_i(\alpha_i)=(\rho_0^{(i)}(\alpha_i),\ldots,
\rho_{t-1}^{(i)}(\alpha_i))^\top
\sim\mathcal N(0,\mathcal R_{i,\alpha_i})$,
\[
\|\mathcal R_{i,\alpha_i}\|_{\mathrm{op}}
\le\mathcal L_{\rho,i}|h_i(\alpha_i)|^4,
\qquad
\|\mathcal R_{i,\alpha_i}\|_F
\le\sqrt t\,\mathcal L_{\rho,i}|h_i(\alpha_i)|^4.
\]
The one-sided Gaussian quadratic-form inequality gives, with
probability at least $1-e^{-x}$,
\begin{align*}
\frac1t\sum_{s=0}^{t-1}|\rho_s^{(i)}(\alpha_i)|^2
&\le
\frac1t\sum_{s=0}^{t-1}
\mathbb E|\rho_s^{(i)}(\alpha_i)|^2\\
&\quad+
2\mathcal L_{\rho,i}|h_i(\alpha_i)|^4
\left(\sqrt{\frac{x}{t}}+\frac{x}{t}\right).
\end{align*}
Lemma~\ref{lem:taylor_remainder} bounds the expectation term by
$K_i^2|h_i(\alpha_i)|^4/4$. Taking $x=q_2(\delta)$ and applying a
union bound proves \eqref{eq:E2_general_clean}. Finally,
\eqref{eq:TE2_clean} implies
\[
2\mathcal L_{\rho,i}
\left(
\sqrt{\frac{q_2(\delta)}{t}}
+\frac{q_2(\delta)}{t}
\right)
\le\frac{K_i^2}{4},
\]
which proves event \textnormal{(E2)}.
\end{proof}
We now have all the tools needed to prove \Cref{thm:vector_alpha_error_bound}. We begin by restating \Cref{thm:vector_alpha_error_bound} in its full finite sample form for the estimation error of $\boldsymbol{\alpha}_\star$.
\begin{theorem}[Full error bound for $\boldsymbol{\alpha}_\star$]
\label{thm:vector_alpha_error_bound_formal}
Under \Cref{ass:stable_A}, define
\(
\epsilon_{\max}
:=
\max_{1\le i\le n}\epsilon_i,
\) and \(
\mu_{\min}
:=
\min_{1\le i\le n}\mu_{t,i}.
\)
Let 
\begin{equation}
\gamma
:=
\underset{1 \le i\le n}{\min}
\inf_{
\alpha\in\mathcal A_{\epsilon,i}, \, 
|\alpha_i - \alpha_{i, \star}|>\frac{\sqrt{\mu_{i,\mathrm{lb}}}}{2K_{i}}
}
\left\{
R_t^{(i)}(\alpha_i)
-
R_t^{(i)}(\alpha_{i}^{\circ})
\right\}
\label{eq: def_gamma}
\end{equation}
and suppose $\gamma>0$.
Fix $\delta\in(0,1/2)$ and choose
\(
0<\theta
\le
\min\left\{
\frac12,\,
\frac{
\delta\mu_{\min}
}{
576\sigma^2 n S_1^2
\operatorname{tr}(\Gamma_t)
}
\right\}.
\)
Define
\begin{align*}
\mathfrak B(\theta,\delta,k)
:= &
\frac{
18\theta\sigma^2nS_1^2t
\operatorname{tr}(\Gamma_t)
}{\delta}
+
\frac{
27\sigma^2nS_1^2t
\operatorname{tr}(\Gamma_t)
}{\delta}
\epsilon_{\max}^2
\nonumber\\
&+
\frac{8\sigma^2}{\theta}
\sum_{i=1}^n
\log\frac{9M_in}{\delta}
+
8\sigma^2\log\frac9\delta
+
\frac{
36(1+\theta^{-1})C^2\sigma^2n
\operatorname{tr}(\Gamma_t)
}{
\delta\lambda_{\min}(\Gamma_k)
}
\Xi_t(\delta,k),
\end{align*}
where \(
\Xi_t(\delta,k)
:=
n\log\frac{9n}{\delta}
+
\log\det\!\left(
\Gamma_t\Gamma_k^{-1}
\right).
\) If $t \gtrsim \frac{1}{\alpha_{\min}^4} \left( n + \log \frac{\sum_{i=1}^n M_i}{\delta} \right),$ and that, for some integer $k$,
\begin{equation}
\frac{t}{k}
\ge
c\,\Xi_t(\delta,k).
\label{eq:excitation conditions}
\end{equation}
Suppose that
\begin{equation}
\mathfrak B(\theta,\delta,k)
<
\frac{t\gamma}{2}.
\label{eq:global_localization_condition}
\end{equation}
Then, with probability at least $1-\delta$,
\begin{align}
\left\|
\hat{\boldsymbol\alpha}
-
\boldsymbol\alpha_\star
\right\|_\infty^2
\le{}&
\frac{
288\sigma^2nS_1^2
\operatorname{tr}(\Gamma_t)
}{
\delta\mu_{\min}
}
\epsilon_{\max}^2
\nonumber\\
&+
\frac{
256
}{
3t\mu_{\min}\theta
}
\Bigg[
\sigma^2
\left(
\sum_{i=1}^n
\log\frac{9M_in}{\delta}
+
\log\frac9\delta
\right)
+
\frac{
9C^2\sigma^2n
\operatorname{tr}(\Gamma_t)
}{
\delta\lambda_{\min}(\Gamma_k)
}
\Xi_t(\delta,k)
\Bigg].
\label{eq:alpha_fast_rate_finite}
\end{align}
Consequently, if
\(
\epsilon_{\max}
=
\mathcal{O}\left(t^{-1/2}\right),
\label{eq:grid_resolution_fast_rate}
\)
then
\(
\left\|
\hat{\boldsymbol{\alpha}}
-
\boldsymbol{\alpha}_\star
\right\|_\infty
=
\mathcal{O}\left(t^{-1/2}\right).
\)
\end{theorem}

\begin{proof}
We intersect three events: the global profiled lower isometry event
from \Cref{lem:global_profiled_lower_isometry}, the local
lower isometry event from \Cref{lem:row_lower_isometry}, and the in-sample
event from \Cref{lem:insample_fast_rate}, with failure probability
$\delta/3$ allocated to each. By $t \gtrsim \frac{1}{\alpha_{\min}^4} \left( n + \log \frac{\sum_{i=1}^n M_i}{\delta} \right)$, their intersection has
probability at least $1-\delta$. We work on this intersection
throughout.

\textbf{Step 1: Global localization.}
Suppose, toward a contradiction, that for some $i\in[n]$,
\[
|\hat\alpha_i-\alpha_{i,\star}|
>
r_{i,\mathrm{sep}}.
\]
Since the global profiled lower-isometry event holds simultaneously
over all grid points outside the separation neighborhood,
\[
\underline{\mathcal Q}_{t,i}(\hat\alpha_i)
\ge
\frac t2
\left[
R_t^{(i)}(\hat\alpha_i)
-
R_t^{(i)}(\alpha_{i,\star})
\right]
\ge
\frac{t\gamma}{2}.
\]
Recall
\begin{align*}
\mathcal E_t
=
\sum_{s=0}^{t-1}
\left\|
\left(
\Delta^{\hat{\boldsymbol\alpha}}
-
\Delta^{\boldsymbol\alpha^\star}
\right)x_{s+1}
-
(\hat A-A_\star)x_s
\right\|_2^2
=
\sum_{j=1}^n
\mathcal Q_{t,j}
\bigl(
\hat\alpha_j,
\hat a_j-a_{j,\star}
\bigr).
\end{align*}
Since every term in the sum is nonnegative and
$\underline{\mathcal Q}_{t,i}$ is obtained by profiling over the row
parameter,
\begin{equation}
\mathcal E_t
\ge
\mathcal Q_{t,i}
\bigl(
\hat\alpha_i,
\hat a_i-a_{i,\star}
\bigr)
\ge
\underline{\mathcal Q}_{t,i}(\hat\alpha_i)
\ge
\frac{t\gamma}{2}.
\label{eq:global_LI_localization_lower}
\end{equation}

We next obtain the corresponding upper bound from
\Cref{lem:insample_fast_rate}. Set $\tau=\rho=\theta$ and use failure
probability $\delta/3$. The Lemma~3 good event retained in its proof
also gives
\begin{equation}
X_tX_t^\top
\preceq
\frac{9\sigma^2nt}{\delta}\Gamma_t,
\qquad
\operatorname{tr}(X_tX_t^\top)
\le
\frac{9\sigma^2nt}{\delta}
\operatorname{tr}(\Gamma_t).
\label{eq:upper_gram_fast}
\end{equation}
Moreover, since both the true and estimated fractional orders lie in
$(0,1]$,
\[
\left\|
\hat{\boldsymbol\alpha}
-
\boldsymbol\alpha^\star
\right\|_\infty
\le1.
\]
Substituting these two bounds into
\Cref{lem:insample_fast_rate} yields
\[
\mathcal E_t
\le
\mathfrak B(\theta,\delta,k).
\]
By $\mathfrak B(\theta,\delta,k)
<
\frac{t\gamma}{2}$,
\[
\mathcal E_t
<
\frac{t\gamma}{2},
\]
contradicting
\eqref{eq:global_LI_localization_lower}. Therefore,
\begin{equation}
|\hat\alpha_i-\alpha_{i,\star}|
\le
r_{i,\mathrm{sep}},
\qquad i\in[n].
\label{eq:coarse_localization_fast}
\end{equation}

We now verify that this coarse localization places the estimator in
the local sets $\mathcal G_i$. Recall that
\[
\mathfrak Z_i
=
\sum_{j\ge1}
\frac{(1+\log j)^2}
{j^{1+\alpha_{i,\star}/2}}
\ge
\zeta\!\left(
1+\frac{\alpha_{i,\star}}2
\right)
\ge
\frac{2}{\alpha_{i,\star}}.
\]
Since $G_0=I_n$, $\widetilde C_G\ge1$. Moreover,
$|d_{i,2}|\le1/2$ and $|d_{i,3}|\le1/3$, so
$\kappa_{i,\star}\le13/36$. Hence
\[
r_{i,\mathrm{sep}}
\le
\frac{
\alpha_{i,\star}\sqrt{13}
}{
108e\sqrt{\zeta(2)}
}
<
\frac{\alpha_{i,\star}}2.
\]
Together with
\eqref{eq:coarse_localization_fast}, this implies
\[
\hat\alpha_i
>
\frac{\alpha_{i,\star}}2
=
\alpha_{i,\mathrm{loc}}.
\]
Since
$\hat\alpha_i\in\mathcal A_{\epsilon,i}\subset(0,1]$,
\[
\hat\alpha_i
\in
\mathcal A_{\epsilon,i}
\cap
[\alpha_{i,\mathrm{loc}},1].
\]
Furthermore, the population curvature lower bound gives
\[
\mu_{t,i}
\ge \mu_{i, \mathrm{lb}},
\qquad
r_{i,\mathrm{sep}}
=
\frac{\sqrt{\mu_{i,\mathrm{lb}}}}{2K_i}.
\]
Hence
\[
|\hat\alpha_i-\alpha_{i,\star}|
\le
\frac{\sqrt{\mu_{i,\mathrm{lb}}}}{2K_i},
\]
and therefore
\[
\hat\alpha_i\in\mathcal G_i,
\qquad i\in[n].
\]
\textbf{Step 2: Local lower isometry.}
Since $\hat\alpha_i\in\mathcal G_i$ for every row,
\Cref{lem:row_lower_isometry} gives
\begin{align}
\mathcal E_t
&=
\sum_{i=1}^n
\mathcal Q_{t,i}
\bigl(
\hat\alpha_i,
\hat a_i-a_{i,\star}
\bigr)
\nonumber\\
&\ge
\frac{t}{8}
\sum_{i=1}^n
\mu_{t,i}
|\hat\alpha_i-\alpha_{i,\star}|^2
\nonumber\\
&\ge
\frac{t\mu_{\min}}{8}
\left\|
\hat{\boldsymbol\alpha}
-
\boldsymbol\alpha^\star
\right\|_\infty^2.
\label{eq:lower_isometry_summed_fast}
\end{align}
\textbf{Step 3: In-sample upper bound and absorption.}
We reuse the same in-sample event from Step~1; no additional failure
probability is required. Setting $\tau=\rho=\theta$ in
\Cref{lem:insample_fast_rate} and substituting
\eqref{eq:upper_gram_fast} gives
\begin{align}
\mathcal E_t
\le{}&
\frac{
18\theta\sigma^2nS_1^2t
\operatorname{tr}(\Gamma_t)
}{\delta}
\left\|
\hat{\boldsymbol\alpha}
-
\boldsymbol\alpha^\star
\right\|_\infty^2
\nonumber\\
&+
\frac{
27\sigma^2nS_1^2t
\operatorname{tr}(\Gamma_t)
}{\delta}
\epsilon_{\max}^2
\nonumber\\
&+
\frac{8\sigma^2}{\theta}
\sum_{i=1}^n
\log\frac{9M_in}{\delta}
+
8\sigma^2\log\frac9\delta
\nonumber\\
&+
\frac{
36(1+\theta^{-1})C^2\sigma^2n
\operatorname{tr}(\Gamma_t)
}{
\delta\lambda_{\min}(\Gamma_k)
}
\Xi_t(\delta,k).
\label{eq:insample_after_energy_fast}
\end{align}

By \(
0<\theta
\le
\min\left\{
\frac12,\,
\frac{
\delta\mu_{\min}
}{
576\sigma^2 n S_1^2
\operatorname{tr}(\Gamma_t)
}
\right\},
\)
\[
\frac{
18\theta\sigma^2nS_1^2
\operatorname{tr}(\Gamma_t)
}{\delta}
\le
\frac{\mu_{\min}}{32}.
\]
Combining
\eqref{eq:lower_isometry_summed_fast} and
\eqref{eq:insample_after_energy_fast} yields
\begin{align}
\frac{3t\mu_{\min}}{32}
\left\|
\hat{\boldsymbol\alpha}
-
\boldsymbol\alpha^\star
\right\|_\infty^2
\le{}&
\frac{
27\sigma^2nS_1^2t
\operatorname{tr}(\Gamma_t)
}{\delta}
\epsilon_{\max}^2
\nonumber\\
&+
\frac{8\sigma^2}{\theta}
\sum_{i=1}^n
\log\frac{9M_in}{\delta}
+
8\sigma^2\log\frac9\delta
\nonumber\\
&+
\frac{
36(1+\theta^{-1})C^2\sigma^2n
\operatorname{tr}(\Gamma_t)
}{
\delta\lambda_{\min}(\Gamma_k)
}
\Xi_t(\delta,k).
\label{eq:absorbed_fast_rate}
\end{align}

Since $\theta\le1/2$,
\[
1+\theta^{-1}
\le
\frac2\theta,
\qquad
1\le
\frac1\theta.
\]
Dividing
\eqref{eq:absorbed_fast_rate}
by $3t\mu_{\min}/32$ therefore gives
\begin{align}
\left\|
\hat{\boldsymbol\alpha}
-
\boldsymbol\alpha^\star
\right\|_\infty^2
\le{}&
\frac{
288\sigma^2nS_1^2
\operatorname{tr}(\Gamma_t)
}{
\delta\mu_{\min}
}
\epsilon_{\max}^2
\nonumber\\
&+
\frac{
256
}{
3t\mu_{\min}\theta
}
\Bigg[
\sigma^2
\left(
\sum_{i=1}^n
\log\frac{9M_in}{\delta}
+
\log\frac9\delta
\right)
\nonumber\\
&\hspace{1.8cm}
+
\frac{
9C^2\sigma^2n
\operatorname{tr}(\Gamma_t)
}{
\delta\lambda_{\min}(\Gamma_k)
}
\Xi_t(\delta,k)
\Bigg],
\end{align}
which proves
\eqref{eq:alpha_fast_rate_finite}.
\end{proof}

\section{Proof of \texorpdfstring{\Cref{thm:A_error_bound_vec}}{boundA}}
\label[appendix]{sec:proof_thm_2}
\begin{proof}
\Cref{thm:A_error_bound_vec} follows immediately by applying \Cref{lem:bound_for_noise_term,lem:bound_for_bias_term} to \eqref{eqn:A_error}.  
\end{proof}
\subsection{Proof of \texorpdfstring{\Cref{lem:bound_for_noise_term}}{boundnoise}}
\label[appendix]{pf:bound_for_noise}
\begin{proof}
\begin{equation*}
\bigl\|W_t X_t^\top (X_t X_t^\top)^{-1}\bigr\| = \bigl\|\underbrace{W_t X_t^\top (X_tX_t^\top)^{-\frac{1}{2}}}_{:=T_1} \underbrace{(X_tX_t^\top)^{-\frac{1}{2}}}_{:=T_2}\bigr\|    
\end{equation*}
We first prove the process $x_t$ satisfies a block martingale small-ball condition (similar to Proposition 3.1 \cite{pmlr-v75-simchowitz18a}). Let $\mathcal{F}_s := \sigma(\eta_\tau: \tau \le s-1)$. For any $l \ge 1$,
\begin{equation*}
    x_{s+l} = \sum_{m=0}^{l-1}G_m \eta_{s+l-1-m} + \sum_{m=l}^{\infty}G_m \eta_{s+l-1-m},
\end{equation*}
so 
\begin{equation*}
    x_{s+l}|\mathcal{F}_s \sim \mathcal{N}(0, \sigma^2H_l), \qquad H_l = \sum_{m=0}^{l-1}G_m G_m^\top.
\end{equation*}
For any $w \in \mathbb{S}^{n-1}$, 
\begin{equation*}
    \langle w, x_{s+l}  \rangle | \mathcal{F}_s \sim \mathcal{N}(0, \sigma^2 w^T H_l w). 
\end{equation*}
Since (i) \(H_\ell \succeq H_{k'}\) for \(\ell \ge k'\), (ii) Paley-Zygmund lower bound, we have
\begin{equation*}
\mathbb{P}\!\left(
\left|\langle w,x_{s+\ell}\rangle\right|
\ge
\sigma \sqrt{w^\top H_{k'} w}
\,\middle|\, \mathcal{F}_s
\right)
\ge
\frac{3}{10},
\qquad
\ell \ge k'.
\end{equation*}
Thus,
\begin{equation*}
\frac{1}{k}\sum_{l=1}^{k}
\mathbb{P}\!\left(\left|\langle w,x_{s+l}\rangle\right|\ge \sigma\sqrt{w^\top H_{k'}w}\right)
\;\ge\;
\frac{1}{k}\sum_{l=k'}^{k}
\mathbb{P}\!\left(\left|\langle w,x_{s+l}\rangle\right|\ge \sigma\sqrt{w^\top H_{k'}w}\right)
\;\ge\;
\frac{3}{10}\cdot\frac{k-k'+1}{k}.
\end{equation*}
Pick \(k'=\lfloor k/2\rfloor\). Therefore \(x_t\) satisfies a \((k,\Gamma_{\mathrm{sb}},p)\)-BMSB condition with \(\Gamma_{\mathrm{sb}}=\sigma^2 H_{\lfloor k/2\rfloor}, p=\frac{3}{20}.\)\\
Let $Z_i = \langle w, x_i \rangle$. Then $Z_i$ satisfies the $(k, v_w, p)$, where $v_w = \sqrt{w^\top \Gamma_{sb}w}, p =\frac{3}{20}$. Proposition 2.5 says that If \((Z_i)_{i=1}^T\) satisfies \((k,\nu,p)\)-BMSB, then
\begin{equation*}
\Pr\!\left(\sum_{i=1}^T Z_i^2 \le \frac{\nu^2 p^2}{8}\,k\lfloor T/k\rfloor\right)
\le
\exp\!\left(-\lfloor T/k\rfloor p^2/8\right).
\end{equation*}
Applying this with \(T=t\) and \(\nu=\nu_w\),
\begin{equation*}
\Pr\!\left(
w^\top X_t X_t^\top w
\le
\frac{p^2}{8}\,k\lfloor t/k\rfloor\, w^\top \Gamma_{\mathrm{sb}} w
\right)
\le
\exp\!\left(-\lfloor t/k\rfloor p^2/8\right).
\end{equation*}
By Lemma 4.1 \cite{pmlr-v75-simchowitz18a} ( in our notation $Q = X_t^\top$) , if \(\inf_{w \in \mathcal{T}}w^\top X_t X_t^\top w \ge 1\) and \(X_t X_t^\top \lesssim \Gamma_{\mathrm{max}}\), then
\begin{equation}
\label{eqn:bound_sing_Xt}
    X_t X_t^\top \succeq \frac{\Gamma_{\min}}{2} = \frac{p^2}{16}\,k\lfloor t/k\rfloor \Gamma_{\mathrm{sb}}.
\end{equation}
Define the following events ($\bar{\Gamma} = \sum_{m=0}^{t-1}G_mG_m^\top$)
\begin{equation*}
\mathcal{E}_1:=\left\{\|W_t V\|_{\mathrm{op}}\ge K\right\},
\qquad
\mathcal{E}_2:=\left\{X_t X_t^\top \succeq \frac{k\lfloor t/k\rfloor p^2 \Gamma_{\mathrm{sb}}}{16}\right\},
\qquad
\mathcal{E}_3:=\left\{X_t X_t^\top \npreceq t\bar{\Gamma} \right\}.
\end{equation*}
\begin{align*}
\mathbb{P}\!\left[
\left\{
\bigl\|W_t X_t^\top (X_t X_t^\top)^{-1}\bigr\|_\mathrm{op}
\ge
\frac{4K}{p\sqrt{k\lfloor t/k\rfloor\,\lambda_{\min}(\Gamma_{\mathrm{sb}})}}
\right\}
\right]
\le \mathbb{P}[\mathcal{E}_1\cap\mathcal{E}_2\cap\mathcal{E}_3^c]
+
\mathbb{P}[\mathcal{E}_2^c\cap\mathcal{E}_3^c]
+
\mathbb{P}[\mathcal{E}_3].
\end{align*}
By the same argument of bounding \( \mathbb{P}[\mathcal{E}_1\cap\mathcal{E}_2\cap\mathcal{E}_3^c]\) and \(\mathbb{P}[\mathcal{E}_2^c\cap\mathcal{E}_3^c]\) and by \(\mathbb{P}[\mathcal{E}_3] \le \delta\), we have
\begin{equation*}
\mathbb{P}\!\left[
\bigl\|W_t X_t^\top (X_t X_t^\top)^{-1}\bigr\|_\mathrm{op}
>
\frac{90\sigma}{p}
\sqrt{
\frac{
n+n\log\frac{10}{p}
+\log\det\overline{\Gamma} \Gamma_{\mathrm{sb}}^{-1}
+\log\!\left(\frac{1}{\delta}\right)
}{
t\,\lambda_{\min}(\Gamma_{\mathrm{sb}})
}
}
\right]
\le 3\delta,
\end{equation*}
if
\begin{equation*}
t \ge \frac{10k}{p^2}
\left(
\log\!\left(\frac{1}{\delta}\right)
+ 2d\log(10/p)
+ \log\det(\Gamma_{\mathrm{sb}}^{-1})
\right).
\end{equation*}
Since
\begin{equation*}
\mathbb{P}\!\left[X_t X_t^\top \npreceq \frac{\sigma^2 d}{\delta}\,t \Gamma_t\right]
\le 
\frac{\delta}{d \sigma^2}\,
\mathbb{E}\!\left[
\operatorname{tr}\!\left((t\Gamma_t)^{-1/2}X_t X_t^\top (t\Gamma_t)^{-1/2}\right)
\right]
\le
\delta,
\end{equation*}
then there exist universal constants \(c,C>0\) such that
\begin{equation*}
\mathbb{P}\!\left[
\bigl\|W_t X_t^\top (X_t X_t^\top)^{-1}\bigr\|_\mathrm{op}
>
\frac{C}{\sqrt{T\,\lambda_{\min}(\Gamma_k)}}
\sqrt{
d\log\frac{d}{\delta}
+
\log\det(\Gamma_T\Gamma_k^{-1})
}
\right]
\le \delta,
\end{equation*}
for any \(k\) such that
\(
\frac{T}{k}
\ge
c\!\left(
d\log(d/\delta)
+
\log\det(\Gamma_T\Gamma_k^{-1})
\right)
\)
holds.
\end{proof}

\subsection{Proof of \texorpdfstring{\Cref{lem:bound_for_bias_term}}{boundbias}}
\label[appendix]{sec:bound_for_bias}

\begin{proof}
For each row $i\in[n]$ and $s=0,\ldots,t-1$, recall that
\[
b_s^{(i)}(\alpha_i)
=
\bigl(\Delta^{\alpha_i}-\Delta^{\alpha_{i,\star}}\bigr)
x_{s+1}^{(i)}.
\]
Since the $j=0$ coefficient does not depend on $\alpha_i$, we have
\begin{align*}
b_s^{(i)}(\alpha_i)
&=
\sum_{j=1}^{s+1}
\bigl(
    \psi(\alpha_i,j)-\psi(\alpha_{i,\star},j)
\bigr)
x_{s+1-j}^{(i)}  \\
&=
\int_{\alpha_{i,\star}}^{\alpha_i}
\sum_{j=1}^{s+1}
\partial_u\psi(u,j)\,
x_{s+1-j}^{(i)}
\,du .
\end{align*}

Let
\[
B_t^{(i)}(\alpha_i)
:=
\bigl(
b_0^{(i)}(\alpha_i),\ldots,
b_{t-1}^{(i)}(\alpha_i)
\bigr)
\in\mathbb R^{1\times t},
\]
and
\[
B_t(\boldsymbol{\alpha})
:=
\begin{bmatrix}
B_t^{(1)}(\alpha_1)\\
\vdots\\
B_t^{(n)}(\alpha_n)
\end{bmatrix}
\in\mathbb R^{n\times t}.
\]

For each row $i$, define
\[
S_{1,i}
:=
\sup_{u\in\mathcal A_i}
\sum_{j\ge1}
|\partial_u\psi(u,j)|,
\qquad
S_1:=\max_{i\in[n]}S_{1,i}.
\]
By Minkowski's integral inequality and Young's convolution
inequality,
\begin{align*}
\|B_t^{(i)}(\alpha_i)\|_2
&\le
|\alpha_i-\alpha_{i,\star}|
\sup_{u\in\mathcal A_i}
\left(
\sum_{s=0}^{t-1}
\left|
\sum_{j=1}^{s+1}
\partial_u\psi(u,j)
x_{s+1-j}^{(i)}
\right|^2
\right)^{1/2} \\
&\le
|\alpha_i-\alpha_{i,\star}|
\sup_{u\in\mathcal A_i}
\left(
\sum_{j=1}^{t}
|\partial_u\psi(u,j)|
\right)
\left(
\sum_{s=0}^{t-1}|x_s^{(i)}|^2
\right)^{1/2} \\
&\le
S_{1,i}
|\alpha_i-\alpha_{i,\star}|
\left(
\sum_{s=0}^{t-1}|x_s^{(i)}|^2
\right)^{1/2}.
\end{align*}
Consequently,
\begin{align*}
\|B_t(\boldsymbol{\alpha})\|_F^2
=
\sum_{i=1}^n
\|B_t^{(i)}(\alpha_i)\|_2^2 
\le
S_1^2
\|\boldsymbol{\alpha}-\boldsymbol{\alpha}_\star\|_\infty^2
\sum_{i=1}^n\sum_{s=0}^{t-1}|x_s^{(i)}|^2 
=
S_1^2
\|\boldsymbol{\alpha}-\boldsymbol{\alpha}_\star\|_\infty^2
\|X_t\|_F^2 .
\end{align*}

On the event $X_tX_t^\top\succ0$,
\begin{align}
\bigl\|
B_t(\boldsymbol{\alpha})
X_t^\top
(X_tX_t^\top)^{-1}
\bigr\|_{\mathrm{op}}
&\le
\|B_t(\boldsymbol{\alpha})\|_F
\bigl\|
X_t^\top(X_tX_t^\top)^{-1}
\bigr\|_{\mathrm{op}} \notag\\
&=
\frac{\|B_t(\boldsymbol{\alpha})\|_F}
{\sigma_{\min}(X_t)} \notag\\
&\le
S_1
\|\boldsymbol{\alpha}-\boldsymbol{\alpha}_\star\|_\infty
\frac{\|X_t\|_F}{\sigma_{\min}(X_t)} .
\label{eq:bias_reduce_condition}
\end{align}

It remains to control
$\|X_t\|_F/\sigma_{\min}(X_t)$.
Since $x_0=0$, the finite-past representation gives, for $s\ge1$,
\[
x_s
=
\sum_{m=0}^{s-1}
G_m\eta_{s-1-m}.
\]
Therefore,
\begin{align*}
\mathbb E\|X_t\|_F^2
=
\sum_{s=0}^{t-1}\mathbb E\|x_s\|_2^2 
=
\sigma^2
\sum_{s=1}^{t-1}
\sum_{m=0}^{s-1}
\|G_m\|_F^2 
\le
t\sigma^2
\sum_{m=0}^{\infty}
\|G_m\|_F^2 .
\end{align*}
Let
\[
\mathcal G_2
:=
\sum_{m=0}^{\infty}\|G_m\|_F^2.
\]
By Markov's inequality, with probability at least $1-\delta$,
\begin{equation}
\|X_t\|_F^2
\le
\frac{t\sigma^2\mathcal G_2}{\delta}.
\label{eq:Xt_energy_hp}
\end{equation}

Moreover, by the block small-ball lower bound, with probability
at least $1-\delta$,
\begin{equation}
X_tX_t^\top
\succeq
\frac{p^2}{16}
k\lfloor t/k\rfloor
\Gamma_{\mathrm{sb}} .
\label{eq:Xt_lower_hp}
\end{equation}
Hence, on the intersection of
\eqref{eq:Xt_energy_hp} and \eqref{eq:Xt_lower_hp},
\begin{align*}
\frac{\|X_t\|_F^2}
{\sigma_{\min}^2(X_t)}
=
\frac{\|X_t\|_F^2}
{\lambda_{\min}(X_tX_t^\top)} 
\le
\frac{
16t\sigma^2\mathcal G_2
}{
\delta p^2
k\lfloor t/k\rfloor
\lambda_{\min}(\Gamma_{\mathrm{sb}})
} 
\le
\frac{
16t\sigma^2\mathcal G_2
}{
\delta p^2
(t-k)
\lambda_{\min}(\Gamma_{\mathrm{sb}})
}.
\end{align*}
If $t\ge10k$, then $t/(t-k)\le10/9$, and thus
\begin{equation}
\frac{\|X_t\|_F}
{\sigma_{\min}(X_t)}
\le
\sqrt{
\frac{
160\sigma^2\mathcal G_2
}{
9\delta p^2
\lambda_{\min}(\Gamma_{\mathrm{sb}})
}
}.
\label{eq:Xt_condition_bound}
\end{equation}

Combining
\eqref{eq:bias_reduce_condition} and
\eqref{eq:Xt_condition_bound}, we obtain, with probability at
least $1-2\delta$,
\begin{equation*}
\bigl\|
B_t(\boldsymbol{\alpha})
X_t^\top
(X_tX_t^\top)^{-1}
\bigr\|_{\mathrm{op}}
\le
S_1
\|\boldsymbol{\alpha}-\boldsymbol{\alpha}_\star\|_\infty
\sqrt{
\frac{
160\sigma^2\mathcal G_2
}{
9\delta p^2
\lambda_{\min}(\Gamma_{\mathrm{sb}})
}
}.
\end{equation*}

Finally, Lemma~\ref{lem:state_cov_decay} gives
\[
\mathcal G_2
\le
n\widetilde C_G^2
\sum_{m=0}^{\infty}
(m+1)^{-2(1+\alpha_{\min})}
=
n\tilde{C}_G^2
\zeta(2+2\alpha_{\min})
\le
2n\tilde{C}_G^2 ,
\]
and therefore
\begin{equation}
\bigl\|
B_t(\boldsymbol{\alpha})
X_t^\top
(X_tX_t^\top)^{-1}
\bigr\|_{\mathrm{op}}
\le
S_1
\|\boldsymbol{\alpha}-\boldsymbol{\alpha}_\star\|_\infty
\sqrt{
\frac{
320n\sigma^2\tilde C_G^2
}{
9\delta p^2
\lambda_{\min}(\Gamma_{\mathrm{sb}})
}
}.
\end{equation}
This proves the claim.
\end{proof}

\section{Auxiliary Lemmas}
\begin{lemma}
\label{lem:state_cov_decay}
Suppose Assumption~1 holds and let
\(
\alpha_{\min}:=\min_{i\in[n]}\alpha_{i,\star}.
\)
Then there exists a constant $\widetilde C_G<\infty$,
depending only on $(A_\star,\alpha_\star)$, such that
\begin{equation}
\|G_m\|_{\mathrm{op}}
\le
\widetilde C_G (m+1)^{-(1+\alpha_{\min})},
\qquad m\ge 0.
\label{eq:def_CG}
\end{equation}
Consequently, for every
$p\ge1$ and $k\ge0$,
\begin{equation}
\|\operatorname{Cov}(x_p,x_{p+k})\|_{\mathrm{op}}
\le
\widetilde C_x (k+1)^{-(1+\alpha_{\min})},
\label{eq:def_Cx}
\end{equation}
where
\[
\widetilde C_x
:=
\sigma^2\widetilde C_G^2
\left(
1+\frac{1}{\alpha_{\min}}
+\frac{1}{2\alpha_{\min}+1}
\right).
\]
\end{lemma}
\begin{proof}
we first bound $G_m$. Let $\Psi_j := \operatorname{diag}\bigl(\psi(\alpha_{1, \star},j),\ldots,\psi(\alpha_{n, \star},j)\bigr).$ We have
\begin{equation*}
\mathcal{A}(z)
=
A_\star + \operatorname{diag}(\boldsymbol{\alpha}_\star)
-
\sum_{j=1}^{\infty} \Psi_{j+1} z^j.
\end{equation*}
For each coordinate \(i\),
$
\sum_{k=0}^{\infty} \psi(\alpha_{i,\star},k)\,z^k = (1-z)^{\alpha_{i,\star}}.
$
Therefore
\begin{equation*}
\sum_{j=1}^{\infty} \psi(\alpha_{i,\star},j+1)\,z^j
=
\frac{(1-z)^{\alpha_{i,\star}} - 1 + \alpha_{i,\star} z}{z}.
\end{equation*}
Define
\(
D_{\boldsymbol{\alpha}_\star}(z)
:=
\operatorname{diag}\bigl((1-z)^{\alpha_{1,\star}},\ldots,(1-z)^{\alpha_{n,\star}}\bigr).
\)
Then
\begin{equation*}
\mathcal{G}(z) = \bigl(D_{\boldsymbol{\alpha}_\star}(z) - zA_\star\bigr)^{-1}.    
\end{equation*}
Let \(w = 1-z\) and \(D(w) := \operatorname{diag}\bigl(w^{\alpha_{1,\star}},\ldots,w^{\alpha_{n,\star}}\bigr).\) Then
\begin{equation*}
\mathcal{G}(z)
=
-\,\bigl(I-B(w)\bigr)^{-1}A_\star^{-1},    
\end{equation*}
where $B(w):=wI+A_\star^{-1}D(w)$. Define
\[
R(z)
:=
B(w)^2(I-B(w))^{-1}A_\star^{-1}.
\] Then, we have 
\begin{equation*}
G_m
=
-\,A_\star^{-1}\operatorname{diag}\bigl(\psi(\alpha_{1,\star},m),\ldots,\psi(\alpha_{n,\star},m)\bigr)A_\star^{-1}
-[z^m]R(z),
\quad m\ge 2.    
\end{equation*}
For $|w|\le1$,
\[
\|D(w)\|_{\mathrm{op}}
=
\max_i |w|^{\alpha_{i,\star}}
\le
|w|^{\alpha_{\min}},
\]
and therefore
\[
\|B(w)\|_{\mathrm{op}}
\le
\bigl(1+\|A_\star^{-1}\|_{\mathrm{op}}\bigr)
|w|^{\alpha_{\min}}.
\]
Choose $\rho>0$ sufficiently small so that
$\|B(w)\|_{\mathrm{op}}\le1/2$ whenever $|w|\le\rho$.
Then
\[
\|(I-B(w))^{-1}\|_{\mathrm{op}}\le2,
\]
and thus
\begin{equation*}
\|R(z)\|_{\mathrm{op}}
\le
2\|A_\star^{-1}\|_{\mathrm{op}}\|B(w)\|_{\mathrm{op}}^2
\le
2\|A_\star^{-1}\|_{\mathrm{op}}\bigl(1+\|A_\star^{-1}\|_{\mathrm{op}}\bigr)^2 |w|^{2\alpha_{\min}}.    
\end{equation*}
Using the singularity transfer argument in \Cref{lem:explicit-transfer}, we have
\begin{equation*}
\|[z^m]R(z)\|_{\mathrm{op}}
\le
C_I \|A_\star^{-1}\|_{\mathrm{op}}\bigl(1+\|A_\star^{-1}\|_{\mathrm{op}}\bigr)^2\, m^{-(1+2\alpha_{\min})},    
\end{equation*}
where \(C_I\) depends only on the interval for \(\alpha_{i,\star}\).

We next bound the fractional coefficient $\psi(\alpha,j)$. For
$0<\alpha<1$,
\[
\psi(\alpha,j)
=
\frac{\Gamma(j-\alpha)}
{\Gamma(-\alpha)\Gamma(j+1)}
=
(-1)^j\binom{\alpha}{j}.
\]
For $j\ge2$,
\[
|\psi(\alpha,j)|
=
\frac{\alpha}{j}
\prod_{k=1}^{j-1}
\left(1-\frac{\alpha}{k}\right).
\]
Using $\log(1-u)\le-u$ for $0<u<1$ and
$\sum_{k=1}^{j-1}k^{-1}\ge\log j$, we obtain
\[
|\psi(\alpha,j)|
\le
\alpha j^{-(1+\alpha)},
\qquad j\ge1.
\]
For $\alpha=1$, the same bound follows directly from
$\psi(1,1)=-1$ and $\psi(1,j)=0$ for $j\ge2$.

Hence
\[
\max_{i\in[n]}
|\psi(\alpha_{i,\star},m)|
\le
m^{-(1+\alpha_{\min})},
\qquad m\ge1.
\]
Then, we have 
\begin{equation*}
\|G_m\|_{\mathrm{op}}
\le
\|A_\star^{-2}\|_{\mathrm{op}}\, m^{-(1+\alpha_{\min})}
+
C_I \|A_\star^{-1}\|_{\mathrm{op}}(1+\|A_\star^{-1}\|_{\mathrm{op}})^2 m^{-(1+2\alpha_{\min})}.    
\end{equation*}
Thus
\begin{equation}
\|G_m\|_{\mathrm{op}}
\le
\tilde{C}_G (m+1)^{-(1+\alpha_{\min})},
\; m\ge 0,    
\end{equation}
where $\tilde{C}_G=C_I'\Bigl(\|A_\star^{-2}\|_{\mathrm{op}}+\|A_\star^{-1}\|_{\mathrm{op}}(1+\|A_\star^{-1}\|_{\mathrm{op}})^2\Bigr)$.

Therefore,
\[
\begin{aligned}
\|\operatorname{Cov}(x_p,x_{p+k})\|_{\mathrm{op}}
&\le
\sigma^2
\sum_{m=0}^{p-1}
\|G_m\|_{\mathrm{op}}\|G_{m+k}\|_{\mathrm{op}}\\
&\le
\sigma^2\widetilde C_G^2
\sum_{m=0}^{\infty}
(m+1)^{-(1+\alpha_{\min})}
(m+k+1)^{-(1+\alpha_{\min})}.
\end{aligned}
\]
Splitting the last sum into $0\le m\le k$ and $m>k$
gives
\[
\sum_{m=0}^{\infty}
(m+1)^{-(1+\alpha_{\min})}
(m+k+1)^{-(1+\alpha_{\min})}
\le
\left(
1+\frac{1}{\alpha_{\min}}
+\frac{1}{2\alpha_{\min}+1}
\right)
(k+1)^{-(1+\alpha_{\min})}.
\]
Thus
\[
\|\operatorname{Cov}(x_p,x_{p+k})\|_{\mathrm{op}}
\le
\widetilde C_x
(k+1)^{-(1+\alpha_{\min})}.
\]
The case $q<p$ follows by covariance symmetry, which proves the
claim for all $p,q\ge0$.
\end{proof}

\begin{lemma}
\label{lem:explicit-transfer}
Fix $\phi\in(0,\pi/2)$ and $\beta_0>0$. Then there exists a constant
$K_{\phi,\beta_0}<\infty$ such that the following holds.

Let $\rho\in(0,1]$, and let
\[
\Delta(\phi,\rho):=\{z:|z|<1+\rho,\ z\neq 1,\ |\arg(z-1)|>\phi\}.
\]
If $f$ is analytic in $\Delta(\phi,\rho)$ and satisfies
\[
|f(z)|\le M |1-z|^\beta,\qquad z\in \Delta(\phi,\rho),
\]
for some $\beta\in[0,\beta_0]$, then for every $m\ge 1$,
\[
|[z^m]f(z)|
\le
K_{\phi,\beta_0}\,M\,\rho^{-(\beta+1)}\,m^{-(\beta+1)}.
\]
\end{lemma}

\begin{proof}
We follow the standard truncated Hankel-contour proof of the
Flajolet--Odlyzko transfer theorem~\citep{Flajolet1990SingularityAO}, keeping the $\rho$-dependence explicit.

For $m<2\rho^{-1}$, the right-hand side is $\ge 2^{-(\beta+1)} \geq 2^{-(\beta_0+1)}$.
Since $f$ is analytic on a fixed contour inside $\Delta(\phi,\rho)$ enclosing the
origin, Cauchy's formula gives $|[z^m]f(z)|\le C_{\phi,\beta_0}M$, so the claim
follows after enlarging $K_{\phi,\beta_0}$. One choice is $C_{\phi, \beta_0} = 2^{\beta_0}$ and $K_{\phi, \beta_0} \ge 2^{2 \beta_0 + 1}$.

Hence, it suffices to consider $m\ge 2\rho^{-1}$. Fix $\vartheta = \frac{\phi+\pi/2}{2}\in (\phi, \pi/2)$ and let $r_{\rho,\vartheta}>0$ be the unique solution to $$|1+re^{i\vartheta}|=1+\frac{\rho}{2},$$ namely, $r_{\rho,\vartheta}= -\cos \vartheta +\sqrt{\cos^2 \vartheta + \rho+\frac{\rho^2}{4}}$. Consider the contour $H_{m,\rho,\vartheta} \in \Delta(\phi,\rho)$ defined by
\[
\Gamma_\pm:=\{1+re^{\pm i\vartheta}: m^{-1}\le r\le r_{\rho,\vartheta}\},
\]
\[
\Gamma_0:=\{1+m^{-1}e^{i\theta}:\vartheta\le \theta\le 2\pi - \vartheta\},
\]
and the outer arc $\Gamma_{\rm out}\subset\{|z|=1+\rho/2\} \cap \Delta(\phi,\rho)$ joining the endpoints
of $\Gamma_+$ and $\Gamma_-$. Therefore, $[z^m]f(z) = \frac{1}{2\pi i} \int_{H_{m,\rho,\vartheta}}f(z)z^{-m-1}dz.$ On $\Gamma_{\pm}$, $|1-z|=r, |dz| = dr$, and
\[
|z|=|1+re^{\pm i\vartheta}| \geq 1+r \cos \vartheta.
\]
Therefore, on $\Gamma_\pm$, $|z|^{-m} \le (1+ r \cos \vartheta)^{-m} = e^{-m\log(1+r\cos \vartheta)} \le e^{-\frac{mr\cos\vartheta}{1+ r_{1, \vartheta} \cos \vartheta}} = e^{-c_\phi m r} $, where $c_\phi := \frac{\cos \vartheta}{1+ r_{1, \vartheta} \cos \vartheta}$. Using $|f(z)|\le M r^\beta$ there, we obtain
\begin{align*}
\int_{\Gamma_\pm} |f(z) z^{-m-1}\,dz|
\le
M\int_{m^{-1}}^{r_{\rho, \vartheta}} r^\beta e^{-c_\phi m r}\,dr
\le
M (c_\phi m)^{-(\beta+1)} \Gamma(\beta + 1)
\le
C_{\phi,\beta_0} M m^{-(\beta+1)}.   
\end{align*}

On $\Gamma_0$, we have $|1-z|=m^{-1}$ and $|f(z)| \le M m^{-\beta}$. The arc length satisfies $(2\pi - 2\vartheta) m^{-1} \le C_\phi m^{-1}$. Since the smallest possible value of $|z|$ on the arc occurs at \(\theta = \pi\), we have $|z| \ge 1 -\frac{1}{m}$. Since $m \ge 2 \rho^{-1} \ge 2$, we can uniformly upper bound $|z|^{-m-1}$ by $|z|^{-m-1} \le (1 - \frac{1}{m})^{-(m+1)} \le 8$. Hence
\begin{align*}
\int_{\Gamma_0} |f(z) z^{-m-1}\,dz| \le 8 Mm^{-\beta} C_\phi m^{-1} = C_\phi^\prime M m^{-(\beta +1)}.
\end{align*}
On the outer arc $\Gamma_{\rm out}$, $|z|=1+\rho/2$, so
\[
|z|^{-m}=(1+\rho/2)^{-m}\le e^{-c\rho m}
\]
for a universal $c>0$. Also, since $\beta\le \beta_0$ and $|1-z|\le 3$ on
$|z|=1+\rho/2$ with $\rho\le 1$, we have $|f(z)|\le 3^{\beta_0}M$. Hence
\[
\int_{\Gamma_{\rm out}}|f(z) z^{-m-1}\,dz|
\le
C_{\phi,\beta_0} M e^{-c\rho m}.
\]
Finally, because
\[
\sup_{x>0} x^{\beta_0+1} e^{-cx}<\infty
\]
and $\rho\le 1$, we have
\[
e^{-c\rho m}\le C_{\beta_0}\rho^{-(\beta+1)}m^{-(\beta+1)}.
\]
Combining the three contour bounds with Cauchy's coefficient formula proves the claim.

\end{proof}

\begin{lemma}
\label{lem:mu_t_positive}
Suppose \Cref{ass:stable_A} holds and consider the stationary
two-sided solution of the system. Define
\[
s_G
:=
\inf_{|z|=1}\sigma_{\min}(G(z)),
\qquad
\alpha_{\max}
:=
\max_{i\in[n]}\alpha_{i,\star}.
\]
Then, for every \(i\in[n]\),
\begin{align*}
\mu_{t,i}
&\ge
\sigma^2 s_G^2
\left[
\left(\alpha_{i,\star}-\frac12\right)^2
+
\left(
\frac{-3\alpha_{i,\star}^2
      +6\alpha_{i,\star}-2}{6}
\right)^2
\right]
\notag\\
&\ge
\frac{\sigma^2}
{\bigl(2^{\alpha_{\max}}+\|A_\star\|_{\mathrm{op}}\bigr)^2}
\left[
\left(\alpha_{i,\star}-\frac12\right)^2
+
\left(
\frac{-3\alpha_{i,\star}^2
      +6\alpha_{i,\star}-2}{6}
\right)^2
\right].
\end{align*}
\end{lemma}

\begin{proof}
Let
\[
d_{i,j}
:=
\partial_\alpha\psi(\alpha_{i,\star},j),
\qquad j\ge1,
\]
and define
\[
g_s^{(i)}
:=
\sum_{j\ge1}
d_{i,j}\,x_{s+1-j}^{(i)}.
\]
By the definition of the population profiled derivative curvature,
\[
\mu_{t,i}
=
\inf_{a_i\in\mathbb R^{1\times n}}
\frac1t
\sum_{s=0}^{t-1}
\mathbb E
\left|
g_s^{(i)}-a_i x_s
\right|^2.
\]
Under stationarity, every term in the preceding average has the same
distribution. Hence
\begin{equation}
\mu_{t,i}
=
\inf_{a_i\in\mathbb R^{1\times n}}
\mathbb E
\left|
g_0^{(i)}-a_i x_0
\right|^2.
\label{eq:mu_stationary}
\end{equation}

Using the stationary moving-average representation
\[
x_s
=
\sum_{m=0}^{\infty}
G_m\eta_{s-1-m},
\]
we have
\[
g_0^{(i)}
=
\sum_{j=1}^{\infty}
d_{i,j}x_{1-j}^{(i)}
=
\sum_{\ell\ge1}
h_{i,\ell}^{\top}\eta_{-\ell},
\]
where
\[
h_{i,\ell}^{\top}
:=
\sum_{j=1}^{\ell}
d_{i,j}\,e_i^\top G_{\ell-j}.
\]
Similarly,
\[
a_i x_0
=
\sum_{\ell=1}^{\infty}
a_iG_{\ell-1}\eta_{-\ell}.
\]
Therefore,
\[
g_0^{(i)}-a_i x_0
=
\sum_{\ell=1}^{\infty}
\bigl(
h_{i,\ell}^{\top}-a_iG_{\ell-1}
\bigr)\eta_{-\ell}.
\]
Since the innovations are independent with covariance
\(\sigma^2I_n\),
\begin{equation}
\mathbb E
\left|
g_0^{(i)}-a_i x_0
\right|^2
=
\sigma^2
\sum_{\ell=1}^{\infty}
\left\|
h_{i,\ell}^{\top}
-
a_iG_{\ell-1}
\right\|_2^2.
\label{eq:mu_coeff}
\end{equation}

Next define
\[
D_i(z)
:=
\sum_{j=1}^{\infty}d_{i,j}z^j.
\]
Since
\[
\sum_{j=0}^{\infty}\psi(\alpha,j)z^j
=
(1-z)^\alpha,
\]
we obtain
\[
D_i(z)
=
\left.
\partial_\alpha(1-z)^\alpha
\right|_{\alpha=\alpha_{i,\star}}
=
(1-z)^{\alpha_{i,\star}}\log(1-z).
\]
Moreover,
\[
\sum_{\ell\ge1}
h_{i,\ell}^{\top}z^\ell
=
D_i(z)e_i^\top G(z),
\]
whereas
\[
\sum_{\ell\ge1}
a_iG_{\ell-1}z^\ell
=
a_i zG(z).
\]
Thus,
\[
\sum_{\ell\ge1}
\bigl(
h_{i,\ell}^{\top}-a_iG_{\ell-1}
\bigr)z^\ell
=
\bigl(
D_i(z)e_i^\top-a_i z
\bigr)G(z).
\]

By Parseval's identity,
\begin{align}
\sum_{\ell\ge1}
\left\|
h_{i,\ell}^{\top}-a_iG_{\ell-1}
\right\|_2^2
=
\frac1{2\pi}
\int_0^{2\pi}
\left\|
\bigl(
D_i(e^{\mathrm{i}\theta})e_i^\top
-
a_i e^{\mathrm{i}\theta}
\bigr)
G(e^{\mathrm{i}\theta})
\right\|_2^2
\,d\theta.
\label{eq:mu_parseval}
\end{align}

By \Cref{ass:stable_A}, \(G(z)\) is invertible for every
\(|z|=1\). Since \(G\) is continuous on the unit circle,
\[
s_G
=
\inf_{|z|=1}\sigma_{\min}(G(z))
>0.
\]
Hence, using
\[
\|MN\|_2
\ge
\sigma_{\min}(N)\|M\|_2,
\]
\eqref{eq:mu_parseval} gives
\begin{align*}
\sum_{\ell\ge1}
\left\|
h_{i,\ell}^{\top}-a_iG_{\ell-1}
\right\|_2^2
\ge
s_G^2
\frac1{2\pi}
\int_0^{2\pi}
\left\|
D_i(e^{\mathrm{i}\theta})e_i^\top
-
a_i e^{\mathrm{i}\theta}
\right\|_2^2
\,d\theta.
\end{align*}

Applying Parseval's identity once more,
\begin{align}
\sum_{\ell\ge1}
\left\|
h_{i,\ell}^{\top}-a_iG_{\ell-1}
\right\|_2^2
\ge
s_G^2
\left[
\left\|
d_{i,1}e_i^\top-a_i
\right\|_2^2
+
\sum_{j=2}^{\infty}d_{i,j}^2
\right].
\label{eq:mu_lower_parseval}
\end{align}
Combining
\eqref{eq:mu_stationary},
\eqref{eq:mu_coeff}, and
\eqref{eq:mu_lower_parseval}, and minimizing over
\(a_i\in\mathbb R^{1\times n}\), yields
\begin{align}
\mu_{t,i}
&\ge
\sigma^2s_G^2
\inf_{a_i\in\mathbb R^{1\times n}}
\left[
\left\|
d_{i,1}e_i^\top-a_i
\right\|_2^2
+
\sum_{j=2}^{\infty}d_{i,j}^2
\right] \notag\\
&=
\sigma^2s_G^2
\sum_{j=2}^{\infty}d_{i,j}^2.
\label{eq:mu_exact}
\end{align}

For the first two terms in this sum,
\[
\psi(\alpha,2)
=
\frac{\alpha(\alpha-1)}{2},
\qquad
\psi(\alpha,3)
=
-\frac{\alpha(\alpha-1)(\alpha-2)}{6}.
\]
Therefore,
\[
d_{i,2}
=
\alpha_{i,\star}-\frac12,
\]
and
\[
d_{i,3}
=
\frac{-3\alpha_{i,\star}^2
      +6\alpha_{i,\star}-2}{6}.
\]
It follows from \eqref{eq:mu_exact} that
\begin{align}
\mu_{t,i}
\ge
\sigma^2s_G^2
\left[
\left(\alpha_{i,\star}-\frac12\right)^2
+
\left(
\frac{-3\alpha_{i,\star}^2
      +6\alpha_{i,\star}-2}{6}
\right)^2
\right].
\label{eq:mu_sG_lower}
\end{align}

It remains to lower bound \(s_G\). For \(|z|=1\),
\[
G(z)^{-1}
=
\operatorname{diag}
\left(
(1-z)^{\alpha_{1,\star}},
\ldots,
(1-z)^{\alpha_{n,\star}}
\right)
-
zA_\star.
\]
Hence
\begin{align*}
\|G(z)^{-1}\|_{\mathrm{op}}
&\le
\max_{i\in[n]}
|1-z|^{\alpha_{i,\star}}
+
\|A_\star\|_{\mathrm{op}}\\
&\le
2^{\alpha_{\max}}
+
\|A_\star\|_{\mathrm{op}},
\end{align*}
because \(|1-z|\le2\) on the unit circle. Therefore,
\[
s_G
=
\inf_{|z|=1}
\frac1{\|G(z)^{-1}\|_{\mathrm{op}}}
\ge
\frac1{
2^{\alpha_{\max}}
+
\|A_\star\|_{\mathrm{op}}
}.
\]
Substituting this bound into \eqref{eq:mu_sG_lower}
completes the proof.
\end{proof}

\end{document}